\documentclass[11pt]{article}
\usepackage[margin=1in]{geometry}
\newif\ifSpringerSubmission
\SpringerSubmissionfalse
\usepackage{amsmath,amssymb,amsthm,mathtools}
\usepackage{xcolor,bm,booktabs,enumitem,microtype,graphicx,aliascnt,placeins}
\ifSpringerSubmission\else
  \usepackage[numbers,sort&compress]{natbib}
\fi
\usepackage{hyperref}
\hypersetup{hidelinks}
\usepackage[nameinlink,noabbrev]{cleveref}

\newcommand{\R}{\mathbb R}
\newcommand{\C}{\mathbb C}
\newcommand{\N}{\mathbb N}
\newcommand{\Z}{\mathbb Z}
\newcommand{\EE}{\mathbb E}
\newcommand{\dd}{\mathrm d}
\newcommand{\eps}{\varepsilon}
\newcommand{\cE}{\mathcal E}

\newcommand{\cA}{\mathcal A}
\newcommand{\cS}{\mathcal S}
\newcommand{\cP}{\mathcal P}
\newcommand{\cK}{\mathcal K}
\newcommand{\cD}{\mathcal D}
\newcommand{\cB}{\mathcal B}
\newcommand{\cY}{\mathcal C}
\newcommand{\cH}{\mathcal H}
\newcommand{\outY}{\mathcal Y}
\newcommand{\tr}{\operatorname{tr}}
\newcommand{\supp}{\operatorname{supp}}
\newcommand{\ReLU}{\operatorname{ReLU}}
\newcommand{\Lip}{\operatorname{Lip}}
\newcommand{\fp}{\operatorname{fp}}
\newcommand{\Ran}{\operatorname{Ran}}
\newcommand{\Id}{\operatorname{Id}}
\newcommand{\iiprod}[1]{(\!(#1)\!)}

\ifSpringerSubmission\theoremstyle{thmstyleone}\else\theoremstyle{plain}\fi
\newtheorem{thm}{Theorem}[section]
\newaliascnt{prop}{thm}
\newtheorem{prop}[prop]{Proposition}
\aliascntresetthe{prop}
\newaliascnt{lem}{thm}
\newtheorem{lem}[lem]{Lemma}
\aliascntresetthe{lem}
\newaliascnt{cor}{thm}
\newtheorem{cor}[cor]{Corollary}
\aliascntresetthe{cor}
\ifSpringerSubmission\theoremstyle{thmstylethree}\else\theoremstyle{definition}\fi
\newaliascnt{dfn}{thm}
\newtheorem{dfn}[dfn]{Definition}
\aliascntresetthe{dfn}
\newaliascnt{ex}{thm}
\newtheorem{ex}[ex]{Example}
\aliascntresetthe{ex}
\ifSpringerSubmission\theoremstyle{thmstyletwo}\else\theoremstyle{remark}\fi
\newaliascnt{rmk}{thm}
\newtheorem{rmk}[rmk]{Remark}
\aliascntresetthe{rmk}

\crefname{thm}{Theorem}{Theorems}
\crefname{prop}{Proposition}{Propositions}
\crefname{lem}{Lemma}{Lemmas}
\crefname{cor}{Corollary}{Corollaries}
\crefname{dfn}{Definition}{Definitions}
\crefname{ex}{Example}{Examples}
\crefname{rmk}{Remark}{Remarks}
\crefname{section}{Section}{Sections}
\crefname{subsection}{Section}{Sections}
\crefname{appendix}{Appendix}{Appendices}
\Crefname{appendix}{Appendix}{Appendices}
\crefname{equation}{}{}
\creflabelformat{equation}{(#2#1#3)}

\newcommand{\manuscriptTitle}{Constructive Universal Approximation by Infinite-Dimensional Neural Networks via Gaussian-Weighted Ridgelet Transform}
\newcommand{\manuscriptKeywords}{Ridgelet transform; neural networks with infinite-dimensional inputs; operator learning; Gaussian measure; constructive universal approximation; Barron-type approximation rate; dimension-free rate; non-polynomial activation functions}
\newcommand{\manuscriptMSC}{41A30, 42C40, 44A12, 46G12, 60G15, 68T07}
\newcommand{\manuscriptDate}{29 September 2026}
\newcommand{\affiliationRIKEN}{%
  \orgname{RIKEN Center for Advanced Intelligence Project (AIP)},
  \orgaddress{\street{Nihonbashi 1-chome Mitsui Building, 1-4-1 Nihonbashi,
  Chuo-ku}, \city{Tokyo} \postcode{103-0027}, \country{Japan}}}
\newcommand{\affiliationCyberAgent}{%
  \orgdiv{AI tech studio}, \orgname{CyberAgent, Inc.},
  \orgaddress{\street{Shibuya Scramble Square, 2-24-12 Shibuya, Shibuya-ku},
  \city{Tokyo} \postcode{150-6121}, \country{Japan}}}
\newcommand{\affiliationKings}{%
  \orgdiv{Department of Mathematics}, \orgname{King's College London},
  \orgaddress{\street{Strand Building}, \city{London}, \postcode{WC2R 2LS},
  \country{UK}}}
\newcommand{\affiliationKyoto}{%
  \orgdiv{Department of Mathematics,
  Graduate School of Science}, \orgname{Kyoto University},
  \orgaddress{\city{Kyoto}, \postcode{606-8502}, \country{Japan}}}
\newcommand{\affiliationOsaka}{%
  \orgdiv{Department of Information Systems Engineering, Graduate School of
  Information Science and Technology}, \orgname{The University of Osaka},
  \orgaddress{\street{1-5 Yamadaoka}, \city{Suita}, \state{Osaka}
  \postcode{565-0871}, \country{Japan}}}
\title{\manuscriptTitle}
\ifSpringerSubmission
\author*[1,2]{\fnm{Sho} \sur{Sonoda}}\email{sho.sonoda@riken.jp}
\author[3]{\fnm{Luca} \sur{Galimberti}}\email{luca.galimberti@kcl.ac.uk}
\author[4]{\fnm{Yuanchao} \sur{Xu}}\email{xu.yuanchao.3a@kyoto-u.ac.jp}
\author[4,1]{\fnm{Isao} \sur{Ishikawa}}\email{ishikawa.isao.5s@kyoto-u.ac.jp}
\author[5,1]{\fnm{Masahiro} \sur{Ikeda}}\email{ikeda@ist.osaka-u.ac.jp}
\affil*[1]{\affiliationRIKEN}
\affil[2]{\affiliationCyberAgent}
\affil[3]{\affiliationKings}
\affil[4]{\affiliationKyoto}
\affil[5]{\affiliationOsaka}
\else
\newcommand{\orgdiv}[1]{#1}
\newcommand{\orgname}[1]{#1}
\newcommand{\orgaddress}[1]{#1}
\newcommand{\street}[1]{#1}
\newcommand{\city}[1]{#1}
\newcommand{\state}[1]{#1}
\newcommand{\postcode}[1]{#1}
\newcommand{\country}[1]{#1}
\newcommand{\arxivAuthorBlock}{%
\begin{minipage}{0.96\textwidth}
\centering
{\large Sho Sonoda\textsuperscript{1,2,*}\quad
Luca Galimberti\textsuperscript{3}\quad
Yuanchao Xu\textsuperscript{4}\\[3pt]
Isao Ishikawa\textsuperscript{4,1}\quad
Masahiro Ikeda\textsuperscript{5,1}\par}
\medskip
{\small
\textsuperscript{1}\affiliationRIKEN\par\smallskip
\textsuperscript{2}\affiliationCyberAgent\par\smallskip
\textsuperscript{3}\affiliationKings\par\smallskip
\textsuperscript{4}\affiliationKyoto\par\smallskip
\textsuperscript{5}\affiliationOsaka\par}
\medskip
{\footnotesize
\textsuperscript{*}Corresponding author: \texttt{sho.sonoda@riken.jp}\\[2pt]
\texttt{luca.galimberti@kcl.ac.uk}; \texttt{xu.yuanchao.3a@kyoto-u.ac.jp}\\[2pt]
\texttt{ishikawa.isao.5s@kyoto-u.ac.jp}; \texttt{ikeda@ist.osaka-u.ac.jp}\par}
\end{minipage}}
\author{\arxivAuthorBlock}
\date{\manuscriptDate}
\fi

\begin{document}
\maketitle
\begin{abstract}
Neural networks whose inputs are functions, as in operator learning, are usually shown to be universal by existence arguments, either by reduction to finite-dimensional networks or by Stone--Weierstrass-type density; these prove that approximants exist but say little about how the parameters depend on the target. We take a constructive route. We extend the ridgelet transform, an integral transform related to the Fourier, Radon, and wavelet transforms, to continuous maps on a separable infinite-dimensional Hilbert space, replacing the Lebesgue measure that is unavailable in infinite dimensions by Gaussian weights with a homogeneous mixture on the direction space. The resulting transform has a Plancherel identity and reconstruction formulas incorporating the frame operator, the correction induced by the weights, and it assigns to each target an explicit coefficient measure, one for every admissible analysis filter. From these representations we obtain uniform approximation on compact sets for all continuous non-polynomial activations of polynomial growth, including ReLU, and, by Monte Carlo sampling of the coefficient, finite-width networks with a dimension-free $O(N^{-1/2})$ error in the width $N$ for globally Lipschitz activations and coefficients with finite variation and second moments. Explicit examples include Gaussian targets and neural-operator layers. A companion Lean~4 formalization provides machine-checked proofs, with its correspondence to the manuscript documented at \url{https://shosonoda.github.io/lean-operator-ridgelet/}.

\end{abstract}
\noindent\textbf{Keywords.} \manuscriptKeywords.
\par\smallskip
\noindent\textbf{Mathematics Subject Classification (2020).} \manuscriptMSC.
\section{Introduction}
\label{sec:intro}

We develop ridgelet analysis for neural networks with infinite-dimensional
inputs.  The construction connects explicit integral representations to
finite-width approximation and uniform approximation on compact sets.

\subsection{From function inputs to network coefficients}

Operator learning seeks to approximate nonlinear maps whose inputs are
functions and whose outputs may also be functions.  Examples include
solution maps of differential equations and layers that combine an
integral operator with a pointwise nonlinearity
\citep{Kovachki2021NOreview,Lu2021,Lanthaler2022}.  On a separable real Hilbert
space \(\cH\), a basic model with values in a Hilbert space \(\outY\) is
\[
 f_N(x)=\sum_{j=1}^N v_j\sigma(\langle a_j,x\rangle-b_j),
 \qquad a_j\in \cH,\ b_j\in\R,\ v_j\in \outY.
\]
Each direction \(a_j\) defines a linear measurement of the input.  Beyond
asking whether these networks can approximate a target, we ask how to
choose their directions, biases, and output coefficients from that target.

An integral representation addresses this question by replacing the
finite sum with a superposition of neurons,
\[
 f(x)=S_\sigma[\Gamma](x)
     :=\int_{\cH\times\R}\sigma(\langle a,x\rangle-b)\,\Gamma(\dd a,\dd b).
\]
Here \(\Gamma\) is a finite scalar or vector coefficient measure.
Its total variation measures the size of the representation.
Once \(\Gamma\) is known, its normalized variation supplies a probability
law for sampling neurons; its polar density supplies the output weights.
Barron's argument then relates the finite-width approximation error to the variation and
moments of this measure \citep{Barron1993,Bach2017b}.

The analytic tool used here is a ridgelet transform: an integral operator
that sends a target to coefficients indexed by directions and biases.
It is related to the Fourier, Radon, and wavelet transforms
\citep{Murata1996,Candes.PhD,Sonoda2017acha}, but we construct and
explain it directly on \(\cH\).  No knowledge of ridgelet analysis is
needed to read the definitions or the proofs below.

Universality of such networks is usually proved by existence arguments:
by reduction to a finite-dimensional approximation theorem, or by
Stone--Weierstrass-type density arguments on the infinite-dimensional space
\citep{Benth2023,Ismailov2026AMAI,CuchieroSchmockerTeichmann2026}; random
features give rates on Hilbert inputs \citep{Gonon2023reservoir}.  These
results establish existence or sampling bounds without an analysis
operator.  In this paper, we supply that operator: an explicit coefficient
formula, a norm identity, and reconstruction, for the full class of
continuous non-polynomial activations of polynomial growth, from which
universality and finite-width rates follow.

The main technical obstacle is that an infinite-dimensional Hilbert space
has no locally finite, translation-invariant analogue of Lebesgue measure, so
the finite-dimensional ridgelet integrals cannot be carried over directly.
We integrate against a Gaussian input measure and a homogeneous mixture of
Gaussian measures on the direction space; the homogeneity parameter of the
mixture replaces the dimension-dependent factor of the finite-dimensional
admissibility condition.  \Cref{sec:transform} constructs these measures
and the transform.

\subsection{Main results}
\label{subsec:main-results}

Our main contributions connect transform identities to constructive
neural-network approximation; the universality statement itself is of a
type already known, and the novelty lies in how the approximants are
obtained.  The precise assumptions are stated with each theorem.

\paragraph{Transform and reconstruction.}
We construct a Gaussian-weighted ridgelet transform with a Plancherel
identity: the norm of the coefficients equals a constant times the
target norm in a Hilbert space defined by weighted Fourier analysis
(\cref{thm:B}).  The transform is injective and has closed range.
Reconstruction uses a \emph{frame operator}, an invertible correction
induced by the Gaussian weights, which can be undone after synthesis
or before analysis (\cref{thm:C}).

\paragraph{Explicit network coefficients.}
For targets given by a Fourier integral, we derive an explicit
coefficient formula and an integral-network representation
(\cref{thm:A}).  Smoothness of the Fourier density along lines through
the origin, together with weighted bounds on its derivatives, makes
the coefficient measure finite and gives the parameter moments needed
for approximation (\cref{thm:E}).
The representation allows continuous activations of polynomial growth;
non-polynomiality permits a nonzero reconstruction constant.
In particular, ReLU is admissible when paired with a suitable analysis
filter (\cref{cor:relu-admissible}).  Different admissible analysis filters
give different coefficients for the same target, so the representation is a
family indexed by the filter.

\paragraph{Approximation rates and universality.}
Monte Carlo approximation of the integral representation gives
finite-width networks.  For globally Lipschitz activations, their
expected uniform error on a compact set is of order \(N^{-1/2}\),
where \(N\) is the width, with no explicit input or output dimension
factor (\cref{thm:lipschitz-barron}).  The constant depends on the
coefficient variation, parameter moments, and compact set.
Density of the represented targets then proves \emph{cc-universality},
meaning uniform approximation of every continuous target on each
compact subset of the input space (\cref{thm:D}).  These results cover
Hilbert-valued outputs, including function-valued operators.

\paragraph{Examples and numerical evidence.}
We give Gaussian targets, Gaussian-parameter networks, and
neural-operator layers built from periodic convolution and a Dirichlet
solution operator, with explicit representations
(\cref{sec:examples}).  Three experiments illustrate the predicted
approximation rates across input resolutions, check an explicit
transform formula, and separate finite-width error from numerical
discretization error (\cref{sec:numerics}).

\paragraph{Scope.}
The coefficient formula is explicit once a spectral density is given; for a
general continuous target, the universality proof first chooses a spectral
approximation, and the constants \(V\) and \(M_2\) depend on that choice.
We do not obtain a uniform width--accuracy relation for all continuous
targets, a coefficient-computation algorithm, or a statistical learning
guarantee.  Universality is proved in the compact-open topology, not globally
on weighted non-compact spaces, and the \(N^{-1/2}\) rate is a Barron-type
statement for targets whose coefficient has finite variation and moments,
consistent with the lower bounds for Lipschitz operator classes
(\cref{subsec:related}).  Every finite-width network factors through finitely
many measurements \(\langle a_j,x\rangle\); representing the directions
themselves with finitely many numbers is a further approximation problem.
Tempered reconstruction is an identity in a dual space, and general nonlinear
operator activations are outside the reconstruction theorem.

\subsection{Related work}
\label{subsec:related}

Ridgelet analysis originates with \citet{Murata1996}, \citet{Candes.PhD},
and \citet{Candes1999a} as a transform adapted to ridge functions.  Its
distributional extension admits unbounded activations, including ReLU
\citep{Sonoda2017acha}; weighted Sobolev activation spaces are developed in
\citet{Sonoda2021ghost}; and the Fourier-slice method yields transforms for
group convolutional and non-Euclidean networks
\citep{Sonoda2022gconv,Sonoda2024fourier}, where the Hilbert-space
construction of \citet[Section~4]{Sonoda2024fourier} integrates over a
finite-dimensional subspace.  Complexity bounds derived from integral
representations are surveyed in \citet{Kurkova2012}.

Integration over an infinite-dimensional space with respect to a Gaussian
measure is classical: Fourier--Wiener transforms
\citep{CameronMartin1945,Kuo1989}, Gaussian Radon transforms
\citep{Hertle1980,MihaiSengupta2007,HolmesSengupta2012,Holmes2013,Becnel2019},
and a ridgelet transform for Wiener functionals \citep{Martias2012ridgelet},
whose atoms are generated by a functional on the Cameron--Martin space and
whose basic identity relates two Gaussian spaces of different variance
rather than reconstructing the target.  In this paper, the atoms are
neurons, the direction measure is homogeneous, and the Plancherel identity
holds on a single space; \cref{sec:extended-related} makes these
comparisons explicit, including the relation of our analysis transform to
the Gaussian Radon transform.

Universal approximation with infinite-dimensional inputs goes back to the
approximation of functionals \citep{ChenChen1993,Chen1995,Stinchcombe1999,
RossiConanGuez2005}.  Compact-open universality is known on Fr\'echet
spaces with operator activations \citep{Benth2023,Galimberti2025}, in
\(L^p(\mu)\) \citep{Galimberti2026Lp}, for scalar ridge neurons on locally
convex spaces with every continuous non-polynomial activation
\citep{Ismailov2026AMAI}, and globally on weighted spaces
\citep{CuchieroSchmockerTeichmann2026,SchmockerTeichmann2026}; neural
operators \citep{Kovachki2021NOreview,Lu2021,Lanthaler2022,
KovachkiLanthalerMishra2021,LanthalerLiStuart2023} motivate the examples in
\cref{sec:examples}.  These are existence theorems; here the approximants
are obtained from an explicit coefficient, one for each admissible analysis
filter.

Finite-width bounds by sampling a parameter distribution follow
\citet{Barron1993}, \citet{Bach2017b}, and \citet*{E-Ma-Wu2022barron}, with
Hilbert-input versions in \citet{Gonon2023reservoir,Gonon2024},
Banach-valued versions in \citet{Korolev2022,NeufeldSchmocker2023random},
and encoder--decoder versions in \citet{yang-shi2026}.  In this paper, the
sampling measure is the homogeneous measure \(\nu_\alpha\), the coefficient
is characterized by an analysis operator with a Plancherel identity and a
closed-range projection, and the uniform rate covers Hilbert-valued outputs.
Lower bounds \citep{LanthalerStuart2025,Lanthaler2024entropy,
AdcockMaierParhi2025} exclude such rates for all Lipschitz operators; ours
concerns a Barron-type class.  Continuous frames
\citep{Fornasier2005,Holighaus2017errata}, Gaussian measures
\citep{Bogachev1998book,DaPrato2014,Prevot2007book}, and Wiener chaos
\citep{Janson1997book} supply the general background.

\subsection{Notation and organization}
\label{subsec:notation}

Unless stated otherwise, \(\cH\) is an infinite-dimensional separable
real Hilbert space with inner product
\(\langle\cdot,\cdot\rangle\), \(\cB(\cH)\) its Borel sigma-algebra, and
\(\mu_{Q}\) the centred Gaussian measure with covariance \(Q\), that is,
the Borel probability measure with
\(\int_\cH e^{i\langle x,\xi\rangle}\mu_{Q}(\dd x)=e^{-\langle Q\xi,\xi\rangle/2}\).
Throughout, \(P\) and \(Q\) are injective, positive, self-adjoint,
trace-class operators on \(\cH\), and
\(D_\omega a:=\omega a\).  Functions on \(\cH\) are complex valued unless
stated otherwise; \(\outY\) denotes a separable complex Hilbert space of
outputs.  On \(\R\) we use
\begin{equation}
  h^\sharp(\omega)=\int_\R h(t)e^{-it\omega}\,\dd t,
  \qquad
  h(t)=\frac1{2\pi}\int_\R h^\sharp(\omega)e^{it\omega}\,\dd\omega,
  \label{eq:fourier-convention}
\end{equation}
We use \(\sharp\) for one-dimensional Fourier transforms, including
tempered distributions and the partial transform in the bias:
\[
 \gamma^\sharp(a,\omega)
 :=\int_\R\gamma(a,b)e^{-ib\omega}\,\dd b.
\]
A hat denotes the spatial Fourier transform, with the same negative
exponential convention. For a finite measure \(M\) on \(\cH\),
\[
 \widehat M(\xi)
 :=\int_\cH e^{-i\langle x,\xi\rangle}\,M(\dd x).
\]
For square-integrable coefficients the partial transform is understood
in \(L^2\); see \cref{lem:partial-fourier-l2}.
The letter \(\rho\) always denotes the analysis filter, \(\sigma\) a
scalar synthesis activation; \(\omega\) is a frequency or dilation variable, \(\xi\) a
spectral variable, \(a\) a direction, and \(b\) a bias.  Distributional
pairings \(\langle u,\varphi\rangle\) are bilinear.
We write \(\mu_C\) for a centred Gaussian law with covariance \(C\),
on the ambient Hilbert space specified in context; in one dimension
\(\mu_v\) denotes the Gaussian law of variance \(v\).
When operator notation is needed, \(F_\R\) denotes the
one-dimensional Fourier transform and \(F_b\) its partial version
in the bias; \(F_Q\) denotes the Gaussian-weighted spatial transform.
The inverse one-dimensional transform is denoted by \(\check h\).
Spectral densities are denoted by \(g\), parameter densities by
\(\gamma\), and finite coefficient measures by \(\Gamma\).

\Cref{sec:networks} fixes the network classes.  \Cref{sec:transform}
constructs the transform and proves the Plancherel identity (\cref{thm:B}).  \Cref{sec:reconstruction}
proves the representation and reconstruction theorems (\cref{thm:A,thm:C}), describes the frame operator, and records the
vector-valued extension.  \Cref{sec:tempered} gives reconstruction formulas
for activations that are tempered distributions, including ReLU.
\Cref{sec:finite-width} contains the approximation rates and the
universality theorem (\cref{thm:D,thm:E}).  \Cref{sec:examples} presents explicit
examples, \cref{sec:numerics} reports numerical experiments, and
\cref{sec:discussion} discusses open questions. The supporting lemmas
and proofs needed for these results appear in the main text before
their use. The appendices give additional kernel and projection
formulas, supplementary approximation bounds, the finite-dimensional
comparison and dilation obstruction, numerical methods, and an extended
overview of related work (\cref{sec:extended-related}).
The experiment code is supplied as supplementary material.

Every theorem, proposition, lemma, corollary, and example of this paper has
been formalized and machine-checked in Lean~4 with Mathlib.  The blueprint at
\url{https://shosonoda.github.io/lean-operator-ridgelet/} links each item to
the Lean statement that carries it and to its proof.  Two clauses are stated
there in a weaker form than here: the \(L^2\) clause in the conclusion of
\cref{thm:weak-sobolev-synthesis}, and the closed forms of the reconstruction
constants in \cref{prop:nonbandpass-sobolev}, which the Lean statements carry
as the Sobolev pairing itself.  Nothing in the present paper depends on the
formalization; it is offered as an independent check.
 \section{Neural networks with Hilbert-space inputs}
\label{sec:networks}

We define finite-width networks and integral networks with scalar
activations and Hilbert-valued outputs.  These definitions fix the
coefficient measures and the compact-open topology used in the
representation and approximation results.

\subsection{Finite-width and integral networks}

\begin{dfn}[Finite-width network]
\label{def:finite-network}
Let \(\sigma:\R\to\C\) be continuous.  A width-\(N\) network on \(\cH\) with
values in \(\outY\) is
\begin{equation}
  f_N(x)=\sum_{j=1}^N v_j\,\sigma(\langle a_j,x\rangle-b_j),
  \qquad v_j\in \outY,\quad (a_j,b_j)\in \cH\times\R.
  \label{eq:finite-network}
\end{equation}
The scalar case is \(\outY=\C\).
\end{dfn}

We measure approximation in the compact-open topology of \(C(\cH;\outY)\),
generated by the seminorms
\[
  \|f\|_{C(K;\outY)}:=\sup_{x\in K}\|f(x)\|_\outY,
  \qquad K\subset \cH\ \text{compact},
\]
and write \(C(K)\) when \(\outY=\C\).  Compact-open approximation does not
truncate the input: it asks for uniform approximation on each fixed
compact subset of \(\cH\).  Finite-rank projections converge uniformly on a
compact set, but a target need not factor through any one of them, and the
examples in \cref{sec:examples} do not.

A vector measure assigns a vector to each Borel set and is countably
additive in the output norm.  Its variation is the positive measure
\(|\Gamma|(E)=\sup\sum_j\|\Gamma(E_j)\|_\outY\), where the supremum is
over finite Borel partitions of \(E\); bounded variation means that its
mass on the whole parameter space is finite.  Since \(\outY\) is a
Hilbert space, such a measure has a density \(h\) with respect to its
variation: \(\Gamma(E)=\int_E h\,\dd|\Gamma|\), with
\(\|h\|_\outY=1\) almost everywhere.  Accordingly, an integral
\(\int u\,\dd\Gamma\) of a scalar function means the Bochner integral
\(\int uh\,\dd|\Gamma|\).  A vector-valued function is Bochner
integrable when it is strongly measurable (an almost-everywhere limit
of measurable simple functions) and its norm is integrable; its
integral is obtained by integrating simple functions that approximate
it in \(L^1\).  We refer to \citep{Diestel1977book} for more details.

\begin{dfn}[Integral network]
\label{def:integral-network}
Let \(\Gamma\) be a \(\outY\)-valued Borel measure of bounded variation on
\(\Theta:=\cH\times\R\), with variation measure \(|\Gamma|\) and total
variation \(\|\Gamma\|_{\mathrm{TV}}:=|\Gamma|(\Theta)\).  Whenever the
Bochner integral exists, define
\begin{equation}
  S_\sigma[\Gamma](x)
  :=\int_{\cH\times\R}\sigma(\langle a,x\rangle-b)\,\Gamma(\dd a,\dd b).
  \label{eq:measure-synthesis}
\end{equation}
If \(\sigma\) is globally Lipschitz, it suffices that
\begin{equation}
  \int_{\cH\times\R}(1+\|a\|+|b|)\,\dd|\Gamma|(a,b)<\infty .
  \label{eq:measure-moment}
\end{equation}
When \(\Gamma=\gamma\lambda\) has a density with respect to a
sigma-finite reference measure \(\lambda\) on \(\Theta\), we also write
\(S_\sigma[\gamma]\), and we say that the network has coefficient
\(\gamma\).
\end{dfn}

Finite-width networks are the syntheses of finite atomic measures.  The
total variation \(\|\Gamma\|_{\mathrm{TV}}\) is the infinite-width analogue
of the \(\ell^1\)-norm of the outer weights, and it is the quantity that
enters Barron's discretization in \cref{sec:finite-width}.  Scalar
coefficients are complex; a real network is recovered by taking real
parts.

 \section{The Gaussian-weighted ridgelet transform}
\label{sec:transform}

We construct the transform using a Gaussian input measure and a
homogeneous mixture of Gaussian direction measures.  The Fourier-slice
and Plancherel identities establish its analytic properties, and explicit
filters show how the admissibility assumptions can be met.

\subsection{The input and direction measures}
\label{subsec:measures}

Recall that \(P\) and \(Q\) are injective, positive, self-adjoint,
trace-class operators on \(\cH\); \(\mu_Q\) is the input Gaussian law.
Gaussian weighting makes the input integral available, but a single
Gaussian law on directions is not homogeneous under dilation.
In finite-dimensional ridgelet analysis, the dilation Jacobian
\(|\omega|^{-m}\) of Lebesgue measure on \(\R^m\) enters the
admissibility condition.  The following mixture replaces \(m\) by a
freely chosen homogeneity parameter \(\alpha\), so the corresponding
condition remains meaningful on \(\cH\).
For \(\alpha>0\),
define a Borel measure on the direction space by
\begin{equation}
  \nu_\alpha(E)
  :=\int_0^\infty \mu_{2sP}(E)\,s^{\alpha/2-1}\,\dd s,
  \qquad E\in\cB(\cH).
  \label{eq:homogeneous-mixture}
\end{equation}
We first realize the Gaussian components on a common probability space
and justify integration against the mixture.

\label{app:gaussian-details}
The Gaussian measures can be realized on a common probability space
\((\Omega,\mathfrak F,\mathbb P)\).  Let \(Pe_j=p_je_j\) with \(p_j>0\),
\(\sum_jp_j<\infty\), and let \(Z_j\) be independent standard normal
variables.  The series
\[
  X=\sum_{j\ge1}\sqrt{p_j}Z_je_j,\qquad X_s=\sqrt{2s}\,X
\]
converges in \(L^2(\Omega;\cH)\) and almost surely: orthogonality gives the
\(L^2\) tail estimate, and \(\EE\sum_jp_jZ_j^2=\tr P<\infty\) gives
almost-sure square summability.  Its finite-dimensional coordinates show
that \(X\sim\mu_{P}\) and \(X_s\sim\mu_{2sP}\); see
\citet[Chapter~2]{Prevot2007book} and \citet[Chapter~2]{Bogachev1998book}.
All moments of \(\|X\|\) are finite: by Jensen's inequality applied to
the probability weights \(p_j/\tr P\),
\begin{equation}
  \EE\|X\|^{2m}
  =\EE\Bigl(\sum_jp_jZ_j^2\Bigr)^m
  \le(\tr P)^m\,\EE Z_1^{2m}
  =(\tr P)^m(2m-1)!! .
  \label{eq:gaussian-norm-moments}
\end{equation}
The same construction with \(Q\) in place of \(P\) realizes \(\mu_Q\).

\begin{lem}[Measurability and integration of the mixture]
\label{lem:mixture-integration}
For every \(E\in\cB(\cH)\), the map \(s\mapsto\mu_{2sP}(E)\) is Borel
measurable.  The mixture \cref{eq:homogeneous-mixture} defines a countably
additive Borel measure, and for every nonnegative Borel function
\(F:\cH\to[0,\infty]\),
\begin{equation}
  \int_\cH F(\xi)\,\nu_\alpha(\dd\xi)
  =\int_0^\infty\int_\cH F(\xi)\,
    \mu_{2sP}(\dd\xi)\,s^{\alpha/2-1}\,\dd s .
  \label{eq:mixture-integration}
\end{equation}
Both sides may be infinite.  For complex \(F\) the identity holds when
\(\int_\cH|F|\,\dd\nu_\alpha<\infty\).
\end{lem}

\begin{proof}
The map \((s,x)\mapsto\sqrt{2s}\,x\) is continuous on
\((0,\infty)\times \cH\), and
\(\mu_{2sP}(E)=\int_\cH\mathbf 1_E(\sqrt{2s}\,x)\,\mu_{P}(\dd x)\).  The
integrand is jointly Borel, so its integral is a Borel function of
\(s\).  Monotone convergence applied to disjoint unions proves countable
additivity.  The integral identity holds for indicators by definition,
then for nonnegative simple functions, and finally for all nonnegative
Borel \(F\) by monotone convergence; the complex case follows by
linearity.
\end{proof}

\begin{lem}[Homogeneous Gaussian mixture]
\label{lem:homogeneous-mixture}
For every \(\alpha>0\), the measure \(\nu_\alpha\) is infinite on \(\cH\)
and has full support, and for \(\omega\ne0\),
\begin{equation}
  (D_\omega)_\#\nu_\alpha=|\omega|^{-\alpha}\nu_\alpha,
  \label{eq:homogeneity}
\end{equation}
equivalently, for every nonnegative Borel \(F\) on \(\cH\),
\begin{equation}
  \int_\cH F(\omega a)\,\nu_\alpha(\dd a)
  =|\omega|^{-\alpha}\int_\cH F(\xi)\,\nu_\alpha(\dd\xi).
  \label{eq:homogeneous-change}
\end{equation}
The measure \(\nu_\alpha\) is moreover sigma-finite and finite on
bounded Borel sets.
\end{lem}

\begin{proof}
\label{app:proof-homogeneous-mixture}
Let \(p_j>0\) be the eigenvalues of \(P\).  For \(a\sim\mu_{2sP}\), any
\(k\ge1\), and the ball \(B_r=\{a\in \cH:\|a\|<r\}\), the first \(k\)
coordinates give
\[
  \mu_{2sP}(B_r)
  \le r^k(\pi s)^{-k/2}\prod_{j\le k}p_j^{-1/2},
\]
because the density of each coordinate is bounded by
\((4\pi sp_j)^{-1/2}\) and the coordinates lie in an interval of length
\(2r\).  Choose \(k>\alpha\), which is possible because
\(\dim \cH=\infty\).  The part \(s\le1\) of
\cref{eq:homogeneous-mixture} is integrable because its integrand is at
most \(s^{\alpha/2-1}\); the small-ball bound treats \(s\ge1\).  Thus
bounded sets have finite mass.  Balls exhaust \(\cH\), whereas
\(\nu_\alpha(\cH)=\int_0^\infty s^{\alpha/2-1}\dd s=\infty\), proving
sigma-finiteness and infinitude.  Every Gaussian component has full support
because \(P\) is injective, and so does the mixture.

The law of \(\omega a\) under \(\mu_{2sP}\) is \(\mu_{2s\omega^2P}\);
the substitution \(u=s\omega^2\) in \cref{eq:homogeneous-mixture} proves
\cref{eq:homogeneity} for both signs of \(\omega\), and
\cref{eq:homogeneous-change} is its integrated form by
\cref{lem:mixture-integration}.
\end{proof}

In finite dimension the mixture is explicit: on \(\cH=\R^m\) with \(P=I\)
and \(0<\alpha<m\), \(\nu_\alpha(\dd a)=c_{m,\alpha}\|a\|^{\alpha-m}\dd a\)
(\cref{subsec:finite-reference-measure}).  Thus \(\alpha\) is the
homogeneity degree of the direction measure, not a dimension, and in
infinite dimension every \(\alpha>0\) is allowed.

\subsection{Admissible filters, the transform, and the Fourier slice}
\label{subsec:transform}

\begin{dfn}[Admissible analysis filter]
\label{def:admissible-filter}
A real function \(\rho\in\cS(\R)\) is \(\alpha\)-admissible if
\begin{equation}
  0<\iiprod{\rho,\rho}_{\alpha}
  :=\frac1{2\pi}\int_\R
  |\rho^\sharp(\omega)|^2|\omega|^{-\alpha}\,\dd\omega<\infty.
  \label{eq:self-admissibility}
\end{equation}
We call it a band-pass filter if moreover
\begin{equation}
  \rho^\sharp\in C_c^\infty(\R\setminus\{0\}),
  \label{eq:bandpass-filter}
\end{equation}
which specifies the compact Fourier support away from zero required by
the term band-pass here.  Such a filter is
\(\alpha\)-admissible for every \(\alpha>0\); since
\(\rho^\sharp\) vanishes near the origin, the integral in
\cref{eq:self-admissibility} is in fact finite and positive for every real
\(\alpha\).  For two admissible filters put
\begin{equation}
  \iiprod{\rho_1,\rho_2}_{\alpha}
  :=\frac1{2\pi}\int_\R
  \rho_1^\sharp(\omega)\overline{\rho_2^\sharp(\omega)}
  |\omega|^{-\alpha}\,\dd\omega .
  \label{eq:cross-admissibility}
\end{equation}
The self-admissibility constant \cref{eq:self-admissibility} is the case
\(\rho_1=\rho_2=\rho\) of this pairing, since \(\rho\) is real.
\end{dfn}

Explicit band-pass filters, and the Mexican hat as an example of a
non-band-pass admissible filter, are given in
\cref{ex:bandlimited-filter,ex:mexican-hat}.

\begin{dfn}[Gaussian-weighted ridgelet transform]
\label{def:ridgelet-analysis}
For \(f\in L^1(\cH,\mu_Q)\) and \(\rho\in\cS(\R)\), define
\begin{equation}
  R_\rho f(a,b)
  :=\int_\cH f(x)\rho(\langle a,x\rangle-b)\,\mu_Q(\dd x),
  \qquad (a,b)\in \cH\times\R,
  \label{eq:ridgelet-analysis}
\end{equation}
and let \(\lambda_\alpha:=\nu_\alpha\otimes\dd b\) be the parameter
measure on \(\cH\times\R\).
\end{dfn}

The Gaussian input weight makes the defining integral meaningful on
\(\cH\).  The associated Fourier operator is the Fourier transform of the
finite measure \(f\mu_Q\),
\begin{equation}
  F_Qf(\xi)
  :=\int_\cH f(x)e^{-i\langle x,\xi\rangle}\,\mu_Q(\dd x)
  =\widehat{f\mu_Q}(\xi),
  \qquad f\in L^1(\cH,\mu_Q).
  \label{eq:gaussian-analysis}
\end{equation}
Because \(\mu_Q\) is a probability measure, \(L^2(\mu_Q)\subset L^1(\mu_Q)\),
and \(F_Qf\) is bounded and continuous.

\begin{lem}[Fourier-slice identity]
\label{lem:fourier-slice}
Let \(f\in L^1(\cH,\mu_Q)\) and \(\rho\in\cS(\R)\).  Then \(R_\rho f\) is
bounded and jointly continuous on \(\cH\times\R\), and for every \(a\in \cH\)
\[
 \begin{aligned}
  \|R_\rho f(a,\cdot)\|_{L^1(\R)}
  &\le\|f\|_{L^1(\mu_Q)}\|\rho\|_{L^1(\R)},\\
  \|R_\rho f(a,\cdot)\|_{L^2(\R)}^2
  &\le\|f\|_{L^2(\mu_Q)}^2\|\rho\|_{L^2(\R)}^2
  \ \text{ if } f\in L^2(\mu_Q).
 \end{aligned}
\]
For every \(a\in \cH\) and \(\omega\in\R\),
\begin{equation}
  (R_\rho f)^\sharp(a,\omega)
  :=\int_\R R_\rho f(a,b)e^{-i\omega b}\,\dd b
  =\rho^\sharp(-\omega)\,F_Qf(\omega a).
  \label{eq:fourier-slice}
\end{equation}
\end{lem}

\begin{proof}
Boundedness follows from
\(|R_\rho f(a,b)|\le\|f\|_1\|\rho\|_\infty\), and joint continuity from
dominated convergence.  Cauchy--Schwarz in the probability measure
\(\mu_Q\) and Tonelli in \(b\) give the \(L^2\) bound.  The \(L^1\) bound
and Fubini follow from
\(\int_\cH\int_\R|f(x)||\rho(\langle a,x\rangle-b)|\dd b\,\mu_Q(\dd x)
=\|f\|_{L^1(\mu_Q)}\|\rho\|_{L^1}\).
With \(u=\langle a,x\rangle-b\), the inner Fourier transform equals
\(e^{-i\omega\langle a,x\rangle}\rho^\sharp(-\omega)\), which is
\cref{eq:fourier-slice}.
\end{proof}

The Fourier-slice identity determines the coefficient from the spectral
side alone.  For square-integrable coefficients we use the partial
Fourier transform in the bias; its precise construction and jointly
measurable representatives are given in \cref{lem:partial-fourier-l2}.

We first construct the partial Fourier transform on the coefficient space.
This lemma, and several auxiliary results below, are stated for
coefficients with values in a separable complex Hilbert space \(\outY\);
the scalar case is \(\outY=\C\).  Vector- and function-valued outputs
arise for neural operators and for compositions of network layers
(\cref{subsec:vector-valued,sec:examples}), and stating these results
once for \(\outY\) settles measurability, the interchange of integrals,
and Plancherel for every later use, with Bochner integrals in place of
scalar integrals.

\begin{lem}[Partial Fourier transform of a square-integrable coefficient]
\label{lem:partial-fourier-l2}
Let \(\nu\) be a sigma-finite Borel measure on \(\cH\), and let \(\outY\) be a
separable complex Hilbert space.  The Fourier transform in the bias
defines a unitary map
\[
 F_b:L^2(\nu\otimes\dd b;\outY)
       \longrightarrow L^2(\nu\otimes\dd\omega/(2\pi);\outY).
\]
Its image has a jointly strongly measurable representative
\(\gamma^\sharp(a,\omega)\) whose section, for \(\nu\)-almost every
\(a\), is the \(L^2\) Fourier transform of \(\gamma(a,\cdot)\).
These representatives are equal almost everywhere in the product
measure, and
\[
 \int_\cH\int_\R\|\gamma^\sharp(a,\omega)\|_\outY^2
       \frac{\dd\omega}{2\pi}\,\nu(\dd a)
 =\|\gamma\|_{L^2(\nu\otimes\dd b;\outY)}^2.
\]
\end{lem}
\begin{proof}
Fubini identifies the two spaces with the Bochner spaces
\[
 L^2(\nu;L^2(\dd b;\outY)),\qquad
 L^2(\nu;L^2(\dd\omega/(2\pi);\outY)).
\]
Apply the Hilbert-valued
one-dimensional Plancherel unitary to each section.  This preserves
strong measurability by approximation with simple functions, and
preserves the squared norm by Tonelli.  The inverse operation is the
sectionwise inverse Fourier transform.  The inverse Fubini
identification supplies a jointly measurable representative;
sectionwise uniqueness and Fubini give product almost-everywhere
uniqueness.  A square-integrable coefficient need not be integrable
in the bias, so its Fourier transform cannot in general be computed by
an absolutely convergent integral.  The transform and its inverse are
defined here through convergence in the \(L^2\) norm; no pointwise
convergence of a Fourier integral is asserted.
\end{proof}

This partial transform yields the following coefficient operator.
Its scaled-isometry property is proved immediately below.

\begin{dfn}[The coefficient operator]
\label{def:spectral-coefficient}
Let \(\rho\) be \(\alpha\)-admissible and let \(g:\cH\to\C\) be Borel with
\(g\in L^2(\nu_\alpha)\).  The coefficient \(W_\rho g\in L^2(\lambda_\alpha)\)
is the function whose partial Fourier transform in the bias is
\begin{equation}
  (W_\rho g)^\sharp(a,\omega)=\rho^\sharp(-\omega)\,g(\omega a).
  \label{eq:spectral-coefficient}
\end{equation}
For every such \(g\), the following inverse Fourier integral converges
absolutely for \(\nu_\alpha\)-almost every \(a\), and defines a
representative for all \(b\in\R\):
\begin{equation}
  \gamma_g(a,b):=W_\rho g(a,b)
  =\frac1{2\pi}\int_\R\rho^\sharp(-\omega)\,g(\omega a)\,
  e^{i\omega b}\,\dd\omega .
  \label{eq:spectral-coefficient-formula}
\end{equation}
\end{dfn}

\begin{lem}[The coefficient operator is a scaled isometry]
\label{lem:coefficient-isometry}
\(W_\rho:L^2(\nu_\alpha)\to L^2(\lambda_\alpha)\) is well defined,
independent of the Borel representative of \(g\), and
\(\|W_\rho g\|_{L^2(\lambda_\alpha)}^2=\iiprod{\rho,\rho}_{\alpha}\|g\|_{L^2(\nu_\alpha)}^2\).
For every \(g\in L^2(\nu_\alpha)\),
\(\rho^\sharp(-\,\cdot\,)g(\,\cdot\,a)\in L^1(\R)\) for
\(\nu_\alpha\)-almost every \(a\), so
\cref{eq:spectral-coefficient-formula} holds without an additional
\(L^1(\nu_\alpha)\) hypothesis.
\end{lem}

\begin{proof}
The map \((a,\omega)\mapsto g(\omega a)\) is Borel, and by
\cref{eq:homogeneous-change} and Tonelli,
\[
  \frac1{2\pi}\int_\R\int_\cH|\rho^\sharp(\omega)|^2|g(\omega a)|^2
  \,\nu_\alpha(\dd a)\,\dd\omega
  =\iiprod{\rho,\rho}_{\alpha}\|g\|_{L^2(\nu_\alpha)}^2 .
\]
This also proves representative independence by applying the identity
to the difference of two representatives, and
\cref{lem:partial-fourier-l2} gives the coefficient.  To justify its
integral formula, put
\[
 A_\rho^2=\int_\R|\rho^\sharp(\omega)|^2|\omega|^{-\alpha}
                      (1+\omega^2)\,\dd\omega<\infty,
 \qquad
 B_g(a)^2=\int_\R|g(\omega a)|^2
                  \frac{|\omega|^\alpha}{1+\omega^2}\,\dd\omega.
\]
The first integral is finite by admissibility near zero and Schwartz
decay at infinity.  Homogeneity and Tonelli give
\(\int_\cH B_g(a)^2\nu_\alpha(\dd a)=\pi\|g\|_2^2\).
Cauchy--Schwarz bounds
\(\int|\rho^\sharp(-\omega)g(\omega a)|\dd\omega\)
by \(A_\rho B_g(a)<\infty\) for almost every \(a\).  Fourier inversion
for an \(L^1\cap L^2\) frequency section proves the formula.
\end{proof}

The construction and norm identity require only
\(\iiprod{\rho,\rho}_{\alpha}<\infty\); its positivity makes the scaled isometry
injective.  The additional condition \(g\in L^1(\nu_\alpha)\) enters the
representation theorem, where it defines a bounded continuous target
and justifies synthesis.

\begin{lem}[The adjoint coefficient operator]
\label{lem:coefficient-adjoint}
For an \(\alpha\)-admissible \(\rho\), let
\[
 W_\rho^*:L^2(\lambda_\alpha)\longrightarrow L^2(\nu_\alpha)
\]
be the Hilbert adjoint of the coefficient operator \(W_\rho\).
Writing \(C=\iiprod{\rho,\rho}_{\alpha}\), one has, for every
\(\gamma\in L^2(\lambda_\alpha)\),
\[
 W_\rho^*W_\rho=C\Id,
 \qquad
 \|W_\rho^*\gamma\|_{L^2(\nu_\alpha)}
 \le\sqrt C\|\gamma\|_{L^2(\lambda_\alpha)}.
\]
For any jointly measurable representative \(\gamma^\sharp\) of the
partial Fourier transform supplied by \cref{lem:partial-fourier-l2},
its explicit formula is
\begin{equation}
 W_\rho^*\gamma(\xi)
 =\frac1{2\pi}\int_{\R\setminus\{0\}}\overline{\rho^\sharp(-\omega)}
       |\omega|^{-\alpha}\gamma^\sharp(\xi/\omega,\omega)\,\dd\omega,
 \label{eq:ray-average}
\end{equation}
where integration is with respect to Lebesgue measure on
\(\R\setminus\{0\}\).  The integral converges absolutely for
\(\nu_\alpha\)-almost every \(\xi\), and its \(L^2\) class is independent
of that representative.
\end{lem}
\begin{proof}
Polarizing \cref{lem:coefficient-isometry} gives
\(W_\rho^*W_\rho=C\Id\), and the adjoint has the same norm as the
operator.  For the formula, homogeneity and
\cref{lem:partial-fourier-l2} give
\[
 \frac1{2\pi}\int_\cH\int_{\R\setminus\{0\}} |\omega|^{-\alpha}
       |\gamma^\sharp(\xi/\omega,\omega)|^2
       \,\dd\omega\,\nu_\alpha(\dd\xi)=\|\gamma\|_2^2.
\]
Weighted Cauchy--Schwarz proves absolute convergence and the same norm
bound for the formula; the change of variables also preserves product
null sets.  Fubini and \(\xi=\omega a\) then give
\(\langle W_\rho^*\gamma,g\rangle
=\langle\gamma,W_\rho g\rangle\), with absolute integral bounded by
\(\sqrt C\|\gamma\|_2\|g\|_2\).  Thus the formula is the adjoint.

\end{proof}

Following the terminology for adjoint transforms in tomography
\citep{Natterer1986book}, we call \(W_\rho^*\) the
\emph{backprojection operator}.  Here it combines coefficient data
along lines through the origin and takes values in the frequency
domain.

In this notation the Fourier-slice identity reads
\begin{equation}
  R_\rho f=W_\rho\,F_Qf,
  \qquad f\in L^1(\mu_Q),\ F_Qf\in L^2(\nu_\alpha):
  \label{eq:slice-as-coefficient}
\end{equation}
the ridgelet transform of \(f\) is the coefficient of its Gaussian-weighted
Fourier transform.

\subsection{The Hilbert space \texorpdfstring{\(\cE_\alpha\)}{E-alpha}}
\label{subsec:spectral-space}

\begin{dfn}[The Hilbert space \(\cE_\alpha\)]
\label{def:spectral-space}
Set
\begin{align}
  \cD_\alpha
  &:=\left\{f\in L^2(\cH,\mu_Q):
       F_Qf\in L^2(\cH,\nu_\alpha)\right\},
       \label{eq:concrete-core}\\
  \langle f,g\rangle_{\cE_\alpha}
  &:=\int_\cH F_Qf(\xi)\overline{F_Qg(\xi)}
       \,\nu_\alpha(\dd\xi),
       \qquad f,g\in\cD_\alpha .
       \label{eq:spectral-inner-product}
\end{align}
The Hilbert space \(\cE_\alpha\) is the completion of
\(\cD_\alpha\) in the norm induced by \cref{eq:spectral-inner-product},
and
\begin{equation}
  \cK_\alpha
  :=\overline{F_Q(\cD_\alpha)}^{\,L^2(\nu_\alpha)}
  \label{eq:spectral-range}
\end{equation}
is the closure of the range of \(F_Q\) on \(\cD_\alpha\).
\end{dfn}

\begin{lem}[Positivity and the unitary extension]
\label{lem:spectral-unitary}
The form \cref{eq:spectral-inner-product} is positive definite on
\(\cD_\alpha\). The restriction of \(F_Q\) to \(\cD_\alpha\)
extends uniquely to a unitary map
\(F_Q:\cE_\alpha\to\cK_\alpha\), denoted by the same symbol.
\end{lem}

\begin{proof}
If the norm of \(f\in\cD_\alpha\) vanishes, then \(F_Qf=0\)
\(\nu_\alpha\)-almost everywhere; continuity of \(F_Qf\) and full support
of \(\nu_\alpha\) make it zero everywhere, and the Fourier transform
determines finite complex Borel measures on a separable Hilbert space
(finite-dimensional projections determine all cylinder sets, which
generate \(\cB(\cH)\)), so \(f\mu_Q=0\) and \(f=0\) in \(L^2(\mu_Q)\).
Thus \(F_Q\) is an isometry from the pre-Hilbert space \(\cD_\alpha\)
onto a dense subspace of \(\cK_\alpha\), and it extends to the completion.
\end{proof}

The completion matters: the norm of \(\cE_\alpha\) need not control
the \(L^2(\mu_Q)\)-norm, so elements of \(\cE_\alpha\) are abstract
limits, while elements of the dense subspace \(\cD_\alpha\) are functions.
On this core the extended \(F_Q\) agrees with \(\widehat{f\mu_Q}\)
as an \(L^2(\nu_\alpha)\) equivalence class. For a general element
of \(\cE_\alpha\), the extension supplies such a class; no pointwise
Gaussian integral representation is assumed.

\begin{lem}[Gaussian decay with polynomial weights]
\label{lem:gaussian-decay}
For every \(t>0\), \(\alpha>0\), and integer \(m\ge0\),
\begin{equation}
  \int_\cH\|\xi\|^{2m}e^{-t\langle Q\xi,\xi\rangle}
  \,\nu_\alpha(\dd\xi)<\infty.
  \label{eq:gaussian-decay}
\end{equation}
Consequently, if \(f\in L^2(\mu_Q)\) and
\(|F_Qf(\xi)|\le C(1+\|\xi\|)^pe^{-t\langle Q\xi,\xi\rangle/2}\) for
some \(C,p,t>0\), then \(f\in\cD_\alpha\) for every \(\alpha>0\).
\end{lem}

\begin{proof}
\label{app:proof-gaussian-decay}
Let \(\theta_j>0\) be the eigenvalues of the injective trace-class
operator \(P^{1/2}QP^{1/2}\).  Gaussian integration gives, for
\(u>0\),
\begin{equation}
  \int_\cH e^{-u\langle Q\xi,\xi\rangle}\,\mu_{2sP}(\dd\xi)
  =\prod_j(1+4su\theta_j)^{-1/2};
  \label{eq:gaussian-layer-integral}
\end{equation}
no commutativity of \(P\) and \(Q\) is required, since under
\(\mu_{2sP}\) the quadratic form \(\langle Q\xi,\xi\rangle\) has the law
of \(2s\sum_j\theta_jZ_j^2\).  By \cref{eq:mixture-integration} it
suffices to bound
\(\int_\cH\|\xi\|^{2m}e^{-t\langle Q\xi,\xi\rangle}\mu_{2sP}(\dd\xi)\).
Cauchy--Schwarz, \cref{eq:gaussian-norm-moments} with
\(\xi=\sqrt{2s}X\), and \cref{eq:gaussian-layer-integral} with \(u=2t\)
give
\[
  \int_\cH\|\xi\|^{2m}e^{-t\langle Q\xi,\xi\rangle}\mu_{2sP}(\dd\xi)
  \le(2s)^m\bigl((\tr P)^{2m}(4m-1)!!\bigr)^{1/2}
  \prod_j(1+8st\theta_j)^{-1/4}.
\]
For \(s\le1\) the product is at most one and the bound is
\(Cs^m\).  For \(s\ge1\), retaining \(k\) factors gives
\(\prod_{j\le k}(8st\theta_j)^{-1/4}=C_ks^{-k/4}\).  Choosing
\(k>4m+2\alpha\), both parts are integrable against
\(s^{\alpha/2-1}\dd s\), proving \cref{eq:gaussian-decay}.  For the
consequence, \(|F_Qf|^2\le C^2(1+\|\xi\|)^{2p}e^{-t\langle Q\xi,\xi\rangle}\)
is integrable by the first part.
\end{proof}

\begin{ex}[Elements of \(\cD_\alpha\)]
\label{ex:core-elements}
The constant function \(1\) has \(F_Q1(\xi)=e^{-\langle Q\xi,\xi\rangle/2}\),
so \(1\in\cD_\alpha\) for every \(\alpha>0\) and \(\cE_\alpha\ne\{0\}\).

\end{ex}

\begin{proof}
The Gaussian characteristic functional gives \(F_Q1\), and
\cref{lem:gaussian-decay} with \(m=0\), \(p=0\), \(t=1\) gives
\(1\in\cD_\alpha\).
\end{proof}

Further examples of elements of \(\cD_\alpha\) will be
the non-cylindrical Gaussian targets and Gaussian-activation
operator layers treated in \cref{sec:examples}.

\subsection{Plancherel identity, closed range, and injectivity}
\label{subsec:plancherel}

The following theorem is the first of the main results: it is the
Plancherel identity of the ridgelet transform in infinite dimension.

\begin{thm}[Plancherel identity and injectivity]
\label{thm:B}
Let \(\alpha>0\).
\begin{enumerate}[label=\textup{(\roman*)}]
\item
For \(f,g\in\cD_\alpha\) and \(\alpha\)-admissible \(\rho_1,\rho_2\),
the transforms \(R_{\rho_1}f\) and \(R_{\rho_2}g\) belong to
\(L^2(\lambda_\alpha)\), and
\begin{equation}
  \langle R_{\rho_1}f,R_{\rho_2}g\rangle_{L^2(\lambda_\alpha)}
  =\iiprod{\rho_1,\rho_2}_{\alpha}
   \langle f,g\rangle_{\cE_\alpha}.
  \label{eq:ridgelet-plancherel}
\end{equation}
\item
An \(\alpha\)-admissible \(\rho\) determines a unique bounded extension
\(R_\rho:\cE_\alpha\to L^2(\lambda_\alpha)\) with
\(\|R_\rho f\|^2=\iiprod{\rho,\rho}_{\alpha}\|f\|_{\cE_\alpha}^2\); its range is
closed, and \(R_\rho=W_\rho F_Q\).
\item
If \(\rho\) is \(\alpha\)-admissible and \(f\in L^1(\cH,\mu_Q)\), then
\(R_\rho f=0\) \(\lambda_\alpha\)-almost everywhere implies \(f=0\)
\(\mu_Q\)-almost everywhere.
\end{enumerate}
\end{thm}

\begin{proof}
\label{app:proof-plancherel}
\label{app:proof-injectivity}
By \cref{lem:fourier-slice}, \(R_{\rho_j}f(a,\cdot)\in L^2(\R)\) for
every \(a\) and \(j=1,2\), so the one-dimensional Plancherel identity
\(\|u\|_{L^2(\R)}^2=\frac1{2\pi}\|u^\sharp\|_{L^2(\R)}^2\) applies in
the bias.  Tonelli, \cref{eq:fourier-slice}, and
\cref{eq:homogeneous-change} give, for \(j=1,2\),
\[
  \|R_{\rho_j}f\|_{L^2(\lambda_\alpha)}^2
  =\frac1{2\pi}\int_\R|\rho_j^\sharp(\omega)|^2
      \int_\cH|F_Qf(\omega a)|^2\,\nu_\alpha(\dd a)\,\dd\omega
  =\iiprod{\rho_j,\rho_j}_{\alpha}\|f\|_{\cE_\alpha}^2<\infty,
\]
where the factor \(\frac1{2\pi}\) is the Plancherel normalization
absorbed in \cref{eq:cross-admissibility}; the same holds for \(g\).
Weighted Cauchy--Schwarz bounds
\(\frac1{2\pi}\int|\rho_1^\sharp\rho_2^\sharp||\omega|^{-\alpha}\dd\omega\)
by \((\iiprod{\rho_1,\rho_1}_{\alpha}\iiprod{\rho_2,\rho_2}_{\alpha})^{1/2}\),
and \(\int_\cH|F_Qf\,F_Qg|\dd\nu_\alpha\le\|f\|_{\cE_\alpha}\|g\|_{\cE_\alpha}\).
These estimates justify Fubini in the cross identity and separate the
frequency integral from the inner product of \(\cE_\alpha\), proving
\cref{eq:ridgelet-plancherel}.  The norm identity extends \(R_\rho\)
from the dense core \(\cD_\alpha\) to its completion, an isometry up to a
nonzero scalar has closed range, and \(R_\rho=W_\rho F_Q\) holds on
the core by \cref{eq:slice-as-coefficient} and extends by continuity of
both sides (\cref{lem:coefficient-isometry}).

\emph{(iii).} For \(f\in L^1(\mu_Q)\), \cref{lem:fourier-slice} gives \(R_\rho f(a,\cdot)\in L^1(\R)\) for every \(a\).  Tonelli applied
to the hypothesis shows that, for \(\nu_\alpha\)-almost every \(a\), this
bias function vanishes almost everywhere, so its Fourier transform is
zero at every frequency.  Choose \(\omega_0\ne0\) with
\(\rho^\sharp(\omega_0)\ne0\); then \cref{eq:fourier-slice} and
\cref{eq:homogeneous-change} imply \(F_Qf=0\) \(\nu_\alpha\)-almost
everywhere.  No finite spectral-norm assumption on \(f\) is used.  Since
\(F_Qf\) is continuous and \(\supp\nu_\alpha=\cH\), it vanishes
everywhere, and Fourier uniqueness for finite measures, as in the proof
of \cref{lem:spectral-unitary}, gives \(f\mu_Q=0\).
\end{proof}

Admissibility in \cref{thm:B} is a condition on the analysis filter.  It
imposes nothing on a synthesis activation such as ReLU, which is paired
with the filter through the distributional constant of \cref{sec:tempered}.
\subsection{Examples of analysis filters}
\label{subsec:analysis-filters}
\label{app:filters}

The analysis filter and the synthesis activation play different roles.
ReLU is a synthesis activation; its Fourier transform is distributional
and it does not satisfy \cref{eq:self-admissibility}.  The following
Schwartz filters satisfy the self-admissibility condition and illustrate
the distinction between admissibility and band-pass support.

\begin{ex}[A band-pass filter for every \(\alpha>0\)]
\label{ex:bandlimited-filter}
Let \(\eta(u)=\exp(-1/(1-u^2))\) for \(|u|<1\) and \(\eta(u)=0\)
otherwise, and define a real even \(\rho_{\mathrm{bp}}\) by
\(\rho_{\mathrm{bp}}^\sharp(\omega):=-\eta(2|\omega|-3)\).  Its Fourier
transform is smooth, nonpositive, nonzero, and supported in
\(\{1\le|\omega|\le2\}\), so \(\rho_{\mathrm{bp}}\in\cS(\R)\) is real and
even, satisfies \cref{eq:bandpass-filter}, and is \(\alpha\)-admissible
for every \(\alpha>0\); the sign makes it admissible for ReLU synthesis
in the sense of \cref{cor:relu-admissible}.  Multiplying by
\((\iiprod{\rho_{\mathrm{bp}},\rho_{\mathrm{bp}}}_{\alpha})^{-1/2}\) normalizes the admissibility
constant to one.
\end{ex}
\begin{proof}
\label{app:proof-bandlimited-filter}
The bump \(\eta\) and all of its derivatives vanish at \(|u|=1\), and the
support stays away from \(\omega=0\), so the absolute value causes no
loss of smoothness.  Fourier inversion maps \(C_c^\infty\) into \(\cS\),
and even real Fourier data give an even real inverse.  On the compact
support, \(|\omega|^{-\alpha}\) is bounded above and below, so
\cref{eq:self-admissibility} is finite and positive.
\end{proof}

\begin{ex}[Mexican hat]
\label{ex:mexican-hat}
The Schwartz function \(\rho_{\mathrm{MH}}(t):=(1-t^2)e^{-t^2/2}\) has
\(\rho_{\mathrm{MH}}^\sharp(\omega)=\sqrt{2\pi}\,\omega^2e^{-\omega^2/2}\).
For \(\alpha>0\) it is \(\alpha\)-admissible in the sense of
\cref{eq:self-admissibility} exactly when \(\alpha<5\), with
\(\iiprod{\rho_{\mathrm{MH}},\rho_{\mathrm{MH}}}_{\alpha}=\Gamma((5-\alpha)/2)\); in particular
\(\iiprod{\rho_{\mathrm{MH}},\rho_{\mathrm{MH}}}_{1}=1\).  It is not band pass, so
\cref{thm:B,thm:C} and \cref{thm:A}(ii) apply when \(\alpha<5\).
The band-pass statements \cref{thm:A}(iii) and
\cref{thm:tempered-reconstruction} do not apply to this filter.
\end{ex}
\begin{proof}
\label{app:proof-mexican-hat}
If \(g(t)=e^{-t^2/2}\), then \(\rho_{\mathrm{MH}}=-g''\), and the
differentiation rule gives the Fourier transform.  Then
\(\iiprod{\rho_{\mathrm{MH}},\rho_{\mathrm{MH}}}_{\alpha}=\int_\R|\omega|^{4-\alpha}e^{-\omega^2}\dd\omega
=\Gamma((5-\alpha)/2)\), convergent at zero exactly when \(4-\alpha>-1\),
that is, when \(\alpha<5\).  Here \(\alpha>0\) is the standing restriction
under which \(\nu_\alpha\) is defined; for \(\alpha\le0\) the integral
converges as well.
\end{proof}

More generally, if \(\rho\in\cS(\R)\) has \(k\) vanishing moments, then
\(\rho^\sharp(\omega)=O(|\omega|^k)\) at zero and
\cref{eq:self-admissibility} is finite whenever \(\alpha<2k+1\).  The
band-pass condition \cref{eq:bandpass-filter} supplies arbitrarily many
vanishing moments at once and is useful for the smooth criterion.  The Sobolev criterion
below permits filters whose Fourier support meets the origin.

 \section{Representation and reconstruction}
\label{sec:reconstruction}

We construct coefficients for targets with spectral densities and derive
integral representations under the stated integrability conditions.
We then identify the correction induced by the weights and give
reconstruction formulas, including the vector-valued extension.

\subsection{Targets with a spectral density}
\label{subsec:representation}

The Fourier-slice identity, read backwards, produces a network from a
spectral density.  For a Borel function \(g\in L^1(\nu_\alpha)\) put
\begin{equation}
  f_g(x):=\int_\cH e^{i\langle x,\xi\rangle}g(\xi)\,\nu_\alpha(\dd\xi),
  \qquad x\in \cH .
  \label{eq:spectral-target}
\end{equation}
Thus \(f_g(x)=\widehat{g\nu_\alpha}(-x)\).  We call \(g\nu_\alpha\)
the spectral measure of \(f_g\), and \(g\) its spectral density with
respect to \(\nu_\alpha\).  The class of such targets is the analogue, for the
reference measure \(\nu_\alpha\), of the class of functions with an
integrable spectral measure used by \citet{Barron1993}.
Its density in \(C(\cH)\), proved in \cref{thm:D}, uses full support and
finiteness of \(\nu_\alpha\) on bounded sets.  The elementary facts
about an individual spectral target require neither condition.

For the band-pass representation in part (iii) of the next theorem,
write \(g_a(\omega):=g(\omega a)\).  For \(a\ne0\), this is the restriction
of \(g\) to the line \(\{\omega a:\omega\in\R\}\) through the origin.
We impose smoothness on \(g_a\) and weighted integrability on its
derivatives.  For a band-pass \(\rho\), fix a symmetric compact set
\(I\subset\R\setminus\{0\}\) containing \(\supp\rho^\sharp\).
The other parts of the theorem do not require this regularity condition.

\begin{lem}[An integrable spectral density defines a unique continuous target]
\label{lem:spectral-target-basic}
Let \(\nu\) be a sigma-finite Borel measure on \(\cH\), let \(\outY\) be a
separable complex Hilbert space, and let \(g\in L^1(\nu;\outY)\).  Then
\(f_g(x):=\int_\cH e^{i\langle x,\xi\rangle}g(\xi)\nu(\dd\xi)\)
belongs to \(C_b(\cH;\outY)\), satisfies
\(\|f_g\|_\infty\le\|g\|_{L^1(\nu;\outY)}\), and determines \(g\)
uniquely up to \(\nu\)-almost-everywhere equality.
Only the Fourier transform of the finite vector measure \(g\nu\) enters:
boundedness, continuity, and uniqueness follow from \(g\in L^1(\nu;\outY)\)
alone, and neither the Gaussian input measure \(\mu_Q\) nor the analysis
filter \(\rho\) plays a role.
\end{lem}
\begin{proof}
The norm bound is the integral triangle inequality, and continuity
follows by dominated convergence with majorant \(\|g(\xi)\|_\outY\).
If \(f_g=0\), pair the finite vector measure \(g\nu\) with any vector
of a countable dense subset of \(\outY\).  Each resulting finite complex
measure has zero Fourier transform and hence is zero: finite-dimensional
Fourier uniqueness determines its cylinder sets, which generate the
Borel sigma-algebra of the separable Hilbert space \(\cH\).
The corresponding scalar densities vanish almost everywhere.
Intersecting the countably many full-measure sets and using continuity
of the pairing gives \(g=0\) almost everywhere.  Apply this to a
difference for uniqueness.
\end{proof}

\begin{dfn}[Regularity along rays]
\label{def:ray-regular}
A bounded Borel function \(g:\cH\to\C\) is regular along rays if for every
\(a\in \cH\) its restriction \(g_a(\omega)=g(\omega a)\) is \(C^\infty\) on a
neighbourhood of \(I\), and for every integer \(m\ge0\)
\begin{equation}
  M_m(g):=\int_\cH(1+\|a\|)^{m+2}\max_{k\le m}
  \sup_{\omega\in I}\bigl|\partial_\omega^kg_a(\omega)\bigr|
  \,\nu_\alpha(\dd a)<\infty .
  \label{eq:ray-regularity}
\end{equation}
\end{dfn}

The next lemma is stated for a density with values in a separable complex
Hilbert space \(\outY\), as needed for the vector-valued extension in
\cref{subsec:vector-valued}; the scalar case is \(\outY=\C\).  The
coefficient formula \cref{eq:spectral-coefficient-formula} applies
verbatim to such \(g\): the scalar factor \(\rho^\sharp(-\omega)e^{i\omega b}\)
multiplies the vector \(g(\omega a)\), and the integral is a Bochner
integral,
\[
  \gamma_g(a,b)=\frac1{2\pi}\int_\R\rho^\sharp(-\omega)\,e^{i\omega b}\,
  g(\omega a)\,\dd\omega\in\outY .
\]
Accordingly \(W_\rho\) acts on \(L^2(\nu_\alpha;\outY)\) through the
\(\outY\)-valued partial Fourier transform of \cref{lem:partial-fourier-l2},
and the proof of \cref{lem:coefficient-isometry} applies with
\(\|\cdot\|_\outY\) in place of the absolute value.

\begin{lem}[Finite-order coefficient decay and moments]
\label{lem:coefficient-finite-order}
Fix a band-pass \(\rho\), a symmetric compact
\(I\subset\R\setminus\{0\}\) containing \(\supp\rho^\sharp\), and an
integer \(r\ge0\).  Let \(\outY\) be a separable complex Hilbert space
and let \(g:\cH\to \outY\) be bounded and Borel, with
\(\omega\mapsto g(\omega a)\) of class \(C^{r+2}\) near \(I\) for
every \(a\).  Set
\[
 D_m(a):=\max_{k\le m}\sup_{\omega\in I}
                  \|\partial_\omega^kg(\omega a)\|_\outY,
 \qquad
 A_{m,r}(g):=\int_\cH(1+\|a\|)^rD_m(a)\,\nu_\alpha(\dd a).
\]
If \(A_{r+2,r}(g)<\infty\), then \(g\in L^1\cap L^2(\nu_\alpha;\outY)\),
the inverse integral \cref{eq:spectral-coefficient-formula} defines a
jointly measurable representative of \(W_\rho g\), and
\begin{align}
 (1+|b|)^{r+2}\|\gamma_g(a,b)\|_\outY
    &\le C_{\rho,I,r}D_{r+2}(a),
       \label{eq:finite-order-coefficient-decay}\\
 \int_{\cH\times\R}(1+\|a\|+|b|)^r\|\gamma_g(a,b)\|_\outY
       \,\dd\lambda_\alpha
    &\le c_{\rho,I,r}A_{r+2,r}(g).
       \label{eq:finite-order-coefficient-moment}
\end{align}
The constants are independent of \(g\) and of the dimension of \(\outY\).
\end{lem}
\begin{proof}
The functions under the suprema are jointly Borel, since derivatives
are limits of difference quotients of \(g(\omega a)\).  Continuity in
\(\omega\) lets us take the suprema over a countable dense subset of
\(I\), so \(D_m\) is measurable.  Fix any \(\omega_0\in I\).
Homogeneity bounds \(\int\|g\|\,\dd\nu_\alpha\) by
\(|\omega_0|^\alpha\int D_0\,\dd\nu_\alpha<\infty\); boundedness
of \(g\) gives its square integrability.

Put \(m=r+2\) and
\(h_a(\omega)=\rho^\sharp(-\omega)g(\omega a)\).  For each \(a\)
this is a compactly supported \(C^m\), \(\outY\)-valued function.  Extension
by zero outside the neighbourhood of \(I\) is \(C^m\), because all
derivatives of \(\rho^\sharp\) vanish at the boundary of its support.
Leibniz' rule and symmetry of \(I\) give
\(\|h_a\|_{L^1}+\|h_a^{(m)}\|_{L^1}
\le C'_{\rho,I,m}D_m(a)\).
The inverse Fourier integral is jointly measurable by parameter
integration and is the coefficient by \cref{lem:coefficient-isometry}.
For \(|b|\le1\) bound it by \((2\pi)^{-1}\|h_a\|_1\).
For \(|b|>1\), integration by parts \(m\) times, with vanishing
boundary terms, bounds it by
\((2\pi)^{-1}|b|^{-m}\|h_a^{(m)}\|_1\).
These inequalities prove \cref{eq:finite-order-coefficient-decay}.
Finally,
\(1+\|a\|+|b|\le(1+\|a\|)(1+|b|)\) and
\(\int_\R(1+|b|)^{-2}\dd b=2\) give
\cref{eq:finite-order-coefficient-moment} by Tonelli.
\end{proof}

In particular, \(C^4\) regularity of \(g_a(\omega)=g(\omega a)\) and \(A_{4,2}(g)<\infty\)
suffice for the second parameter moment.  This finite-order estimate
does not by itself establish an integral representation with a
continuous activation of polynomial growth.
For a density satisfying \cref{def:ray-regular},
\(A_{r+2,r}(g)\le M_{r+2}(g)\), so every parameter moment is finite.

\begin{lem}[Examples satisfying the spectral regularity condition]
\label{lem:ray-regular-examples}
The following functions satisfy \cref{def:ray-regular} for every band-pass \(\rho\).
\begin{enumerate}[label=\textup{(\alph*)}]
\item
\(g(\xi)=q(\xi)e^{-\kappa(\xi)/2}\), where
\(\kappa(\xi)=\langle S\xi,\xi\rangle\) for a bounded positive operator
\(S\) with \(S\ge\theta Q\) for some \(\theta>0\), and where \(q\) is
a polynomial in \(\kappa(\xi)\) and in finitely many bounded linear
functionals \(\ell_1,\dots,\ell_k\) of \(\xi\), each satisfying
\(|\ell_i(\xi)|^2\le C_i\kappa(\xi)\) for all \(\xi\), equivalently
\(\ell_i=\langle S^{1/2}v_i,\cdot\rangle\) for some \(v_i\in \cH\).  A
polynomial in \(\kappa\) alone, in particular a constant \(q\), needs no
further condition.
\item
\(g(\xi)=\varphi(\|\xi-\xi_0\|^2)\) with \(\xi_0\in \cH\) and
\(\varphi\in C_c^\infty(\R)\); more generally any bounded \(g\) that vanishes
outside a bounded set, whose restrictions \(g_a(\omega)=g(\omega a)\)
are \(C^\infty\) near \(I\), and for which
\(\sup_{\omega\in I}|\partial_\omega^kg(\omega a)|\le C_k(1+\|a\|)^{p_k}\)
for all \(k\).
\item
Finite linear combinations of functions satisfying \cref{def:ray-regular}, and
Bochner integrals \(\int_\Omega g_y\,m(\dd y)\) of a uniformly bounded
measurable family \((g_y)_{y\in\Omega}\) of such functions over a finite
measure \(m\) on \(\Omega\) whose bounds in \cref{eq:ray-regularity} are
uniform in \(y\) on a neighbourhood of \(I\), in the following sense:
there are an open set \(U\supseteq I\), on which
\(\omega\mapsto g_y(\omega a)\) is \(C^\infty\) for every \(y,a\),
and finite-valued Borel functions
\(h_k:\cH\to[0,\infty)\), \(k\ge0\), with
\begin{align*}
  \sup_{\omega\in U}\bigl|\partial_\omega^kg_y(\omega a)\bigr|
  &\le h_k(a)
  \qquad\text{for every }y\in\Omega,\ a\in \cH,\ k\ge0,\\
  \int_\cH(1+\|a\|)^{m+2}\max_{k\le m}h_k(a)\,\nu_\alpha(\dd a)
  &<\infty
  \qquad\text{for every integer }m\ge0 .
\end{align*}
The supremum must be taken over the neighbourhood \(U\), and the bound
must hold for every \(a\), not only \(\nu_\alpha\)-almost everywhere.
\end{enumerate}
\end{lem}

\begin{proof}
For (a), the domination \(|\ell_i(\xi)|^2\le C_i\kappa(\xi)\) gives
\(|q(\xi)|\le C(1+\kappa(\xi))^p\) for some \(C\) and \(p\), hence
\(|g(\xi)|\le C\sup_{t\ge0}(1+t)^pe^{-t/2}\), and \(g\) is bounded; without
the domination \(g\) need not be bounded, because \(\kappa\) may be small
in directions where \(\ell_i\) is not.  The two forms of the condition
are equivalent because \(|\ell_i(\xi)|^2\le C_i\|S^{1/2}\xi\|^2\) says
that \(\ell_i\) factors through \(S^{1/2}\) as a bounded functional on the
closure of its range, which the Riesz representation writes as
\(\langle v_i,S^{1/2}\xi\rangle\).  Since \(\kappa(\omega a)=\omega^2\kappa(a)\)
and \(\ell_i(\omega a)=\omega\ell_i(a)\),
\(\partial_\omega^k\bigl[q(\omega a)e^{-\omega^2\kappa(a)/2}\bigr]\)
is a polynomial in \(\omega\), \(\kappa(a)\), and \(\ell_i(a)\), times
\(e^{-\omega^2\kappa(a)/2}\); on \(I\subset\{r\le|\omega|\le R\}\)
it is bounded by \(C_k(1+\|a\|)^{p_k}e^{-r^2\theta\langle Qa,a\rangle/2}\),
and \cref{lem:gaussian-decay} integrates this against
\((1+\|a\|)^{m+2}\nu_\alpha(\dd a)\).  The same bound holds with \(I\)
replaced by any compact set of frequencies, in particular by a compact
neighbourhood of \(I\), with \(r\) the smallest \(|\omega|\) on that set.
For (b), if \(g\) vanishes
outside the ball of radius \(R_0\), then \(\partial_\omega^kg(\omega a)\)
vanishes for all \(\omega\in I\) unless \(\|a\|\le R_0/r\), and
\(\nu_\alpha\) is finite on bounded sets; the polynomial bound is
integrable there.  For (b), a finite linear combination is handled by
the triangle inequality in \cref{eq:ray-regularity}.  For the Bochner
integral \(g:=\int_\Omega g_y\,m(\dd y)\), which is bounded by
\(m(\Omega)\sup_y\|g_y\|_\infty\), fix \(a\in \cH\).  The dominating bounds
\(\sup_{\omega\in U}|\partial_\omega^kg_y(\omega a)|\le h_k(a)\) are
finite constants in \(y\), hence \(m\)-integrable, so differentiation under the
integral sign applies on \(U\): the function \(\omega\mapsto g(\omega a)\)
is \(C^\infty\) on the neighbourhood \(U\) of \(I\), with
\(\partial_\omega^kg(\omega a)=\int_\Omega\partial_\omega^kg_y(\omega a)\,m(\dd y)\).
This is exactly what dominating bounds on a neighbourhood provide;
suprema over \(I\) alone, or bounds valid only for \(\nu_\alpha\)-almost
every \(a\), control no difference quotient and leave \(\omega\mapsto
g(\omega a)\) possibly nowhere differentiable.  Taking suprema over
\(I\subset U\) and integrating in \(a\), Tonelli gives
\[
  M_m(g)\le m(\Omega)\int_\cH(1+\|a\|)^{m+2}
  \max_{k\le m}h_k(a)\,\nu_\alpha(\dd a)<\infty
\]
for every \(m\ge0\).
\end{proof}

Homogeneity \cref{eq:homogeneous-change} at a fixed \(\omega\in I\)
shows that \cref{def:ray-regular} implies
\(g\in L^1(\nu_\alpha)\cap L^2(\nu_\alpha)\).
The preceding lemma covers Gaussian-type densities and smooth densities
with bounded support. The coefficient estimate above supplies
integrability before any synthesis identity is invoked.

More precisely, the independent decay estimate of
\cref{lem:coefficient-finite-order} shows that \(C^4\) regularity of \(g_a\)
and the weighted integral \(A_{4,2}(g)<\infty\) defined there suffice
for the second parameter moment.  The full condition in \cref{def:ray-regular} supplies
moments of every order and will also justify the integral representation
with continuous activations of polynomial growth.

The next theorem is the representation theorem announced in the
introduction, the first main result of this section.

\begin{thm}[Integral representation of targets with a spectral density]
\label{thm:A}
Let \(\alpha>0\).  In part (ii), let \(\rho\) be an
\(\alpha\)-admissible Schwartz filter; in part (iii), let it be band pass.
\begin{enumerate}[label=\textup{(\roman*)}]
\item
For \(g\in L^1(\nu_\alpha)\), the function \(f_g\) is bounded and
continuous, \(\|f_g\|_\infty\le\|g\|_{L^1(\nu_\alpha)}\), and \(f_g=0\)
only if \(g=0\) \(\nu_\alpha\)-almost everywhere.
\item
For \(g\in L^1(\nu_\alpha)\cap L^2(\nu_\alpha)\) and every \(x\in \cH\),
the iterated integral
\begin{equation}
  \int_\cH\left[\int_\R\gamma_g(a,b)\,\rho(\langle a,x\rangle-b)\,\dd b\right]
  \nu_\alpha(\dd a)
  =\iiprod{\rho,\rho}_{\alpha}\,f_g(x)
  \label{eq:spectral-synthesis}
\end{equation}
converges absolutely.  If moreover \(\gamma_g\in L^1(\lambda_\alpha)\),
the left side is the integral network \(S_\rho[\gamma_g\lambda_\alpha](x)\)
of \cref{def:integral-network}.
\item
Let \(\sigma\in\cS'(\R)\) be a continuous function of polynomial growth,
and let \(g\) satisfy \cref{def:ray-regular}.  For every \(x\in \cH\),
\(\gamma_g(a,b)\sigma(\langle a,x\rangle-b)\) is absolutely integrable
in the product measure \(\lambda_\alpha\).  Thus the following is an
ordinary integral network with the finite coefficient measure
\(\gamma_g\lambda_\alpha\):
\begin{equation}
 \begin{aligned}
  \int_\cH\left[\int_\R\gamma_g(a,b)\,\sigma(\langle a,x\rangle-b)\,\dd b\right]
  \nu_\alpha(\dd a)
  &=\iiprod{\sigma,\rho}_{\alpha}\,f_g(x),\\
  \iiprod{\sigma,\rho}_{\alpha}
  &:=\frac1{2\pi}\bigl\langle\sigma^\sharp,
  \rho^\sharp(-\,\cdot\,)|\cdot|^{-\alpha}\bigr\rangle .
 \end{aligned}
  \label{eq:tempered-spectral-synthesis}
\end{equation}
The constant is a pairing with a test function supported in
\(-\supp\rho^\sharp\).  The identity allows
\(\iiprod{\sigma,\rho}_{\alpha}=0\).  If \(\sigma\) is not a polynomial,
a band-pass \(\rho\) with nonzero constant exists; this separate
condition is needed for normalization and universality.
\end{enumerate}
\end{thm}

\begin{proof}
\label{app:proof-thm-A}
\emph{(i).}  Apply \cref{lem:spectral-target-basic} with \(\nu=\nu_\alpha\)
and \(\outY=\C\).

\emph{(ii).}  Fix \(x\) and write \(u=\langle a,x\rangle\).  By
\cref{lem:coefficient-isometry}, for \(\nu_\alpha\)-almost every \(a\)
the function \(\gamma_g(a,\cdot)\) belongs to \(L^2(\R)\) with Fourier
transform \(h_a(\omega)=\rho^\sharp(-\omega)g(\omega a)\), which is
integrable and square integrable; compact support is not needed.
Parseval's identity, applied to
\(\gamma_g(a,\cdot)\) and to \(b\mapsto\rho(u-b)\), whose Fourier
transform is \(e^{-i\omega u}\rho^\sharp(-\omega)\), gives
\begin{equation}
  \int_\R\gamma_g(a,b)\rho(u-b)\,\dd b
  =\frac1{2\pi}\int_\R\rho^\sharp(-\omega)g(\omega a)\,
  \rho^\sharp(\omega)\,e^{i\omega u}\,\dd\omega ,
  \label{eq:bias-parseval}
\end{equation}
where we used that \(\rho\) is real, so that the conjugate of the Fourier
transform of \(b\mapsto\rho(u-b)\) at \(\omega\) equals
\(e^{i\omega u}\rho^\sharp(\omega)\).  The right side is bounded by
\(\frac1{2\pi}\int|\rho^\sharp(\omega)|^2|g(\omega a)|\dd\omega\),
whose \(\nu_\alpha\)-integral equals
\(\frac1{2\pi}\int|\rho^\sharp(\omega)|^2|\omega|^{-\alpha}\dd\omega\,
\|g\|_{L^1(\nu_\alpha)}<\infty\) by \cref{eq:homogeneous-change}.  This
proves absolute convergence of the iterated integral and justifies
Fubini in
\[
  \int_\cH\int_\R\gamma_g(a,b)\rho(u-b)\,\dd b\,\nu_\alpha(\dd a)
  =\frac1{2\pi}\int_\R\rho^\sharp(\omega)\rho^\sharp(-\omega)
  \left[\int_\cH g(\omega a)e^{i\omega\langle a,x\rangle}\nu_\alpha(\dd a)\right]
  \dd\omega .
\]
With \(\xi=\omega a\), homogeneity turns the bracket into
\(|\omega|^{-\alpha}f_g(x)\).  The remaining frequency integral is
\[
 \frac1{2\pi}\int_\R\rho^\sharp(\omega)\rho^\sharp(-\omega)
       |\omega|^{-\alpha}\,\dd\omega=\iiprod{\rho,\rho}_{\alpha},
\]
because \(\rho\) is real.  If
\(\gamma_g\in L^1(\lambda_\alpha)\), the double integral converges
absolutely since \(\rho\) is bounded, and the iterated integral is the
integral network.

\emph{(iii).} Let \(g\) satisfy \cref{def:ray-regular}.  Choose an integer \(r\)
at least as large as a polynomial growth order of \(\sigma\).
By \cref{lem:coefficient-finite-order},
\[
 \int_{\cH\times\R}|\gamma_g(a,b)|
       |\sigma(\langle a,x\rangle-b)|\,\dd\lambda_\alpha
 \le C_{\sigma,x,r}\int(1+\|a\|+|b|)^r|\gamma_g|\,\dd\lambda_\alpha<\infty.
\]
Thus the ordinary integral network already exists, independently of
its identification below.  For every \(a\), the function
\(\omega\mapsto h_a(\omega)=\rho^\sharp(-\omega)g(\omega a)\) is
\(C_c^\infty\) with support in \(I\), so \(\gamma_g(a,\cdot)\in\cS(\R)\)
for every \(a\), and \(\int|\gamma_g(a,b)||\sigma(u-b)|\dd b<\infty\)
because \(\sigma\) has polynomial growth.  For \(\gamma\in\cS(\R)\) and
\(\sigma\in\cS'(\R)\), the translation identity
\(\gamma(u-t)=\frac1{2\pi}\int\gamma^\sharp(\omega)e^{i\omega u}e^{-i\omega t}\dd\omega\)
and the definition of the distributional Fourier transform give
\begin{equation}
  \int_\R\gamma(b)\sigma(u-b)\,\dd b
  =\frac1{2\pi}\bigl\langle\sigma^\sharp,\
  \omega\mapsto\gamma^\sharp(\omega)e^{i\omega u}\bigr\rangle .
  \label{eq:tempered-bias-pairing}
\end{equation}
Applied to \(\gamma=\gamma_g(a,\cdot)\), the test function is
\(\varphi_{a}(\omega):=\rho^\sharp(-\omega)g(\omega a)e^{i\omega\langle a,x\rangle}\),
supported in \(-\supp\rho^\sharp\subset I\).

A tempered distribution has finite order: there are \(m\) and \(C\)
with \(|\langle\sigma^\sharp,\varphi\rangle|\le C\|\varphi\|_{C^m}\) for
every \(\varphi\in C_c^\infty(\R)\) supported in the fixed compact set
\(I\).  Let \(I'\) be a compact neighbourhood of \(I\) in
\(\R\setminus\{0\}\), and let \(X\) be the closure in \(C^m(\R)\) of
\(C_c^\infty(\operatorname{int}I')\).  Then \(X\) is a separable Banach
space containing every \(C^m\) function supported in \(I\) (mollify), the
pairing with \(\sigma^\sharp\) extends to a bounded functional on \(X\),
and evaluation at each \(\omega\) is bounded on \(X\).  The map
\(a\mapsto\varphi_a\in X\) is weakly measurable, because
\((a,\omega)\mapsto\partial_\omega^kg(\omega a)\) is Borel as a pointwise
limit of difference quotients of the Borel function
\((a,\omega)\mapsto g(\omega a)\); by Pettis' theorem it is strongly
measurable.  Leibniz' rule and \(|\langle a,x\rangle|\le\|a\|\|x\|\) give
\[
  \|\varphi_a\|_{C^m}
  \le C_{\rho,m}(1+\|x\|)^m(1+\|a\|)^m
  \max_{k\le m}\sup_{\omega\in I}|\partial_\omega^kg(\omega a)|,
\]
which is \(\nu_\alpha\)-integrable by \cref{eq:ray-regularity}.  Hence
\(a\mapsto\varphi_a\) is Bochner integrable in \(X\), and the bounded
functional \(\langle\sigma^\sharp,\cdot\rangle\) commutes with the
integral:
\[
  \int_\cH\int_\R\gamma_g(a,b)\sigma(\langle a,x\rangle-b)\,\dd b\,\nu_\alpha(\dd a)
  =\frac1{2\pi}\Bigl\langle\sigma^\sharp,\ \int_\cH\varphi_a\,\nu_\alpha(\dd a)\Bigr\rangle .
\]
Evaluating the Bochner integral at \(\omega\) and substituting
\(\xi=\omega a\) in \cref{eq:homogeneous-change},
\[
  \Bigl(\int_\cH\varphi_a\,\nu_\alpha(\dd a)\Bigr)(\omega)
  =\rho^\sharp(-\omega)\int_\cH g(\omega a)e^{i\langle\omega a,x\rangle}\nu_\alpha(\dd a)
  =\rho^\sharp(-\omega)|\omega|^{-\alpha}f_g(x),
\]
so the right side equals
\(f_g(x)\frac1{2\pi}\langle\sigma^\sharp,\rho^\sharp(-\cdot)|\cdot|^{-\alpha}\rangle
=\iiprod{\sigma,\rho}_{\alpha}f_g(x)\).  Absolute convergence of the outer
integral follows from the same bound on \(\|\varphi_a\|_{C^m}\).

Finally, if \(\langle\sigma^\sharp,\rho^\sharp(-\cdot)|\cdot|^{-\alpha}\rangle=0\)
for every real band-pass \(\rho\), then, since every
\(\varphi\in C_c^\infty(\R\setminus\{0\})\) is a combination
\(\varphi_1+i\varphi_2\) of Hermitian-symmetric functions
\(\varphi_j\) and each \(\varphi_j(-\omega)|\omega|^{\alpha}\) is the
Fourier transform of a real band-pass filter, \(\sigma^\sharp\) vanishes
on \(C_c^\infty(\R\setminus\{0\})\).  Then \(\supp\sigma^\sharp\subset\{0\}\),
so \(\sigma^\sharp\) is a finite combination of derivatives of
\(\delta_0\) and \(\sigma\) is a polynomial.
\end{proof}

No Gaussian input measure enters \cref{thm:A}: the representation uses
only the sigma-finiteness and homogeneity of \(\nu_\alpha\).  The input
measure is needed to \emph{compute} a spectral density from a given
function \(f\), which is the content of the ridgelet transform and of
the next subsection.  If \(\sigma\) is a polynomial, its Fourier
transform is supported at zero, so the band-pass constant in (iii)
vanishes; no nontrivial reconstruction is then claimed.

\subsection{The frame operator and the reconstruction formula}
\label{subsec:frame}

Let \(\cE_\alpha'\) be the continuous anti-dual of \(\cE_\alpha\).
Define the frame operator \(T_\alpha\) as the Riesz isomorphism,
and define the anti-dual transpose of the extended \(F_Q\) by
\begin{align}
  T_\alpha f[g]
  &:=\langle f,g\rangle_{\cE_\alpha},
  \qquad f,g\in\cE_\alpha,
  \label{eq:riesz-map}\\
  F_Q'G[g]
  &:=\langle G,F_Q g\rangle_{L^2(\nu_\alpha)},
  \qquad G\in L^2(\nu_\alpha),\ g\in\cE_\alpha .
  \label{eq:transpose-analysis}
\end{align}
Thus \(F_Q':L^2(\nu_\alpha)\to\cE_\alpha'\); the prime denotes
the anti-dual transpose formed using the completed input space.
The Hilbert adjoint of the unitary map
\(F_Q:\cE_\alpha\to\cK_\alpha\) takes values in
\(\cE_\alpha\). These constructions assert no everywhere-defined
adjoint into the ambient \(L^2(\mu_Q)\). Synthesis with the analysis
filter is the transpose of the transform,
\begin{equation}
  S_\rho:=R_\rho':L^2(\lambda_\alpha)\to\cE_\alpha',
  \qquad
  (S_\rho\gamma)[g]:=\langle\gamma,R_\rho g\rangle_{L^2(\lambda_\alpha)} .
  \label{eq:weak-synthesis}
\end{equation}
When \(\gamma\lambda_\alpha\) is a finite measure with the moment
\cref{eq:measure-moment}, \(S_\rho\gamma\) is the integral network
\(S_\rho[\gamma\lambda_\alpha]\) of \cref{def:integral-network} paired with
\(g\) through \(\mu_Q\) (\cref{lem:weak-equals-strong}); the notation is
therefore consistent.

The term \emph{frame operator} refers to the dual-space map
\(T_\alpha:\cE_\alpha\to\cE_\alpha'\). If \(R_\rho^*\) denotes the Hilbert
adjoint into \(\cE_\alpha\), then
\[
 R_\rho^*R_\rho=C\Id_{\cE_\alpha},
 \qquad R_\rho'R_\rho=CT_\alpha,
 \qquad C=\iiprod{\rho,\rho}_{\alpha},
\]
by \cref{thm:B} and the definition of the transpose.  This distinguishes
the Hilbert adjoint from synthesis taking values in \(\cE_\alpha'\).
We first establish two facts used in reconstruction: the agreement
of weak and integral synthesis, and recovery from Hermite coefficients.

\begin{lem}[Synthesis of an integrable coefficient is the integral network]
\label{lem:weak-equals-strong}
Let \(\rho\in\cS(\R)\) be real and let \(\gamma\in L^1(\lambda_\alpha)\cap L^2(\lambda_\alpha)\).
Then the integral network \(S_\rho[\gamma\lambda_\alpha]\) of
\cref{def:integral-network} is a bounded Borel function on \(\cH\), and for
every \(g\in\cD_\alpha\)
\[
  (S_\rho\gamma)[g]
  =\int_\cH S_\rho[\gamma\lambda_\alpha](x)\,\overline{g(x)}\,\mu_Q(\dd x),
\]
with \(S_\rho\gamma\) the synthesis operator \cref{eq:weak-synthesis}.
\end{lem}

\begin{proof}
Boundedness follows from
\(|S_\rho[\gamma\lambda_\alpha]|\le\|\gamma\|_1\|\rho\|_\infty\).
The absolute integral of the product in the synthesis pairing is at
most \(\|\gamma\|_1\|\rho\|_\infty\|g\|_{L^1(\mu_Q)}\), which is finite.
Since \(\rho\) is real, Fubini therefore gives
\begin{align*}
 \langle\gamma,R_\rho g\rangle_{L^2(\lambda_\alpha)}
 &=\int_{\cH\times\R}\gamma(a,b)
   \overline{\int_\cH g(x)\rho(\langle a,x\rangle-b)\,\mu_Q(\dd x)}
      \,\lambda_\alpha(\dd a,\dd b)\\
 &=\int_\cH\overline{g(x)}S_\rho[\gamma\lambda_\alpha](x)\,\mu_Q(\dd x).\qedhere
\end{align*}
\end{proof}

\begin{lem}[Entire extension and totality of the Hermite coefficients]
\label{lem:hermite-totality}
For \(f\in L^2(\mu_Q)\) and \(\xi\ne0\), put
\(\tau(\xi)=\langle Q\xi,\xi\rangle^{1/2}\) and
\(E_f(z\xi):=e^{z^2\tau(\xi)^2/2}F_Qf(z\xi)\) for \(z\in\C\).  Then
\(z\mapsto E_f(z\xi)\) is entire,
\[
 E_f(z\xi)=\sum_{n=0}^\infty\frac{(-iz\tau(\xi))^n}{n!}
  \EE_{\mu_Q}\!\left[f\,\mathrm{He}_n
       (\langle x,\xi\rangle/\tau(\xi))\right],
\]
the series converges locally uniformly,
\(|E_f(z\xi)|\le\|f\|_{L^2(\mu_Q)}e^{|z|^2\tau(\xi)^2/2}\), and
the following inversion formula holds:
\begin{equation}
  \EE_{\mu_Q}\!\left[
    f(x)\,\mathrm{He}_n\!\left(
      \frac{\langle x,\xi\rangle}{\tau(\xi)}\right)\right]
  =\frac{i^n}{\tau(\xi)^n}
   \frac{\dd^n}{\dd t^n}
   \Bigl(e^{t^2\tau(\xi)^2/2}F_Qf(t\xi)\Bigr)\Big|_{t=0},
  \qquad \tau(\xi):=\langle Q\xi,\xi\rangle^{1/2},
  \label{eq:hermite-inversion}
\end{equation}  The coefficients, over all \(\xi\)
and \(n\), determine \(f\) in \(L^2(\mu_Q)\).
\end{lem}

\begin{proof}
For a standard normal \(Z\), the generating function
\(e^{tZ-t^2/2}=\sum_n\mathrm{He}_n(Z)t^n/n!\) converges in \(L^2\) for each
complex \(t\), locally uniformly, since orthogonality gives the squared
norm sum \(\sum_n|t|^{2n}/n!=e^{|t|^2}\).  Take
\(Z=\langle x,\xi\rangle/\tau(\xi)\) and \(t=-iz\tau(\xi)\); pairing with
\(f\) gives the series and entireness, Cauchy--Schwarz gives the bound,
and termwise differentiation proves \cref{eq:hermite-inversion}.

For totality, take an eigenbasis \(Qe_j=q_je_j\), \(q_j>0\), and the
independent coordinates \(Z_j=\langle x,e_j\rangle/\sqrt{q_j}\).  Finite
products \(\prod_j\mathrm{He}_{n_j}(Z_j)\) form a complete orthogonal
system in \(L^2(\mu_Q)\): they are complete at each finite-dimensional
Gaussian level, and conditional expectations onto the first \(m\)
coordinates converge in \(L^2\) \citep[Chapter~2]{Janson1997book}.  For every
fixed total degree, polarization of the Wick power of \(\sum_jt_jZ_j\)
expresses these products as linear combinations of the directional Wick
powers \(\mathrm{He}_n(\langle x,\xi\rangle/\tau(\xi))\) with
\(\xi=\sum_jt_j q_j^{-1/2}e_j\) of finite support.  Thus all orthogonal
expansion coefficients of \(f\) are recovered.
\end{proof}

We can now state the reconstruction theorem, a central result of the paper.

\begin{thm}[Reconstruction and the frame operator]
\label{thm:C}
Let \(\alpha>0\) and let \(\rho\) be an \(\alpha\)-admissible Schwartz
filter.
\begin{enumerate}[label=\textup{(\roman*)}]
\item[\textup{(i)}]
\emph{Frame identity.}
The frame operator satisfies \(T_\alpha=F_Q'F_Q\); it is
an isometric bijection
\(\cE_\alpha\to\cE_\alpha'\), and
\begin{equation}
  S_\rho R_\rho f=\iiprod{\rho,\rho}_{\alpha}\,T_\alpha f,
  \qquad f\in\cE_\alpha .
  \label{eq:frame-identity}
\end{equation}
\item[\textup{(ii)}]
\emph{Two reconstruction formulas.}
For \(f\in\cE_\alpha\) and \(g\in\cE_\alpha'\),
\begin{equation}
  f=\bigl(\iiprod{\rho,\rho}_{\alpha}\bigr)^{-1}T_\alpha^{-1}S_\rho R_\rho f,
  \qquad
  g=\bigl(\iiprod{\rho,\rho}_{\alpha}\bigr)^{-1}S_\rho\bigl(R_\rho T_\alpha^{-1}g\bigr).
  \label{eq:main-reconstruction}
\end{equation}
\item[\textup{(iii)}]
\emph{The frame operator in terms of the spectral density.}
If \(f\in\cD_\alpha\) and \(F_Qf\in L^1(\nu_\alpha)\), then \(T_\alpha f\)
is represented by the bounded continuous function \(f_{F_Qf}\), in the
sense that \(T_\alpha f[v]=\int_\cH f_{F_Qf}(x)\overline{v(x)}\,\mu_Q(\dd x)\)
for \(v\in\cD_\alpha\).  Conversely, if \(g\in\cK_\alpha\), then
\(F_Q'g\in\cE_\alpha'\) and \(R_\rho T_\alpha^{-1}F_Q'g=W_\rho g\);
when \(g\in L^1(\nu_\alpha)\), \(F_Q'g\) is represented by \(f_g\) and
the second formula in \cref{eq:main-reconstruction} is
\cref{eq:spectral-synthesis}.
\item[\textup{(iv)}]
\emph{Reconstruction by backprojection.}
The bounded adjoint \(W_\rho^*\) of
\cref{lem:coefficient-adjoint}, given by \cref{eq:ray-average}, satisfies
\(W_\rho^*R_\rho f=\iiprod{\rho,\rho}_{\alpha}F_Q f\) in
\(L^2(\nu_\alpha)\) for every \(f\in\cE_\alpha\).
For \(f\in\cD_\alpha\), using the continuous Fourier-slice
representative makes this identity pointwise in \(\xi\).
The following final step uses the Gaussian input specifically:
\(f\) is recovered from \(F_Qf\) by the Hermite coefficients in
\cref{eq:hermite-inversion}, which determine \(f\) in
\(L^2(\mu_Q)\).  We write \(\Delta_Q\) for this inverse of \(F_Q\) on its
range on \(\cD_\alpha\), so that
\(f=\Delta_Q\bigl[(\iiprod{\rho,\rho}_{\alpha})^{-1}W_\rho^*R_\rho f\bigr]\).
\end{enumerate}
\end{thm}

\begin{proof}
\label{app:proof-thm-C}
Since \(F_Q\) is unitary onto \(\cK_\alpha\),
\(F_Q'F_Q f[g]=\langle F_Q f,F_Q g\rangle
=\langle f,g\rangle_{\cE_\alpha}=T_\alpha f[g]\), and the Riesz
representation theorem makes \(T_\alpha\) an isometric bijection.  By
\cref{eq:weak-synthesis} and \cref{eq:ridgelet-plancherel},
\((S_\rho R_\rho f)[g]=\langle R_\rho f,R_\rho g\rangle
=\iiprod{\rho,\rho}_{\alpha}T_\alpha f[g]\), which is (i); (ii) follows by applying
\(T_\alpha^{-1}\), respectively by substituting \(f=T_\alpha^{-1}g\).

\emph{(iii).} Let \(f\in\cD_\alpha\), put \(g=F_Qf\in L^1(\nu_\alpha)\),
and take \(v\in\cD_\alpha\).  The product
\(|g(\xi)||v(x)|\) is integrable against \(\nu_\alpha\otimes\mu_Q\), so
Fubini gives
\[
 \int_\cH f_g(x)\overline{v(x)}\,\mu_Q(\dd x)
 =\int_\cH g(\xi)\overline{F_Qv(\xi)}\,\nu_\alpha(\dd\xi)
 =\langle f,v\rangle_{\cE_\alpha}=T_\alpha f[v].
\]
Conversely, for \(g\in\cK_\alpha\) set
\(u=T_\alpha^{-1}F_Q'g\).  For every \(v\in\cE_\alpha\),
\[
 \langle F_Q u,F_Q v\rangle
 =\langle u,v\rangle_{\cE_\alpha}
 =F_Q'g[v]=\langle g,F_Q v\rangle.
\]
Unitarity onto \(\cK_\alpha\) implies \(F_Q u=g\), and hence
\(R_\rho T_\alpha^{-1}F_Q'g=W_\rho g\).
If also \(g\in L^1(\nu_\alpha)\), the same Fubini computation identifies
\(F_Q'g\) with the functional represented by \(f_g\).
Thus \(F_Q'g=C^{-1}S_\rho W_\rho g\); when
\(\gamma_g\in L^1(\lambda_\alpha)\), \cref{lem:weak-equals-strong}
identifies this with ordinary synthesis paired against \(v\).

\emph{(iv).}  Boundedness of \(W_\rho^*\) and the identity
\(W_\rho^*W_\rho=C\Id\) are proved in
\cref{lem:coefficient-adjoint}.  For \(f\in\cD_\alpha\), the
Fourier-slice representative \cref{eq:fourier-slice} is continuous, and
\((R_\rho f)^\sharp(\xi/\omega,\omega)=\rho^\sharp(-\omega)F_Qf(\xi)\)
for every \(\xi\) and \(\omega\ne0\); inserting this into
\cref{eq:ray-average} gives
\(W_\rho^*R_\rho f(\xi)=\frac1{2\pi}\int_{\R\setminus\{0\}}|\rho^\sharp|^2|\omega|^{-\alpha}\dd\omega\,F_Qf(\xi)
=\iiprod{\rho,\rho}_{\alpha}F_Qf(\xi)\) pointwise.  The Hermite formula and the
totality statement are \cref{lem:hermite-totality}.
\end{proof}

\begin{cor}[Stability under coefficient perturbations]
\label{cor:coefficient-stability}
Let \(\rho\) be \(\alpha\)-admissible, write
\(C=\iiprod{\rho,\rho}_{\alpha}\), and define the bounded decoder
\(D_\rho:=C^{-1}T_\alpha^{-1}S_\rho\).
Then \(D_\rho R_\rho=\Id\) and
\(\|D_\rho\|\le C^{-1/2}\).  In particular, if
\(f\in\cE_\alpha\), \(\gamma_\delta\in L^2(\lambda_\alpha)\), and
\(\|\gamma_\delta-R_\rho f\|_2\le\delta\), then
\[
 \|D_\rho\gamma_\delta-f\|_{\cE_\alpha}\le\delta/\sqrt C.
\]
\end{cor}
\begin{proof}
The left-inverse identity is \cref{thm:C}(ii).
The transpose \(S_\rho=R_\rho'\) has norm at most \(\sqrt C\), and
\(T_\alpha^{-1}\) is an isometry, so
\(\|D_\rho\|\le C^{-1}\sqrt C\).  Apply this bound to
\(\gamma_\delta-R_\rho f\).
\end{proof}

This stability estimate is in the norm of \(\cE_\alpha\); it does not
give an ambient \(L^2(\mu_Q)\) or pointwise error estimate.

The frame correction can therefore be undone after synthesis, before
analysis, or through the coefficients using backprojection and Hermite
expansion.  For approximation we use the density formulation directly:
\(W_\rho g\) represents the spectral target \(f_g\).  The finite-dimensional
weighted case, where the correction is a familiar Fourier multiplier,
is recorded separately in \cref{cor:finite-backprojection}.

\begin{rmk}[Which inverse is bounded]
\label{rmk:inverse-norm}
\(T_\alpha^{-1}:\cE_\alpha'\to\cE_\alpha\) is bounded with norm one.  No
bounded inverse in the ambient \(L^2(\mu_Q)\)-norm is asserted; the two
statements differ because \(\cE_\alpha\) is completed in the spectral
norm. On \(f\in\cD_\alpha\) the inverse factorizes as
\(T_\alpha f\mapsto F_Qf\mapsto f\mu_Q\mapsto f\), through the unitary
coordinate, Fourier--Stieltjes uniqueness, and the Radon--Nikodym
derivative; the Hermite expansion \cref{eq:hermite-inversion} implements
the last two steps constructively.
\end{rmk}

\subsection{Vector-valued targets}
\label{subsec:vector-valued}

Let \(\outY\) be a separable complex Hilbert space.  All objects above have
\(\outY\)-valued versions: \(L^2(\mu_Q;\outY)\), \(F_Qf(\xi)=\int f(x)e^{-i\langle
x,\xi\rangle}\mu_Q(\dd x)\in \outY\) as a Bochner integral, the core
\(\cD_\alpha(\outY)\), the completion \(\cE_\alpha(\outY)\) with inner product
\(\int\langle F_Qf,F_Qg\rangle_\outY\dd\nu_\alpha\), the transform
\(R_\rho f\in L^2(\lambda_\alpha;\outY)\), the coefficient \(W_\rho g\) of a
density \(g\in L^2(\nu_\alpha;\outY)\), the target \(f_g\), and the backprojection.

\begin{thm}[Vector-valued extension]
\label{thm:vector-valued}
\Cref{thm:A,thm:B,thm:C} hold for \(\outY\)-valued targets, with the same
constants, with absolute values replaced by norms in \(\outY\), scalar
integrals by Bochner integrals, and \(L^2\) spaces by their \(\outY\)-valued
counterparts.  The Riesz map and its inverse remain isometries;
their operator norms are one if \(\outY\ne\{0\}\) and zero if \(\outY=\{0\}\).
\end{thm}

\begin{proof}
Use \cref{lem:partial-fourier-l2} for bias Fourier representatives and
Plancherel, \cref{lem:spectral-target-basic} for spectral synthesis and
uniqueness, and \cref{lem:coefficient-finite-order} for joint absolute
integrability and coefficient moments.  The proof of
\cref{lem:coefficient-adjoint} applies with Hilbert norms and Bochner
integrals.  These facts justify the Fubini and Parseval steps of
\cref{thm:A,thm:B} with the same constants.  Completing the vector core
and applying Riesz representation on \(\cE_\alpha(\outY)\) proves
\cref{thm:C}; its Gaussian Hermite formula is applied componentwise.
If \(\outY\ne\{0\}\), any nonzero constant \(\outY\)-valued function belongs
to the core by \cref{lem:gaussian-decay}, so the isometries have norm
one; if \(\outY=\{0\}\), both spaces and operator norms are zero.
The regularity condition in \cref{def:ray-regular} uses \(\|\cdot\|_\outY\)
in place of absolute value.
For the integral representation with a continuous activation of polynomial
growth, the fixed-support \(C^m\) Bochner argument of
the proof of \cref{thm:A} works in \(\outY\)-valued test functions, with the
distributional pairing tensored with the identity of \(\outY\).
\end{proof}

For \(\outY\)-valued networks the coefficient \(W_\rho g\) takes values in
\(\outY\); sampling it is discussed in \cref{subsec:vector-sampling}.
 \section{Reconstruction formulas and activation functions}
\label{sec:tempered}
\label{app:tempered-proofs}

We extend reconstruction to tempered activations and give conditions
under which synthesis is an absolutely convergent network integral.
ReLU and explicit filters connect these formulas to the network
approximation results that follow.

The analysis filter \(\rho\) is a Schwartz function, but the activation
that synthesizes a network may be unbounded.  Following the distributional
ridgelet theory of \citet{Sonoda2017acha} and the weighted Sobolev
formulation of \citet{Sonoda2021ghost}, we take
\(\sigma\in\cS'(\R)\) real valued for synthesis and pair its Fourier
transform with the band-pass filter.  Since \(\rho^\sharp\) vanishes
near the origin, the pairing
\begin{equation}
  \iiprod{\sigma,\rho}_{\alpha}
  =\frac1{2\pi}
  \left\langle\sigma^\sharp,
  \rho^\sharp(-\,\cdot\,)|\,\cdot\,|^{-\alpha}\right\rangle
  \label{eq:distributional-admissibility}
\end{equation}
of \cref{thm:A}(iii) is well defined for every tempered \(\sigma\), and no
value of \(\sigma^\sharp\) at the singular point \(0\) is used.  The
weighted Sobolev activation spaces
\(\cA_{s,t}=\langle\cdot\rangle^tH^s(\R)\) of
\citet{Sonoda2021ghost}, in which this pairing is continuous, are
defined below.

\subsection{Weighted Sobolev activation spaces}
\label{app:sobolev}

Write \(\langle u\rangle=(1+u^2)^{1/2}\) and define the Bessel operator on
the frequency variable by
\(B^qv:=F_\R\bigl[\langle\cdot\rangle^qF_\R^{-1}v\bigr]\).  For
\(s\in\R\), \(t\ge0\), the weighted Sobolev activation space of
\citet{Sonoda2021ghost} is
\begin{equation}
  \cA_{s,t}:=\langle\,\cdot\,\rangle^tH^s(\R)\subset\cS'(\R),
  \qquad
  \|\sigma\|_{\cA_{s,t}}
  :=\|\langle\omega\rangle^sB^{-t}\sigma^\sharp\|_{L^2},
  \label{eq:activation-space}
\end{equation}
and the dual test norm is
\(\|r\|_{\mathcal T^\sharp_{s,t}}:=\|\langle\omega\rangle^{-s}B^tr\|_{L^2}\)
for \(r\in\cS(\R)\).  Membership in \(\cA_{s,t}\) controls the
distributional pairing; it does not by itself provide a continuous
pointwise representative.

\begin{lem}[Hilbert structure and continuous activation pairing]
\label{lem:weighted-duality}
The map \(\sigma\mapsto\langle\omega\rangle^sB^{-t}\sigma^\sharp\) is an
isometric isomorphism \(\cA_{s,t}\to L^2(\R)\), and
\begin{equation}
 \left|\frac1{2\pi}\langle\sigma^\sharp,r\rangle\right|
 \le\frac1{2\pi}\|\sigma\|_{\cA_{s,t}}\|r\|_{\mathcal T^\sharp_{s,t}} .
 \label{eq:weighted-duality}
\end{equation}
Thus the pairing extends to the completion of the test filters in
\(\mathcal T^\sharp_{s,t}\).
\end{lem}

\begin{proof}
For \(g\in L^2\), \(\sigma=\langle\cdot\rangle^tF_\R^{-1}[\langle\omega\rangle^{-s}g]\)
is a preimage, proving surjectivity and completeness.  The multiplier
\(\langle u\rangle^t\) is real and even, so \(B^t\) is symmetric for the
bilinear pairing, and for Schwartz \(\sigma,r\)
\(\langle\sigma^\sharp,r\rangle=\int(\langle\omega\rangle^sB^{-t}\sigma^\sharp)
(\langle\omega\rangle^{-s}B^tr)\dd\omega\).  Schwartz functions are dense
in \(\cA_{s,t}\) by the isomorphism, and convergence there implies
convergence in \(\cS'\); the identity extends by density and
Cauchy--Schwarz proves the bound.
\end{proof}

\begin{lem}[Standard activations and admissible test filters]
\label{lem:standard-activation-class}
ReLU, \(\tanh\), the Gaussian distribution function, and \(e^{-u^2/2}\)
belong to \(\cA_{0,2}\), are globally Lipschitz, and are not polynomials.
For every non-polynomial real \(\sigma\in\cS'\) there is a real band-pass
\(\rho\) with \(\iiprod{\sigma,\rho}_{\alpha}=1\).
\end{lem}

\begin{proof}
Membership in \(\cA_{0,2}\) means \(\langle u\rangle^{-2}\sigma\in L^2\);
three of the functions are bounded, and ReLU satisfies
\(\int_0^\infty u^2(1+u^2)^{-2}\dd u<\infty\).  Their derivatives are
bounded wherever defined.  The bounded functions are nonconstant and
ReLU is not smooth at zero, so none is a polynomial.  The last statement
is the final part of the proof of \cref{thm:A}, followed by rescaling
\(\rho\).
\end{proof}

\subsection{Reconstruction with tempered activations}
We now use band-pass regularizations to define synthesis in the dual space.

\begin{dfn}[Regularized synthesis]
\label{def:regularized-synthesis}
Let \(\sigma\in\cS'(\R)\) be real and let \(\rho\) be a band-pass filter.
Choose an even \(\chi\in C_c^\infty(\R\setminus\{0\})\) equal to one on a
neighbourhood of \(\supp\rho^\sharp\), and an even, compactly supported,
smooth approximate identity \((\eta_\eps)_{\eps>0}\).  Define the real
Schwartz functions \(\sigma_\eps\) by
\begin{equation}
  \sigma_\eps^\sharp:=\chi\,(\sigma^\sharp*\eta_\eps)
  \in C_c^\infty(\R\setminus\{0\}),
  \label{eq:regularized-activation}
\end{equation}
and, for \(\gamma\in\Ran R_\rho\subset L^2(\lambda_\alpha)\), the
regularized synthesis \(S_{\sigma_\eps}\gamma:=R_{\sigma_\eps}'\gamma\in\cE_\alpha'\)
as in \cref{eq:weak-synthesis}.  The synthesis with \(\sigma\) is
\begin{equation}
  S_\sigma\gamma:=\lim_{\eps\downarrow0}S_{\sigma_\eps}\gamma
  \qquad\text{in }\cE_\alpha' ,
  \label{eq:tempered-synthesis}
\end{equation}
whenever the limit exists.
\end{dfn}

\begin{thm}[Reconstruction with a tempered activation]
\label{thm:tempered-reconstruction}
Let \(\sigma\in\cS'(\R)\) be real and let \(\rho\) be a band-pass filter.
For every \(f\in\cE_\alpha\) the limit \cref{eq:tempered-synthesis}
exists for \(\gamma=R_\rho f\), does not depend on \(\chi\) or
\((\eta_\eps)\), and
\begin{equation}
  S_\sigma R_\rho f
  =\iiprod{\sigma,\rho}_{\alpha}T_\alpha f .
  \label{eq:tempered-frame-identity}
\end{equation}
If \(\iiprod{\sigma,\rho}_{\alpha}\ne0\), then
\begin{equation}
  f=\bigl(\iiprod{\sigma,\rho}_{\alpha}\bigr)^{-1}
  T_\alpha^{-1}S_\sigma R_\rho f,
  \qquad
  g=\bigl(\iiprod{\sigma,\rho}_{\alpha}\bigr)^{-1}
  S_\sigma\bigl(R_\rho T_\alpha^{-1}g\bigr),
  \qquad f\in\cE_\alpha,\ g\in\cE_\alpha' .
  \label{eq:tempered-reconstruction}
\end{equation}
If \(\sigma\) is not a polynomial, equivalently \(\sigma\ne0\) in
\(\cS'/\cP\) with \(\cP\) the polynomials, then a band-pass \(\rho\) with
\(\iiprod{\sigma,\rho}_{\alpha}\ne0\) exists.
\end{thm}

\begin{proof}
\label{app:proof-tempered-reconstruction}
Since \(\sigma\) is real, \(\sigma^\sharp\) is Hermitian symmetric,
\(\sigma^\sharp(-\omega)=\overline{\sigma^\sharp(\omega)}\) in the
distributional sense; as \(\chi\) and \(\eta_\eps\) are even and real,
\(\sigma_\eps^\sharp=\chi(\sigma^\sharp*\eta_\eps)\) is Hermitian
symmetric, so \(\sigma_\eps\) is a real Schwartz function with
\(\sigma_\eps^\sharp\in C_c^\infty(\R\setminus\{0\})\).  Hence
\(\sigma_\eps\) is \(\alpha\)-admissible whenever it is nonzero, and
\cref{thm:B}(i) applies to the pair \((\rho,\sigma_\eps)\).  With
\(S_{\sigma_\eps}=R_{\sigma_\eps}'\) as in \cref{eq:weak-synthesis}, for
\(f\in\cD_\alpha\) and \(g\in\cE_\alpha\),
\[
  (S_{\sigma_\eps}R_\rho f)[g]
  =\langle R_\rho f,R_{\sigma_\eps}g\rangle_{L^2(\lambda_\alpha)}
  =\frac1{2\pi}\int_\R\rho^\sharp(\omega)\overline{\sigma_\eps^\sharp(\omega)}
       |\omega|^{-\alpha}\,\dd\omega\ \langle f,g\rangle_{\cE_\alpha}
  =\iiprod{\sigma_\eps,\rho}_{\alpha}\,T_\alpha f[g],
\]
where the last equality uses Hermitian symmetry,
\(\overline{\sigma_\eps^\sharp(\omega)}=\sigma_\eps^\sharp(-\omega)\),
and the substitution \(\omega\mapsto-\omega\).  Both sides are
continuous in \(f\in\cE_\alpha\), so the identity holds on
\(\cE_\alpha\).  Since \(\chi=1\) on a neighbourhood of
\(\supp\rho^\sharp\), the pairing defining
\(\iiprod{\sigma_\eps,\rho}_{\alpha}\) equals
\(\frac1{2\pi}\langle\sigma^\sharp*\eta_\eps,\rho^\sharp(-\cdot)|\cdot|^{-\alpha}\rangle\),
and distributional convergence \(\sigma^\sharp*\eta_\eps\to\sigma^\sharp\)
against the fixed test function \(\rho^\sharp(-\omega)|\omega|^{-\alpha}\)
gives \(\iiprod{\sigma_\eps,\rho}_{\alpha}\to \iiprod{\sigma,\rho}_{\alpha}\).
Since \(T_\alpha\) is an isometry,
\[
  \|S_{\sigma_\eps}R_\rho f-\iiprod{\sigma,\rho}_{\alpha}T_\alpha f\|_{\cE_\alpha'}
  =|\iiprod{\sigma_\eps,\rho}_{\alpha}-\iiprod{\sigma,\rho}_{\alpha}|\,\|f\|_{\cE_\alpha}
  \longrightarrow0 .
\]
The limit depends only on the pairing
\cref{eq:distributional-admissibility}, hence not on \(\chi\) or
\((\eta_\eps)\).  Applying \(T_\alpha^{-1}\), respectively substituting
\(f=T_\alpha^{-1}g\), proves \cref{eq:tempered-reconstruction}.  The
existence of a band-pass \(\rho\) with nonzero constant for
non-polynomial \(\sigma\) was proved at the end of the proof of \cref{thm:A}.
\end{proof}

The argument uses the direction measure only through the homogeneity
\cref{eq:homogeneous-change} and the finiteness of the admissibility
constants of \(\rho\) and of the regularizations \(\sigma_\eps\); neither the
positivity of \(\alpha\) nor the full support of \(\nu_\alpha\) enters
\cref{eq:tempered-frame-identity} or the independence of \(\chi\) and
\((\eta_\eps)\).  Full support is used earlier, in \cref{thm:B}, to make
\(R_\rho\) injective.

Here a synthesis activation is called admissible if it can be paired
with a real band-pass analysis filter to give a nonzero reconstruction
constant \(\iiprod{\sigma,\rho}_{\alpha}\).
This is a condition on the pair \((\sigma,\rho)\);
the self-admissibility condition \cref{eq:self-admissibility} concerns
the analysis filter alone.

\begin{cor}[ReLU is admissible]
\label{cor:relu-admissible}
Let \(\sigma(t)=\ReLU(t)=\max(t,0)\).  Under the convention
\cref{eq:fourier-convention},
\begin{equation}
  \ReLU^\sharp
  =-\fp(\omega^{-2})+i\pi\delta_0',
  \label{eq:relu-fourier}
\end{equation}
which equals \(-\omega^{-2}\) away from the origin.  If
\(\rho^\sharp\in C_c^\infty(\R\setminus\{0\})\) is nonzero, even, and
nonpositive, then
\begin{equation}
  \iiprod{\ReLU,\rho}_{\alpha}
  =-\frac1{2\pi}\int_\R
  \rho^\sharp(\omega)|\omega|^{-\alpha-2}\,\dd\omega>0 .
  \label{eq:relu-admissibility}
\end{equation}
After rescaling \(\rho\) the constant is one, and
\cref{eq:tempered-reconstruction} and \cref{thm:A}(iii) hold with ReLU
synthesis for every \(\alpha>0\).
\end{cor}

\begin{proof}
From \(\ReLU(t)=(|t|+t)/2\), the identities
\((|t|)^\sharp=-2\fp(\omega^{-2})\) and \(t^\sharp=2\pi i\delta_0'\)
give \cref{eq:relu-fourier}.  The test function in
\cref{eq:distributional-admissibility} is supported away from zero, so
the \(\delta_0'\) term vanishes and the finite part is ordinary
multiplication by \(\omega^{-2}\); evenness of \(\rho^\sharp\) and its
sign give \cref{eq:relu-admissibility}.
\end{proof}

\begin{ex}[Standard activations]
\label{ex:standard-activations}
ReLU, \(\tanh\), the Gaussian distribution function, and the Gaussian
\(e^{-u^2/2}\) are globally Lipschitz, are not polynomials, and belong to
\(\cA_{0,2}=\langle\cdot\rangle^2L^2(\R)\); see
\cref{lem:standard-activation-class}.  Each of them is therefore covered
by \cref{thm:tempered-reconstruction}, by \cref{thm:A}(iii), and by the
finite-width bounds of \cref{sec:finite-width}.
\end{ex}

\begin{proof}
The three properties are verified in \cref{lem:standard-activation-class};
a globally Lipschitz function has polynomial growth, so
\cref{thm:A}(iii) applies. The activation hypotheses of the
finite-width results below are therefore satisfied.
\end{proof}

\begin{rmk}[Identity in the dual space versus integral network]
\label{rmk:weak-versus-strong}
\Cref{eq:tempered-frame-identity} is an identity in \(\cE_\alpha'\); it
does not by itself produce a finite coefficient measure on
\(\cH\times\R\).  When the target has a spectral density \(g\) satisfying
\cref{def:ray-regular}, \cref{thm:A}(iii) upgrades the second formula in
\cref{eq:tempered-reconstruction} to an absolutely convergent integral
network with coefficient \(\gamma_g\), and \cref{thm:E} shows that this
coefficient measure has finite total variation and moments.  The examples of
\cref{sec:examples} with ReLU synthesis are of this kind.
\end{rmk}
 \subsection{Integral representation under Sobolev conditions}
\label{subsec:weak-sobolev}

Band-pass support avoids the singular frequency at zero, but it is not
necessary for an absolutely convergent network integral.  The next theorem
replaces smoothness of \(g_a(\omega)=g(\omega a)\) by a condition on
the inverse Fourier transform of
\(h_a(\omega):=\rho^\sharp(-\omega)g_a(\omega)\).
This condition also allows non-band-pass filters. We first give the
weighted inverse Fourier estimates and the dual pairing used in the proof.
The direction measure may be \(\nu_\alpha\), or any homogeneous
sigma-finite measure as specified below.  The Fourier pairing in the
statement is its bounded Sobolev-dual extension, defined precisely in
\cref{lem:sobolev-pairing}.

Write \(\langle t\rangle=(1+t^2)^{1/2}\), and use the angular Fourier
convention
\[
 f^\sharp(\omega)=\int_{\mathbb R}f(t)e^{-it\omega}\,dt,
 \qquad
 \check h(t)=\frac1{2\pi}\int_{\mathbb R}h(\omega)e^{it\omega}\,d\omega.
\]
Here \(\outY\) is the standing separable complex Hilbert space.
The inverse transform is understood in \(L^2\) when the integral need
not exist pointwise.  For \(s\geq0\), define the Bessel-potential space
in the frequency variable by
\begin{equation}
 H^s_\omega(\mathbb R;\outY)
 =\{h\in L^2(\mathbb R;\outY):\langle t\rangle^s\check h(t)\in L^2\},
 \qquad
 \|h\|_{H^s_\omega}^2
 =2\pi\int_{\mathbb R}\langle t\rangle^{2s}\|\check h(t)\|_\outY^2\,dt.
 \label{eq:sobolev-norm}
\end{equation}
Thus \(\|h\|_{H^0_\omega}=\|h\|_2\) by Plancherel.  The displayed norm
is the Bessel-potential norm for this Fourier normalization.
It is a separable Hilbert space: the map
\(h\mapsto\sqrt{2\pi}\langle\cdot\rangle^s\check h\) is an
isometric isomorphism onto \(L^2(\mathbb R;\outY)\).
For Banach-valued data in \cref{thm:weak-sobolev-synthesis}, only the
weighted inverse-transform norm in \cref{eq:sobolev-norm} is used.
The weighted \(L^1\) estimate below follows directly from
Cauchy--Schwarz for that norm, without a Fourier isometry.

\begin{lem}[Weighted inverse Fourier estimates and modulation]
\label{lem:sobolev-tools}
Let \(s>1/2\), \(0\leq r<s-1/2\), and put
\[
 A_{s,r}:=\frac1{\sqrt{2\pi}}
   \left(\int_{\mathbb R}(1+t^2)^{-(s-r)}\,dt\right)^{1/2}.
\]
For every \(h\in H^s_\omega(\mathbb R;\outY)\),
\begin{equation}
 \int_{\mathbb R}\langle t\rangle^r\|\check h(t)\|_\outY\,dt
 \leq A_{s,r}\|h\|_{H^s_\omega}.
 \label{eq:sobolev-weighted-l1}
\end{equation}
Reflection \(Rh(\omega)=h(-\omega)\) is an isometry on this space.
For \(M_uh(\omega)=e^{iu\omega}h(\omega)\),
\begin{equation}
 \|M_uh\|_{H^s_\omega}\leq(1+|u|)^s\|h\|_{H^s_\omega}.
 \label{eq:sobolev-modulation}
\end{equation}
The map \((u,h)\mapsto M_uh\) is jointly continuous.
\end{lem}

\begin{proof}
Cauchy--Schwarz, applied to
\(\langle t\rangle^{-(s-r)}\) and
\(\langle t\rangle^s\|\check h(t)\|_\outY\), proves
\eqref{eq:sobolev-weighted-l1}.  The scalar integral is finite
precisely when \(s-r>1/2\).  Fourier inversion in \(L^2\) gives
\(F_\R^{-1}(Rh)(t)=\check h(-t)\) and
\(F_\R^{-1}(M_uh)(t)=\check h(t+u)\).
Since \(\langle t-u\rangle\leq\langle t\rangle+|u|
\leq(1+|u|)\langle t\rangle\), changing variables proves the
reflection and modulation estimates.
Translations are strongly continuous in
\(L^2(\mathbb R,\langle t\rangle^{2s}dt;\outY)\): this holds by dominated
convergence for continuous compactly supported functions with values
in a finite-dimensional subspace of \(\outY\); these functions are dense,
and the displayed translation bound is uniform for \(u\) in a bounded
interval.  Approximation therefore proves strong continuity for every
element.  Finally, for \(u_n\to u\) and \(h_n\to h\), split
\(M_{u_n}h_n-M_uh\) into
\(M_{u_n}(h_n-h)+(M_{u_n}-M_u)h\), and apply the local uniform
bound and strong continuity.
\end{proof}

Only the weighted \(L^1\) estimate \eqref{eq:sobolev-weighted-l1} and the reflection
isometry are used for \cref{thm:weak-sobolev-synthesis}; the modulation bound
\eqref{eq:sobolev-modulation} and the joint continuity are the mapping properties that the
alternative route through an \(H^s_\omega\)-valued Bochner integral would need, and are
recorded because they identify the bias translation \eqref{eq:sobolev-bias-pairing} as a
modulation.

\begin{lem}[The bilinear Sobolev pairing and its translation formula]
\label{lem:sobolev-pairing}
Suppose \(\sigma:\mathbb R\to\mathbb C\) is continuous and
\(|\sigma(t)|\leq C_\sigma(1+|t|)^p\), where \(p\geq0\), and
\(s>p+1/2\).  Set
\(b_{\sigma,s}=\|\langle\cdot\rangle^{-s}\sigma\|_2<\infty\).
Its Fourier transform defines an element of the bilinear dual
\(H^{-s}_\omega\), with the following normalization:
\begin{equation}
 \frac1{2\pi}\langle\sigma^\sharp,q\rangle
 =\int_{\mathbb R}\sigma(t)\check q(-t)\,dt,
 \qquad q\in H^s_\omega(\mathbb R),
 \label{eq:sobolev-pairing}
\end{equation}
and the norm of this functional is at most
\((2\pi)^{-1/2}b_{\sigma,s}\).
For \(\outY\)-valued tests the right side defines a bounded linear map
\(L_\sigma^{\outY}:H^s_\omega(\mathbb R;\outY)\to \outY\) with the same
bound.  If \(h\in H^s_\omega(\mathbb R;\outY)\), \(\gamma=\check h\),
and \(u\in\mathbb R\), then the Bochner integral below is absolute and
\begin{equation}
 \int_{\mathbb R}\sigma(u-b)\gamma(b)\,\dd b
 =L_\sigma^{\outY}(M_uh).
 \label{eq:sobolev-bias-pairing}
\end{equation}
\end{lem}
\begin{proof}
The growth assumption makes \(b_{\sigma,s}\) finite.  Weighted
Cauchy--Schwarz proves
\[
 \int_{\mathbb R}|\sigma(t)|\|\check q(-t)\|_\outY\,dt
 \leq\frac{b_{\sigma,s}}{\sqrt{2\pi}}\|q\|_{H^s_\omega},
\]
so the claimed integrals and bounded maps exist.  For a scalar Schwartz
test \(q\), the definition of the distributional Fourier transform
and \(q^\sharp(t)=2\pi\check q(-t)\) give
\eqref{eq:sobolev-pairing}.  Scalar Schwartz functions are dense
in \(H^s_\omega\): under the isomorphism following
\eqref{eq:sobolev-norm}, the inverse Fourier transforms of
Schwartz functions are dense in weighted \(L^2\), as follows by
truncation and mollification.  Thus the bounded extension agrees with
the Fourier distribution, establishing the stated dual interpretation.
The vector-valued construction uses the same absolute Bochner integral
and does not require a scalar choice of coordinates.

Now \(F_\R^{-1}(M_uh)(t)=\gamma(t+u)\), so
\(F_\R^{-1}(M_uh)(-t)=\gamma(u-t)\).
Substitution \(b=u-t\) in its absolutely convergent pairing proves
\eqref{eq:sobolev-bias-pairing}.  In particular, the test is
\(e^{iu\omega}h(\omega)\), with positive modulation and no reflection, corresponding to the
minus sign in the neuron bias.
\end{proof}

\begin{thm}[Integral representation under Sobolev conditions]
\label{thm:weak-sobolev-synthesis}
Let \(\cH\) be a real inner-product space with its Borel structure and let \(\nu\) be a
sigma-finite Borel measure on \(\cH\) such that, for every \(t\neq0\),
\((D_t)_\#\nu=|t|^{-\alpha}\nu\), where \(D_ta=ta\) and \(\alpha>0\).  Let \(\outY\) be a
complex Banach space, let \(\rho\in\mathcal S(\mathbb R;\mathbb R)\) be nonzero, let
\(\sigma:\R\to\C\) be continuous with \(|\sigma(t)|\le C_\sigma(1+|t|)^p\), where \(p\ge0\)
and \(s>p+1/2\), and let \(g:\cH\to \outY\) be strongly measurable.  Write
\(g_a(\omega):=g(\omega a)\) and
\(h_a(\omega):=\rho^\sharp(-\omega)g_a(\omega)\), and assume the following.
\begin{enumerate}
\item There is a jointly measurable \(\gamma_g:\cH\times\R\to \outY\) such that,
for \(\nu\)-almost every \(a\), the function \(b\mapsto\gamma_g(a,b)\)
has the weighted integrability and Fourier transform
\[
 \int_\R\langle b\rangle^{2s}\|\gamma_g(a,b)\|_\outY^2\,\dd b<\infty
 \qquad\text{and}\qquad
 \int_\R\gamma_g(a,b)e^{-i\omega b}\,\dd b=h_a(\omega)
 \quad\text{for every }\omega\in\R.
\]
When \(\outY\) is a Hilbert space, this is equivalently
\(h_a\in H^s_\omega(\mathbb R;\outY)\) with
\(\check h_a=\gamma_g(a,\cdot)\).
\item
\(\displaystyle
 \mathfrak B_s(\rho,g)
 :=\int_\cH(1+\|a\|)^s\|h_a\|_{H^s_\omega}\,\nu(da)<\infty.
\)
\item The Lebesgue class
\(q_{\alpha,\rho}(\omega)=\rho^\sharp(-\omega)|\omega|^{-\alpha}\)
belongs to \(H^s_\omega(\mathbb R)\).  No value at \(\omega=0\) is
prescribed.
\end{enumerate}
Then \(g\in L^1(\nu;\outY)\), so that
\(f_g(x)=\int_\cH e^{i\langle x,\xi\rangle}g(\xi)\,\nu(d\xi)\) is defined, and
\(\gamma_g\in L^1(\nu\otimes\,\dd b;\outY)\): for every \(0\leq r<s-1/2\),
\begin{equation}
 \int_{\cH\times\mathbb R}(1+\|a\|+|b|)^r
          \|\gamma_g(a,b)\|_\outY\,\nu(da)\,\dd b
 \leq 2^{r/2}A_{s,r}\,\mathfrak B_s(\rho,g),
 \label{eq:sobolev-moments}
\end{equation}
so the vector measure \(\Gamma_g=\gamma_g(\nu\otimes\,\dd b)\) is finite.  For every
\(x\in \cH\) the direction average
\[
 \Psi_x(t):=\int_\cH\gamma_g(a,\langle a,x\rangle-t)\,\nu(da)
\]
converges absolutely for almost every \(t\), belongs to \(L^1(\R;\outY)\), and
\begin{equation}
 \Psi_x(t)=\check q_{\alpha,\rho}(-t)\,f_g(x)
 \qquad\text{for almost every }t\in\R.
 \label{eq:sobolev-direction-average}
\end{equation}
Consequently the ordinary, absolutely convergent synthesis satisfies, for every
\(x\in \cH\),
\begin{equation}
 S_\sigma[\Gamma_g](x)
 :=\int_{\cH\times\mathbb R}\sigma(\langle a,x\rangle-b)
                    \gamma_g(a,b)\,\nu(da)\,\dd b
 =\iiprod{\sigma,\rho}_{\alpha}f_g(x),
 \qquad
 \iiprod{\sigma,\rho}_{\alpha}
 =\frac1{2\pi}\langle\sigma^\sharp,q_{\alpha,\rho}\rangle .
 \label{eq:weak-sobolev-synthesis}
\end{equation}
The absolute convergence is uniform on bounded input sets, and the
synthesis is continuous.  No non-polynomial assumption is needed for
this identity; normalization to reproduce \(f_g\) requires the
separate condition \(\iiprod{\sigma,\rho}_{\alpha}\neq0\).
If, in addition, \(\outY\) is a separable complex Hilbert space,
\(\iiprod{\rho,\rho}_{\alpha}<\infty\), and \(g\in L^2(\nu;\outY)\), then
\(\gamma_g\in L^2(\nu\otimes\,\dd b;\outY)\) with
\begin{equation}
 \int_{\cH\times\R}\|\gamma_g(a,b)\|_\outY^2\,\nu(da)\,\dd b
 =\iiprod{\rho,\rho}_{\alpha}\,\|g\|_{L^2(\nu;\outY)}^2 .
 \label{eq:sobolev-coefficient-l2}
\end{equation}
\end{thm}

\begin{proof}
\label{app:weak-sobolev-proof}
Throughout, \(\gamma_g\) is the jointly measurable coefficient of hypothesis 1, and
\(h_a\) is the Fourier transform of \(b\mapsto\gamma_g(a,b)\), so that
\(\|h_a\|_{H^s_\omega}^2
 =2\pi\int_\R\langle b\rangle^{2s}\|\gamma_g(a,b)\|_\outY^2\,\dd b\).
Measurable versions may be chosen Borel, since the base measures are sigma-finite Borel
measures; completion does not change the resulting density measure.

\emph{Moments.}  Since \(r<s\),
\[
 (1+\|a\|+|b|)^r
 \leq (1+\|a\|)^r(1+|b|)^r
 \leq 2^{r/2}(1+\|a\|)^s\langle b\rangle^r,
\]
so integrating the weighted inverse Fourier estimate
\eqref{eq:sobolev-weighted-l1} over \(\nu\) proves \eqref{eq:sobolev-moments}.  In
particular \(\Gamma_g\) has finite variation, with
\(|\Gamma_g|=\|\gamma_g\|_\outY(\nu\otimes\,\dd b)\), and
to see that \(g\in L^1(\nu;\outY)\), choose \(\omega_0\ne0\)
with \(\rho^\sharp(-\omega_0)\ne0\). Hypothesis 1 and the weighted
inverse Fourier estimate with \(r=0\) give
\[
 \|g\|_{L^1(\nu;\outY)}
 \le |\omega_0|^\alpha |\rho^\sharp(-\omega_0)|^{-1}
 A_{s,0}\,\mathfrak B_s(\rho,g)<\infty.
\]
Here we integrated \(\|h_a(\omega_0)\|\le A_{s,0}\|h_a\|_{H^s_\omega}\)
and used homogeneity. Thus \(f_g\) is defined, bounded and continuous.

\emph{A majorant.}  If \(\|x\|\leq R\), then
\[
 |\sigma(\langle a,x\rangle-b)|\,\|\gamma_g(a,b)\|_\outY
 \leq C_\sigma\max\{1,R\}^{p}
       (1+\|a\|+|b|)^p\|\gamma_g(a,b)\|_\outY,
\]
whose right side is integrable by \eqref{eq:sobolev-moments} with \(r=p\), which is
admissible because \(p<s-1/2\).  This is a common absolute majorant on the whole ball
\(\|x\|\leq R\); it makes the synthesis absolutely convergent, uniformly there, justifies
the direction--bias Fubini interchange below, and proves continuity of the synthesis by
dominated convergence and continuity of \(\sigma\).

\emph{The direction average.}  Fix \(x\in \cH\) and write \(u_a=\langle a,x\rangle\).  For
each \(a\) the map \(b\mapsto u_a-b\) preserves Lebesgue measure, and \((a,b)\mapsto u_a-b\)
is jointly measurable, so \eqref{eq:sobolev-moments} with \(r=0\) shows that
\((a,t)\mapsto\gamma_g(a,u_a-t)\) is integrable for \(\nu\otimes\,\dd t\).  Hence \(\Psi_x\)
is defined for almost every \(t\), lies in \(L^1(\R;\outY)\), and Fubini gives
\begin{equation}
 S_\sigma[\Gamma_g](x)=\int_\R\sigma(t)\Psi_x(t)\,\dd t .
 \label{eq:sobolev-synthesis-average}
\end{equation}

\emph{Identification of \(\Psi_x\).}  Both \(\Psi_x\) and
\(\Theta_x:=\check q_{\alpha,\rho}(-\,\cdot\,)f_g(x)\) are integrable, the latter by
\eqref{eq:sobolev-weighted-l1} applied to \(q_{\alpha,\rho}\), so both have continuous
Fourier transforms.  Let \(\omega\neq0\).  Fubini in \((a,t)\), the substitution
\(b=u_a-t\), hypothesis 1 and homogeneity give
\begin{align*}
 \Psi_x^\sharp(\omega)
 &=\int_\cH e^{-i\omega u_a}h_a(-\omega)\,\nu(da)
  =\rho^\sharp(\omega)\int_\cH e^{i\langle x,-\omega a\rangle}g(-\omega a)\,\nu(da)\\
 &=\rho^\sharp(\omega)|\omega|^{-\alpha}f_g(x)
  =q_{\alpha,\rho}(-\omega)f_g(x)
  =\Theta_x^\sharp(\omega).
\end{align*}
Two continuous functions that agree off the origin agree everywhere, so
\(\Psi_x^\sharp=\Theta_x^\sharp\), and uniqueness of the Fourier transform on
\(L^1(\R;\outY)\) proves \eqref{eq:sobolev-direction-average}.  The argument uses no value
of \(q_{\alpha,\rho}\) at the origin, and no evaluation of a Sobolev class at a prescribed
frequency.

\emph{The identity.}  By \eqref{eq:sobolev-synthesis-average},
\eqref{eq:sobolev-direction-average} and \eqref{eq:sobolev-pairing},
\[
 S_\sigma[\Gamma_g](x)
 =\int_\R\sigma(t)\check q_{\alpha,\rho}(-t)\,\dd t\;f_g(x)
 =\frac1{2\pi}\langle\sigma^\sharp,q_{\alpha,\rho}\rangle f_g(x)
 =\iiprod{\sigma,\rho}_{\alpha}f_g(x),
\]
the first integral being absolutely convergent by \cref{lem:sobolev-pairing}.  This proves
\eqref{eq:weak-sobolev-synthesis}.

\emph{The \(L^2\) statement.}  Assume now that \(\outY\) is a separable complex Hilbert
space, that \(\iiprod{\rho,\rho}_{\alpha}<\infty\) and that \(g\in L^2(\nu;\outY)\).
Plancherel in the bias for almost every \(a\), Tonelli, and homogeneity give
\begin{align*}
 \int_{\cH\times\mathbb R}\|\gamma_g(a,b)\|_\outY^2\,\nu(da)\,\dd b
 &=\frac1{2\pi}\int_{\mathbb R}|\rho^\sharp(\omega)|^2
       \int_\cH\|g(\omega a)\|_\outY^2\,\nu(da)\,d\omega\\
 &=\iiprod{\rho,\rho}_{\alpha}\|g\|_{L^2(\nu;\outY)}^2,
\end{align*}
the omitted frequency \(\omega=0\) being a Lebesgue null set.  This also identifies the
coefficient with \(W_\rho g\), the \(L^2\) coefficient of \cref{def:spectral-coefficient}.
\end{proof}

\begin{rmk}[What the synthesis identity uses]
\label{rmk:sobolev-synthesis-structure}
The identity \eqref{eq:weak-sobolev-synthesis} uses no property of \(\cH\) beyond its inner
product and its Borel structure, and \(\alpha>0\) only through the existence of homogeneous
measures of that degree.  It uses no property of \(\outY\) beyond completeness of its norm:
the Hilbert structure and the separability of the standing output space enter only in
\eqref{eq:sobolev-coefficient-l2}, through Plancherel in the bias variable.  Admissibility of
\(\rho\) is not assumed either: positivity of \(\iiprod{\rho,\rho}_{\alpha}\) is never used,
and its finiteness only in \eqref{eq:sobolev-coefficient-l2}.  Finally
\(g\in L^1(\nu;\outY)\) need not be assumed, because hypotheses 1 and 2 give it: choose
\(\omega_0\neq0\) with \(\rho^\sharp(-\omega_0)\neq0\), so that
\eqref{eq:sobolev-weighted-l1} with \(r=0\) bounds
\(|\rho^\sharp(-\omega_0)|\,\|g(\omega_0a)\|_\outY=\|h_a(\omega_0)\|_\outY\)
by \(A_{s,0}\|h_a\|_{H^s_\omega}\); integrating in \(a\) and using homogeneity,
\[
 \|g\|_{L^1(\nu;\outY)}
 \leq|\omega_0|^{\alpha}\,|\rho^\sharp(-\omega_0)|^{-1}A_{s,0}\,
     \mathfrak B_s(\rho,g).
\]
\end{rmk}

For reference,
\(A_{s,r}=(2\pi)^{-1/2}
(\int_\R(1+t^2)^{-(s-r)}\,\dd t)^{1/2}\).
The distinction between synthesis and sampling matters: a second
parameter moment requires \(s>5/2\).  With this extra condition and
a real globally Lipschitz activation, \cref{thm:lipschitz-barron}
applies; reproducing \(f_g\) requires division of the coefficient
by a nonzero \(\iiprod{\sigma,\rho}_{\alpha}\).

These hypotheses allow concrete non-band-pass filters.  For
\(\rho^\sharp(\omega)=\omega^{2k}e^{-\omega^2}\) and
\(g(\xi)=e^{-\|\xi\|^2}v\), it suffices that
\(2k>\alpha+2s-1/2\).  For example, \(\alpha=1,s=3,k=4\)
permits a second moment and nonzero reconstruction constants for both
Gaussian and ReLU activations.  The following proposition states these
conditions and constants, followed by their verification.

\begin{prop}[Non-band-pass Gaussian-derivative filters]
\label{prop:nonbandpass-sobolev}
Let \(\nu\) be a sigma-finite Borel measure on \(\cH\) satisfying
\((D_\omega)_\#\nu=|\omega|^{-\alpha}\nu\) for \(\omega\neq0\), with
\(\alpha>0\), and suppose \(\nu(\{a:\|a\|\leq1\})<\infty\).
Fix \(s>1/2\), choose an integer \(k\geq1\) with
\begin{equation}
 2k>\alpha+2s-\tfrac12,
 \label{eq:nonbandpass-order}
\end{equation}
and define
\[
 \rho^\sharp(\omega)=\omega^{2k}e^{-\omega^2},
 \qquad g(\xi)=e^{-\|\xi\|^2}v,\qquad v\in \outY.
\]
Then \(\rho\) is a real Schwartz filter, is not band-pass, and
satisfies all filter and density hypotheses of
Theorem~\ref{thm:weak-sobolev-synthesis} for this \(s\).
In particular this gives an absolutely convergent network integral for every continuous
activation of growth order \(p<s-1/2\).

Put \(\delta=2k-\alpha\).  For
\(\sigma(t)=e^{-t^2/2}\), the reconstruction constant is
\[
 \iiprod{\sigma,\rho}_{\alpha}
 =\frac{\Gamma((\delta+1)/2)}{\sqrt{2\pi}}
       (3/2)^{-(\delta+1)/2}>0.
\]
If \(s>3/2\), ReLU also satisfies the hypotheses, and
\[
 \iiprod{\operatorname{ReLU},\rho}_{\alpha}
 =-\frac1{2\pi}\Gamma((\delta-1)/2)\neq0.
\]
The local finiteness assumption holds for the homogeneous Gaussian
mixtures used in this paper on an infinite-dimensional Hilbert space.
\end{prop}
\begin{proof}
\label{app:proof-nonbandpass-sobolev}
Homogeneity gives \(\nu(B_R)=R^\alpha\nu(B_1)\) for closed balls
centered at zero and \(R>0\).  Splitting into \(B_1\) and the annuli
\(2^j<\|a\|\leq2^{j+1}\), \(j\geq0\), shows that
\begin{equation}
 \int_\cH(1+\|a\|^2)^{-d/2}\,\nu(da)<\infty
 \quad\hbox{whenever }d>\alpha.
 \label{eq:homogeneous-polynomial-integrability}
\end{equation}
Indeed the annular integrals are bounded by
\(\nu(B_1)2^\alpha 2^{-j(d-\alpha)}\), a summable geometric
series.  Gaussian decay therefore gives \(g\in L^1(\nu;\outY)\cap L^2(\nu;\outY)\);
the first is automatic by \cref{rmk:sobolev-synthesis-structure}, and the second supplies
\eqref{eq:sobolev-coefficient-l2}.

The inverse transform of the real even Schwartz function
\(\omega^{2k}e^{-\omega^2}\) is real and Schwartz.  Its Fourier
transform is nonzero at every \(\omega\neq0\), so it is not
band-pass.  Self-admissibility is equivalent here to
\(4k-\alpha>-1\), since the squared Fourier transform is
\(|\omega|^{4k}e^{-2\omega^2}\); this follows from
\eqref{eq:nonbandpass-order}.  The integral is positive.

Set \(A=(1+\|a\|^2)^{1/2}\),
\(h_0(\omega)=\omega^{2k}e^{-\omega^2}\), and
\(\gamma_0=\check h_0\).  The functions \(h_a\) and their inverse
transforms satisfy
\[
 h_a(\omega)=A^{-2k}h_0(A\omega)v,
 \qquad
 \check h_a(b)=A^{-2k-1}\gamma_0(b/A)v.
\]
Since \(A\geq1\), substitution \(b=At\) gives
\[
 \|h_a\|_{H^s_\omega}
 \leq\|v\|_\outY\|h_0\|_{H^s_\omega}A^{s-2k-1/2}.
\]
The coefficient \((a,b)\mapsto A^{-2k-1}\gamma_0(b/A)v\) is jointly
continuous, because \(a\mapsto A\) is continuous with \(A\geq1\) and
\(\gamma_0\) is Schwartz; for each \(a\), its restriction to the bias variable is \(\check h_a\).  This
is hypothesis 1 of \cref{thm:weak-sobolev-synthesis}.  Moreover,
\[
 (1+\|a\|)^s\|h_a\|_{H^s_\omega}
 \leq2^{s/2}\|v\|_\outY\|h_0\|_{H^s_\omega}
          (1+\|a\|^2)^{(2s-2k-1/2)/2}.
\]
This is integrable by \eqref{eq:homogeneous-polynomial-integrability}
and \eqref{eq:nonbandpass-order}, which say
\(2k+1/2-2s>\alpha\).  Thus \(\mathfrak B_s<\infty\).

The Sobolev membership of \(q_{\alpha,\rho}(\omega)=|\omega|^{\delta}e^{-\omega^2}\)
follows from the same dilates, by the subordination formula
\begin{equation}
 |\omega|^{-\alpha}
 =\frac1{\Gamma(\alpha/2)}\int_0^\infty u^{\alpha/2-1}e^{-u\omega^2}\,\dd u,
 \qquad\omega\neq0 .
 \label{eq:gamma-subordination}
\end{equation}
It writes \(q_{\alpha,\rho}\) as the superposition
\[
 q_{\alpha,\rho}(\omega)
 =\frac1{\Gamma(\alpha/2)}\int_0^\infty u^{\alpha/2-1}
     \tau^{-2k}h_0(\tau\omega)\,\dd u,
 \qquad \tau=\tau(u)=(1+u)^{1/2},
\]
of the rescaled functions \(\tau^{-2k}h_0(\tau\,\cdot)\), whose inverse Fourier transforms are the dilates
\(\tau^{-2k-1}\gamma_0(\cdot/\tau)\).  Since
\(\|\tau^{-2k-1}\gamma_0(\cdot/\tau)\|_{L^1}=\tau^{-2k}\|\gamma_0\|_{L^1}\) and
\(\int_0^\infty u^{\alpha/2-1}(1+u)^{-k}\,\dd u<\infty\) for \(0<\alpha<2k\), Fubini gives
the coefficient
\[
 \check q_{\alpha,\rho}(b)
 =\frac1{\Gamma(\alpha/2)}\int_0^\infty u^{\alpha/2-1}
     \tau^{-2k-1}\gamma_0(b/\tau)\,\dd u .
\]
For the norm, weighted Cauchy--Schwarz in \(u\) with the weights
\(w(u)=u^{\alpha/2-1}\) and \(m(u)=\tau(u)^{s-2k-1/2}\), followed by Tonelli in \((u,b)\),
bounds \(\|q_{\alpha,\rho}\|_{H^s_\omega}^2\) by
\[
 \begin{aligned}
 &
 \frac1{\Gamma(\alpha/2)^2}
 \Bigl(\int_0^\infty wm\,\dd u\Bigr)
 \int_0^\infty\frac{w}{m}\,
     \bigl\|\tau^{-2k}h_0(\tau\,\cdot\,)\bigr\|_{H^s_\omega}^2\,\dd u
  \\
 &\qquad\leq\frac{\|h_0\|_{H^s_\omega}^2}{\Gamma(\alpha/2)^2}
       \Bigl(\int_0^\infty u^{\alpha/2-1}(1+u)^{(s-2k-1/2)/2}\,\dd u\Bigr)^2,
 \end{aligned}
\]
by the dilation estimate already displayed.  The Beta-type integral converges at the origin
because \(\alpha>0\) and at infinity because \(2k>\alpha+s-1/2\), which is weaker than
\eqref{eq:nonbandpass-order}.  Hence \(q_{\alpha,\rho}\in H^s_\omega(\R)\), with no
discussion of weak derivatives or of boundary terms at the origin.

For the Gaussian activation,
\(\sigma^\sharp(\omega)=\sqrt{2\pi}e^{-\omega^2/2}\).
The ordinary Fourier pairing and substitution \(z=(3/2)\omega^2\)
give the first displayed constant.  For ReLU use
\(\operatorname{ReLU}^\sharp=-\operatorname{Pf}(\omega^{-2})
+i\pi\delta'_0\), and pair it with the superposition
\eqref{eq:gamma-subordination} rather than with \(q_{\alpha,\rho}\) directly.  Each
test \(\omega\mapsto \tau^{-2k}h_0(\tau\omega)\) is a Schwartz function vanishing to order
\(2k\geq2\) at the origin, so \(\delta'_0\) contributes nothing and the finite part is the
absolutely convergent integral
\(-\int_\R\omega^{2k-2}e^{-\tau^2\omega^2}\,\dd\omega\); no cutoff sequence is
needed.  Fubini in \((u,\omega)\), licensed by
\(\int_0^\infty u^{\alpha/2-1}(1+u)^{-(2k-1)/2}\,\dd u<\infty\) for
\(\alpha<2k-1\), which follows from \eqref{eq:nonbandpass-order} and \(s>3/2\), then
returns the subordination integral in the \(\omega\) variable and gives
\[
 \iiprod{\operatorname{ReLU},\rho}_{\alpha}
 =-\frac1{2\pi}\int_{\mathbb R}|\omega|^{\delta-2}e^{-\omega^2}\,d\omega
 =-\frac1{2\pi}\Gamma((\delta-1)/2).
\]

Finally, for the Gaussian mixture write \(Z\sim\mu_{P}\).
Tonelli and the substitution implicit in
\(\|\sqrt{2t}Z\|\leq R\) give
\[
 \nu_\alpha(B_R)
 =\mathbb E\int_0^{R^2/(2\|Z\|^2)}t^{\alpha/2-1}\,dt
 =\frac{2^{1-\alpha/2}}{\alpha}R^\alpha
       \mathbb E\|Z\|^{-\alpha}<\infty.
\]
For finiteness choose an integer \(n>\alpha\) and \(n\) positive
eigenvalues of the injective covariance \(P\).  If their minimum is
\(\lambda_*>0\), then
\(\|Z\|^2\geq\lambda_*\sum_{j=1}^n Z_j^2\), with independent
standard real Gaussians \(Z_j\).  Polar integration in
\(\mathbb R^n\) gives a finite negative moment of order \(\alpha<n\).
This proves the claimed local finiteness and completes the example.
\end{proof}

\section{Approximation by finite-width networks}
\label{sec:finite-width}

We obtain finite-width networks by Monte Carlo approximation of the
coefficient integral and estimate their error uniformly on compact sets.
Conditions on the spectral density provide the required coefficient
moments, and density of these targets yields cc-universality.  The
construction does not require a projection of the input space or the
direction parameter.

\subsection{Approximation rates}
\label{subsec:sampling}

Let \(\Gamma\) be a finite complex measure on \(\cH\times\R\) with polar
decomposition \(\Gamma=h|\Gamma|\), \(|h|=1\) \(|\Gamma|\)-almost
everywhere.  If \(\Gamma=0\) take the zero network.  Otherwise write
\begin{equation}
  V:=\|\Gamma\|_{\mathrm{TV}},
  \qquad p:=|\Gamma|/V,
  \label{eq:polar-decomposition}
\end{equation}
sample independent \(\theta_j=(a_j,b_j)\sim p\), and form the sampled network
\begin{equation}
  f_N(x):=\frac{V}{N}\sum_{j=1}^N
  h(\theta_j)\,\sigma(\langle a_j,x\rangle-b_j),
  \label{eq:polar-network}
\end{equation}
which is unbiased for \(f=S_\sigma[\Gamma]\).

\begin{dfn}[Activation-dependent Rademacher complexity]
\label{def:rademacher-complexity}
For a compact \(K\subset \cH\), define
\begin{equation}
  \mathfrak R_N(K;p,\sigma)
  :=\EE_{\theta,\varepsilon}\sup_{x\in K}
  \left|\frac1N\sum_{j=1}^N\varepsilon_jh(\theta_j)
  \sigma(\langle a_j,x\rangle-b_j)\right|,
  \label{eq:rademacher-complexity}
\end{equation}
where \(\varepsilon_j\) are independent Rademacher signs.
\end{dfn}

\begin{lem}[Approximation by averages of integrable Banach-valued atoms]
\label{lem:banach-rademacher-vanishing}
Let \(X\) be a separable Banach space, \(p\) a probability measure,
and \(\Phi\in L^1(p;X)\).  For independent samples \(\theta_j\sim p\)
and independent Rademacher signs \(\varepsilon_j\),
\[
 \mathbb E\left\|\frac1N\sum_{j=1}^N
                 \varepsilon_j\Phi(\theta_j)\right\|_X\longrightarrow0.
\]
Moreover, with \(m=\int\Phi\,\dd p\),
\[
 \mathbb E\left\|\frac1N\sum_{j=1}^N\Phi(\theta_j)-m\right\|_X
 \le2\mathbb E\left\|\frac1N\sum_{j=1}^N
                 \varepsilon_j\Phi(\theta_j)\right\|_X.
\]
\end{lem}
\begin{proof}
Approximate \(\Phi\) in \(L^1(p;X)\) by a simple function
\(\Psi=\sum_{k=1}^Jx_k\mathbf1_{E_k}\).  The expected norm of the
signed-average error is at most \(\|\Phi-\Psi\|_{L^1}\).
For the simple function it is at most
\(\sum_k\|x_k\|\sqrt{p(E_k)/N}\), by the scalar second-moment
identity and Cauchy--Schwarz.  First let \(N\to\infty\), then let the
approximation error tend to zero.  For the second assertion introduce
independent copies \(\theta'_j\), apply Jensen to the ghost sample,
and use exchangeability of each pair to insert signs into
\(\Phi(\theta_j)-\Phi(\theta'_j)\).  The triangle inequality bounds
the result by two copies of the signed average.  Every integral is
justified by \(\Phi\in L^1\).
\end{proof}

\begin{thm}[Compact-open approximation bound]
\label{thm:general-rademacher}
Whenever the atoms \(x\mapsto h(\theta)\sigma(\langle a,x\rangle-b)\) are
strongly measurable and integrable in \(C(K)\), the sampled network satisfies
\begin{equation}
  \EE\|f_N-f\|_{C(K)}
  \le 2V\,\mathfrak R_N(K;p,\sigma).
  \label{eq:general-rademacher-bound}
\end{equation}
\end{thm}

\begin{proof}
Apply \cref{lem:banach-rademacher-vanishing} with \(X=C(K)\) and multiply
by \(V\).  This also proves convergence to zero.  No Lipschitz or
second-moment assumption enters this step.
\end{proof}

\begin{lem}[A two-coordinate Rademacher comparison]
\label{lem:two-coordinate-comparison}
Let $S$ be a nonempty set, let $N\geq1$, and let
$\psi_i,u_i,v_i:S\to\mathbb R$ be bounded functions, for $1\leq i\leq N$.
Suppose that
\[
 |\psi_i(s)-\psi_i(t)|
 \leq |u_i(s)-u_i(t)|+|v_i(s)-v_i(t)|
 \qquad(s,t\in S).
\]
For independent Rademacher signs $\varepsilon_i,\varepsilon_{i1},
\varepsilon_{i2}$, one has
\begin{equation}
 \mathbb E\sup_{s\in S}\sum_{i=1}^N\varepsilon_i\psi_i(s)
 \leq
 2\mathbb E\sup_{s\in S}
 \sum_{i=1}^N\bigl(\varepsilon_{i1}u_i(s)
                       +\varepsilon_{i2}v_i(s)\bigr).
 \label{eq:two-coordinate-comparison}
\end{equation}
\end{lem}
\begin{proof}
We first prove a one-sign comparison.  Let $F,\psi,u,v:S\to\mathbb R$ be
bounded, with $\psi,u,v$ satisfying the stated increment condition, so
that every supremum below is finite.  Fix $\delta>0$ and choose
$s_\pm\in S$ with
$F(s_+)+\psi(s_+)>\sup_S(F+\psi)-\delta/2$ and
$F(s_-)-\psi(s_-)>\sup_S(F-\psi)-\delta/2$; no maximizer need exist.  Put
\[
 A=F(s_+)-F(s_-),\qquad
 r=u(s_+)-u(s_-),\qquad q=v(s_+)-v(s_-).
\]
For two independent signs $\eta_1,\eta_2$, the variable
$D=2(\eta_1r+\eta_2q)$ is symmetric.  Consequently
\[
 \mathbb E|A+D|
 =\frac12\mathbb E\bigl(|A+D|+|A-D|\bigr)
 \geq\mathbb E|D|.
\]
Direct enumeration of the four sign pairs gives
\[
 \mathbb E|\eta_1r+\eta_2q|
 =\frac{|r+q|+|r-q|}{2}
 =\max\{|r|,|q|\}
 \geq\frac{|r|+|q|}{2}.
\]
Using only the two indices $s_+,s_-$ to bound a supremum from below,
and using $\max\{a,b\}=(a+b+|a-b|)/2$, we therefore obtain
\begin{align*}
 &\mathbb E_{\eta_1,\eta_2}
    \sup_{s\in S}\bigl(F(s)+2\eta_1u(s)+2\eta_2v(s)\bigr)\\
 &\quad\geq\frac{F(s_+)+F(s_-)}{2}
                   +\frac12\mathbb E|A+D|\\
 &\quad\geq\frac{F(s_+)+F(s_-)}{2}
                   +\frac{|r|+|q|}{2}\\
 &\quad\geq\frac{F(s_+)+F(s_-)+\psi(s_+)-\psi(s_-)}{2}\\
 &\quad\geq\mathbb E_{\varepsilon}
            \sup_{s\in S}\bigl(F(s)+\varepsilon\psi(s)\bigr)-\frac\delta2 .
\end{align*}
Letting $\delta\downarrow0$ proves the one-sign inequality, which allows an
arbitrary bounded offset $F$.  Condition on
all signs except $\varepsilon_i$, apply it to replace
$\varepsilon_i\psi_i$ by
$2\varepsilon_{i1}u_i+2\varepsilon_{i2}v_i$, and then average over the
conditioned signs.  Repeating this for $i=1,\ldots,N$ proves
\eqref{eq:two-coordinate-comparison}.

Both sides are finite averages of bounded suprema over the $2^N$,
respectively $4^N$, sign patterns, so no countability, measurability or
separability of $S$ enters, and no assumption that any of the functions
vanish at a common index is needed.
\end{proof}

\begin{thm}[Dimension-free uniform Hilbert-valued Barron bound]
\label{thm:lipschitz-barron}
Let \(\Gamma\) be a finite-variation \(\outY\)-valued Borel measure,
\(\sigma:\R\to\R\) globally Lipschitz, and \(N\ge1\).
For \(V>0\), write \(\Gamma=h|\Gamma|\), \(\|h\|_\outY=1\) almost
everywhere, and \(p=|\Gamma|/V\).  Suppose
\begin{equation}
 M_2^2:=\int_{\cH\times\R}(\|a\|^2+|b|^2)\,p(\dd a,\dd b)<\infty.
 \label{eq:second-moment}
\end{equation}
For a nonempty compact \(K\subset \cH\), put
\(R_K=\sup_{x\in K}\sqrt{\|x\|^2+1}\).  The sampled network
\cref{eq:polar-network}, now with \(\outY\)-valued weights, satisfies
\begin{equation}
 \begin{split}
 \EE\|f_N-f\|_{C(K;\outY)}
 &\le\frac{V}{\sqrt N}
       \bigl(4|\sigma(0)|+8\Lip(\sigma)R_KM_2\bigr)\\
 &\le\frac{8V}{\sqrt N}
       \bigl(|\sigma(0)|+\Lip(\sigma)R_KM_2\bigr).
 \end{split}
 \label{eq:lipschitz-barron-bound}
\end{equation}
At least one deterministic width-\(N\) realization satisfies the first
bound.  For \(V=0\) take the zero network; for empty \(K\) the error is zero.
\end{thm}
\begin{proof}
\label{app:proof-lipschitz-barron}
Put \(L=\Lip(\sigma)\) and \(b_0=|\sigma(0)|\).
We may assume $V>0$ and $K\ne\varnothing$.
The Radon--Nikodym property of a Hilbert space gives a strongly
$|\Gamma|$-measurable polar density $h$ with $\|h\|_\outY=1$
$|\Gamma|$-almost everywhere.  We use a Borel representative and change
it to zero on a Borel null set so that $\|h(\theta)\|_\outY\leq1$ everywhere.
Such a representative exists: a strongly measurable function is an
almost-everywhere limit of simple measurable functions; if the measure
has been completed, replace the finitely many sets in each of those
simple functions by Borel sets modulo null sets, and take the limit
outside the resulting Borel null set.  These modifications do not
change $\Gamma$, $f$, or the law of the sampled network.

Equip $\widetilde \cH=\cH\oplus\mathbb R$ with its Hilbert norm, and set
\[
 t=(a,b),\qquad z_x=(x,-1),\qquad
 \Phi(t)(x)=h(t)\sigma(\langle t,z_x\rangle).
\]
For every $t$, the function $\Phi(t)$ belongs to $C(K;\outY)$, and
\begin{equation}
 \|\Phi(t)\|_{C(K;\outY)}\leq b_0+LR_K\|t\|.
 \label{eq:supp-vector-envelope}
\end{equation}
We verify its strong measurability, rather than presupposing
measurability of an uncountable supremum.  The map
\[
 (t,y)\longmapsto
       \bigl[x\longmapsto y\sigma(\langle t,z_x\rangle)\bigr]
       \quad\hbox{from }\widetilde \cH\times \outY\hbox{ to }C(K;\outY)
\]
is continuous: the norm of the difference at $(t,y)$ and $(t',y')$ is
at most
\[
 LR_K\|t-t'\|\,\|y\|
       +(b_0+LR_K\|t'\|)\|y-y'\|.
\]
Its domain is separable, so its range is separable.  Composing it with
the Borel map $t\mapsto(t,h(t))$ gives a Borel, separably valued
$C(K;\outY)$-valued map, which is strongly measurable.
Cauchy--Schwarz gives $\int\|t\|\,dp\leq M_2$; hence
\eqref{eq:supp-vector-envelope} makes $\Phi$ Bochner integrable.
Evaluation at any $x\in K$ is continuous linear, so
$f|_K=V\int\Phi\,dp$ in $C(K;\outY)$.
The same envelope also shows that all random uniform norms used below
are measurable and integrable.

We first estimate the Rademacher sum conditionally on arbitrary fixed
samples $t_1,\ldots,t_N$.  Put $h_i=h(t_i)$ and
$B_i=b_0+LR_K\|t_i\|$.  Regard $\outY$ as a real Hilbert space with inner
product $(y,w)_{\mathbb R}=\operatorname{Re}\langle y,w\rangle_\outY$;
its norm is unchanged.  For $y$ in its closed unit ball $B_\outY$, put
$q_i(y)=(h_i,y)_{\mathbb R}$.  For every $w\in \outY$,
\[
 \|w\|_\outY=\sup_{y\in B_\outY}(w,y)_{\mathbb R}.
\]
This follows by Cauchy--Schwarz and, when $w\ne0$, by taking
$y=w/\|w\|_\outY$; for $w=0$ both sides vanish.

Choose countable norm-dense subsets $D_K\subset K$ and $D_\outY\subset B_\outY$,
with $0\in D_\outY$.  All processes below are continuous in $(x,y)$, so
their suprema on $K\times B_\outY$ equal their suprema on the countable
set $S=D_K\times D_\outY$.  On $S$ define
\[
 \psi_i(x,y)=\sigma(\langle t_i,z_x\rangle)q_i(y),\qquad
 u_i(x,y)=L\langle t_i,z_x\rangle,\qquad
 v_i(x,y)=B_iq_i(y).
\]
These functions are bounded.  Since $|q_i(y)|\leq1$ and
$|\sigma(\langle t_i,z_x\rangle)|\leq B_i$, splitting the difference
of the two products gives
\begin{align*}
 |\psi_i(x,y)-\psi_i(x',y')|
 &\leq L|\langle t_i,z_x-z_{x'}\rangle|
          +B_i|q_i(y)-q_i(y')|\\
 &=|u_i(x,y)-u_i(x',y')|+|v_i(x,y)-v_i(x',y')|.
\end{align*}
Lemma~\ref{lem:two-coordinate-comparison}, followed by real Hilbert
duality, therefore yields
\begin{align*}
 &\mathbb E_{\varepsilon}
      \sup_{x\in K}\Bigl\|\sum_{i=1}^N
                  \varepsilon_i h_i\sigma(\langle t_i,z_x\rangle)\Bigr\|_\outY\\
 &=\mathbb E_{\varepsilon}\sup_{(x,y)\in S}
                    \sum_{i=1}^N\varepsilon_i\psi_i(x,y)\\
 &\leq2\mathbb E_{\varepsilon_{i1},\varepsilon_{i2}}
      \sup_{x\in K,\ y\in B_\outY}
      \left(
       L\Bigl\langle\sum_{i=1}^N\varepsilon_{i1}t_i,z_x\Bigr\rangle
       +\Bigl(\sum_{i=1}^N\varepsilon_{i2}B_i h_i,y\Bigr)_{\mathbb R}
      \right)\\
 &\leq2LR_K\mathbb E\Bigl\|\sum_{i=1}^N\varepsilon_{i1}t_i\Bigr\|
       +2\mathbb E\Bigl\|\sum_{i=1}^N\varepsilon_{i2}B_i h_i\Bigr\|_\outY\\
 &\leq2LR_K\left(\sum_{i=1}^N\|t_i\|^2\right)^{1/2}
          +2\left(\sum_{i=1}^NB_i^2\right)^{1/2}\\
 &\leq2b_0\sqrt N+4LR_K\left(\sum_{i=1}^N\|t_i\|^2\right)^{1/2}.
\end{align*}
Here the penultimate inequality follows by expanding the squared
Hilbert norms: mixed terms have zero sign expectation, and
$\|h_i\|_\outY\leq1$.  The last inequality is the triangle inequality in
$\mathbb R^N$ applied to $(B_i)_i=(b_0)_i+LR_K(\|t_i\|)_i$.
These arguments also cover $L=0$ or $M_2=0$ without division by either
quantity.  Jensen's inequality now gives
\[
 \mathbb E_{t,\varepsilon}
   \left\|\frac1N\sum_{i=1}^N\varepsilon_i\Phi(t_i)\right\|_{C(K;\outY)}
 \leq\frac{2b_0+4LR_KM_2}{\sqrt N},
\]
because
$\mathbb E(\sum_i\|t_i\|^2)^{1/2}\leq\sqrt N M_2$.

For completeness, symmetrization here needs no Hilbert structure in
$C(K;\outY)$.  Let $Z_i=\Phi(t_i)$, let $Z'_i$ be independent copies, and
let $m=\mathbb EZ_1$.  Conditional Jensen, exchangeability of each
independent pair $(Z_i,Z'_i)$, and the triangle inequality imply
\begin{align*}
 \mathbb E\left\|\frac1N\sum_{i=1}^N Z_i-m\right\|_{C(K;\outY)}
 &\leq\mathbb E\left\|\frac1N\sum_{i=1}^N(Z_i-Z'_i)\right\|_{C(K;\outY)}\\
 &=\mathbb E\left\|\frac1N\sum_{i=1}^N
                         \varepsilon_i(Z_i-Z'_i)\right\|_{C(K;\outY)}\\
 &\leq2\mathbb E\left\|\frac1N\sum_{i=1}^N
                         \varepsilon_i Z_i\right\|_{C(K;\outY)}.
\end{align*}
Multiplying by $V$ proves the first inequality in
\cref{eq:lipschitz-barron-bound}; the second uses $4b_0\leq8b_0$.
Finally, a nonnegative integrable random variable with expectation at
most $A$ cannot exceed $A$ almost surely: otherwise its difference
from $A$ would be strictly positive almost surely and have positive
integral.  Thus some sample satisfies the first displayed bound,
which proves the deterministic assertion.
\end{proof}

The Hilbert-valued variance identity is given in
\cref{subsec:vector-sampling}. Concentration for bounded parameters and
the separation of input truncation from finite-width approximation
error are additional consequences recorded in \cref{app:sampling}.

\subsection{Finite variation from the spectral density}
\label{subsec:spectral-moments}

The coefficient \(\gamma_g\) of \cref{thm:A} lies in
\(L^2(\lambda_\alpha)\), and \(\lambda_\alpha\) is an infinite measure,
so finite total variation and moments of the coefficient measure require
additional hypotheses.  The condition in \cref{def:ray-regular} supplies them;
this is the content of the following theorem, which together with
\cref{thm:D} below is the main result of this section.

\begin{thm}[Finite total variation and moments of the coefficient measure]
\label{thm:E}
Let \(\rho\) be a band-pass filter, fix the frequency window \(I\) in
\cref{def:ray-regular}, and let \(g\) satisfy that definition.  For every
integer \(r\ge0\) there is a constant \(c_{\rho,I,r}\), independent of
\(g\), such that
\begin{equation}
 \int_{\cH\times\R}(1+\|a\|+|b|)^r|\gamma_g(a,b)|
      \,\dd\lambda_\alpha
 \le c_{\rho,I,r}M_{r+2}(g)<\infty.
 \label{eq:spectral-moments-all}
\end{equation}
The same estimate holds for \(\outY\)-valued densities, with norms and a
constant independent of the dimension of \(\outY\).  In particular,
\begin{equation}
  \int_{\cH\times\R}(1+\|a\|^2+|b|^2)\,|\gamma_g(a,b)|
  \,\lambda_\alpha(\dd a,\dd b)\le c_{\rho,I,2}\,M_4(g)<\infty .
  \label{eq:spectral-moments}
\end{equation}
Consequently, for every real globally Lipschitz \(\sigma\), the target
\(\iiprod{\sigma,\rho}_{\alpha}f_g\) is the integral
network \(S_\sigma[\gamma_g\lambda_\alpha]\), and its sampled network \cref{eq:polar-network} satisfies, with
\(V=\|\gamma_g\|_{L^1(\lambda_\alpha)}\) and \(M_2\) as in
\cref{eq:second-moment} for \(p=|\gamma_g|\lambda_\alpha/V\) when
\(V>0\) (take the zero network when \(V=0\)),
\begin{equation}
  \EE\bigl\|f_N-\iiprod{\sigma,\rho}_{\alpha}f_g\bigr\|_{C(K)}
  \le\frac{8V}{\sqrt N}
  \left(|\sigma(0)|+\Lip(\sigma)R_KM_2\right)
  \qquad\text{for every compact }K\subset \cH .
  \label{eq:spectral-barron}
\end{equation}
\end{thm}

\begin{proof}
\label{app:proof-thm-E}
For each integer \(r\ge0\), \cref{lem:coefficient-finite-order} gives
\[
 \int(1+\|a\|+|b|)^r\|\gamma_g(a,b)\|_\outY\,\dd\lambda_\alpha
 \le c_{\rho,I,r}A_{r+2,r}(g)
 \le c_{\rho,I,r}M_{r+2}(g).
\]
This proves all moments of the coefficient measure, including finite total variation at \(r=0\) and
\cref{eq:spectral-moments} at \(r=2\), independently of synthesis.
For globally Lipschitz real \(\sigma\),
\(|\sigma(\langle a,x\rangle-b)|
\le|\sigma(0)|+\Lip(\sigma)(\|a\|\|x\|+|b|)\).
The first moment therefore makes the synthesis absolutely integrable,
and \cref{thm:A}(iii) identifies it with
\(\iiprod{\sigma,\rho}_{\alpha}f_g\), including when the constant is zero.
If \(V>0\), the second moment gives \cref{eq:second-moment} for the
sampling distribution \(p=|\gamma_g|\lambda_\alpha/V\), and \cref{thm:lipschitz-barron} proves
\cref{eq:spectral-barron}.  If \(V=0\), the coefficient and its
synthesis vanish and the zero network suffices.
\end{proof}

\subsection{Constructive universal approximation}
\label{subsec:universality}

We now state the approximation theorem that is the main result of the
paper from the point of view of approximation theory.

The following qualitative result supplies discretization when no
explicit Lipschitz rate is required.

Use the polar decomposition \cref{eq:polar-decomposition} with \(V>0\);
the zero measure is represented by the zero network.

\begin{lem}[Qualitative finite-atomic approximation]
\label{lem:qualitative-sampling}
Let \(\sigma:\R\to\C\) be continuous.  For compact \(K\subset \cH\), assume
\(\int\|\sigma(\langle a,\cdot\rangle-b)\|_{C(K)}\dd|\Gamma|(a,b)<\infty\).
Then for every \(\eps>0\) there is a finite atomic complex measure
\(\Gamma_\eps\) with \(\|S_\sigma\Gamma_\eps-S_\sigma\Gamma\|_{C(K)}<\eps\).
\end{lem}

\begin{proof}
The atom map \((a,b)\mapsto\sigma(\langle a,\cdot\rangle-b)|_K\) is
continuous into the separable Banach space \(C(K)\): convergent
parameters have uniformly bounded preactivations on \(K\), where
continuity of \(\sigma\) is uniform.  Thus
\(\outY(a,b)=h(a,b)\sigma(\langle a,\cdot\rangle-b)|_K\) is Bochner integrable,
and its mean belongs to the closed convex hull of its essential range
by Hahn--Banach separation.  Approximate the mean by a finite convex
combination of essential-range values, then each value by an actual
\(\outY(a_j,b_j)\), and multiply by \(V\).
\end{proof}

\begin{thm}[Constructive universal approximation with approximation rates]
\label{thm:D}
Let \(\sigma:\R\to\R\) be continuous, of polynomial growth, and not a
polynomial, and let \(\rho\) be a band-pass filter with
\(\iiprod{\sigma,\rho}_{\alpha}=1\), which exists by
\cref{thm:tempered-reconstruction}.  Let \(f:\cH\to\C\) be continuous, let
\(K\subset \cH\) be compact, and let \(\eps>0\).  Then there is a smooth
spectral density \(g\), vanishing outside a bounded set and satisfying
\cref{def:ray-regular}, such that:
\begin{enumerate}[label=\textup{(\roman*)}]
\item
\(\|f-f_g\|_{C(K)}<\eps\);
\item
\(f_g=S_\sigma[\gamma_g\lambda_\alpha]\) is an integral network with
activation \(\sigma\) whose coefficient measure \(\gamma_g\lambda_\alpha\)
is finite and has finite moments of all orders;
\item
if \(\sigma\) is globally Lipschitz, the sampled network \(f_N\) of \(\gamma_g\lambda_\alpha\) satisfies
\begin{equation}
  \EE\|f-f_N\|_{C(K)}
  \le\eps+\frac{8V}{\sqrt N}\left(|\sigma(0)|+\Lip(\sigma)R_KM_2\right),
  \label{eq:universal-rate}
\end{equation}
with \(V=\|\gamma_g\|_{L^1(\lambda_\alpha)}\) and \(M_2\) the second
moment of \(|\gamma_g|\lambda_\alpha/V\); at least one deterministic
width-\(N\) network satisfies the same bound.
\end{enumerate}
In particular the finite-width networks with activation \(\sigma\) are
dense in \(C(\cH)\) for the compact-open topology.  The same statements
hold for continuous \(f:\cH\to \outY\), with the norms in (i) and (iii)
replaced by \(C(K;\outY)\); the same explicit bound follows from
\cref{thm:lipschitz-barron}.
\end{thm}

\begin{proof}
\label{app:proof-thm-D}
\emph{Step 1: characters are dense.}  The functions
\(\chi_\xi(x):=e^{i\langle x,\xi\rangle}\), \(\xi\in \cH\), restricted to
\(K\), satisfy \(\chi_\xi\chi_{\xi'}=\chi_{\xi+\xi'}\),
\(\overline{\chi_\xi}=\chi_{-\xi}\), and \(\chi_0=1\), so their finite
linear combinations form a self-conjugate subalgebra of \(C(K)\)
containing the constants.  If \(x\ne x'\) in \(K\), then
\(\xi=t(x-x')\) with \(0<t\|x-x'\|^2<2\pi\) gives
\(\chi_\xi(x)\ne\chi_\xi(x')\).  By the Stone--Weierstrass theorem there
are \(w_1,\dots,w_J\in\C\) and \(\xi_1,\dots,\xi_J\in \cH\) with
\(\|f-\sum_jw_j\chi_{\xi_j}\|_{C(K)}<\eps/2\).

\emph{Step 2: characters have spectral densities.}  Let
\(\varphi\in C_c^\infty(\R)\) be nonnegative, positive on \((-1,1)\), and
zero outside \([-1,1]\), and for \(\delta>0\) put
\(g_j(\xi):=\varphi(\|\xi-\xi_j\|^2/\delta^2)/Z_j\), where
\(Z_j:=\int\varphi(\|\xi-\xi_j\|^2/\delta^2)\nu_\alpha(\dd\xi)\) is
positive because \(\nu_\alpha\) has full support and is finite on
bounded sets.  Then \(g_j\ge0\), \(\int g_j\dd\nu_\alpha=1\), and
\(g_j\) vanishes outside the ball of radius \(\delta\) around \(\xi_j\).
For \(x\in K\),
\[
  |\chi_{\xi_j}(x)-f_{g_j}(x)|
  =\Bigl|\int_\cH\bigl(e^{i\langle x,\xi_j\rangle}-e^{i\langle x,\xi\rangle}\bigr)
  g_j(\xi)\nu_\alpha(\dd\xi)\Bigr|
  \le\sup_{\|\xi-\xi_j\|\le\delta}|\langle x,\xi_j-\xi\rangle|
  \le r_K\delta ,
\]
with \(r_K=\sup_K\|x\|\).  Choose \(\delta\) with
\(r_K\delta\sum_j|w_j|<\eps/2\) and put \(g:=\sum_jw_jg_j\).  Then
\(\|f-f_g\|_{C(K)}<\eps\), which is (i).

\emph{Step 3: regularity.}  Each \(g_j\) is bounded and vanishes outside
a bounded set.  Its restrictions \(\omega\mapsto g_j(\omega a)\) are
\(C^\infty\), because \(\omega\mapsto\|\omega a-\xi_j\|^2\) is a polynomial,
and their derivatives on \(I\) are bounded by polynomials in \(\|a\|\)
by the chain rule.  Hence \(g\) satisfies \cref{def:ray-regular}
by \cref{lem:ray-regular-examples}(b) and the
finite-combination clause of (c). \Cref{thm:A}(iii)
with \(\iiprod{\sigma,\rho}_{\alpha}=1\) gives
\(f_g=\int\gamma_g\sigma(\langle a,\cdot\rangle-b)\dd\lambda_\alpha\), and
the moment bound of \cref{thm:E}, with an arbitrary integer order, shows
that \(\gamma_g\lambda_\alpha\) is finite with finite moments of all
orders; this is (ii).  Part (iii) is \cref{thm:lipschitz-barron} applied
to \(f_g\), together with the triangle inequality and (i).  Density of
finite-width networks follows from (iii), or from (ii) and
\cref{lem:qualitative-sampling} when \(\sigma\) is merely continuous.

\emph{Vector-valued case.}  For \(f\in C(\cH;\outY)\) and compact \(K\), \(f\)
is uniformly continuous on \(K\); given \(\eps\), choose a finite cover of
\(K\) by balls of radius \(\delta'\) on which \(f\) varies by less than
\(\eps/4\), a continuous partition of unity \((\pi_i)\) subordinate to
it, and points \(x_i\) in the balls.  Then
\(\|f-\sum_i\pi_if(x_i)\|_{C(K;\outY)}<\eps/4\), and each scalar \(\pi_i\) is
approximated by a trigonometric sum as in Step~1.  The resulting
approximant is \(\sum_jw_j\chi_{\xi_j}\) with \(w_j\in \outY\), and Steps 2
and 3 apply with \(\outY\)-valued \(g=\sum_jw_jg_j\), using
\cref{thm:vector-valued}.  The uniform rate for the resulting vector
coefficient measure follows from \cref{thm:lipschitz-barron}.
\end{proof}

The coefficient is constructed on the full direction space \(\cH\),
without choosing an input truncation first.  A sampled network still
depends on only finitely many linear measurements of its input.
The constants \(V\) and \(M_2\) depend on the spectral approximant,
and hence on \(f,K,\eps\); the theorem does not quantify this first
approximation step.  If \(f\) already has a density satisfying
\cref{thm:E}, sampling applies directly with no preliminary error.
The weaker Sobolev route of \cref{thm:weak-sobolev-synthesis} gives
the same conclusion for its class when \(s>5/2\) and the
reconstruction constant is nonzero.

\subsection{Vector-valued approximation}
\label{subsec:vector-sampling}

Let \(\Gamma\) be a \(\outY\)-valued measure of bounded variation on
\(\cH\times\R\).  Since \(\outY\) is a Hilbert space, \(\Gamma=h|\Gamma|\) with
\(\|h\|_\outY=1\) \(|\Gamma|\)-almost everywhere, and the sampled network \cref{eq:polar-network} is defined with \(\outY\)-valued weights.

\begin{lem}[Mean-square error identity for Hilbert-valued averages]
\label{lem:hilbert-sampling}
Let \(X\) be a separable Hilbert space and \(\outY\in L^2(p;X)\).  For
independent copies \(\outY_j\), put \(f=V\EE \outY\) and
\(f_N=VN^{-1}\sum_jY_j\).  Then
\begin{equation}
 \EE\|f_N-f\|_X^2
 =\frac{V^2}N\bigl(\EE\|\outY\|_X^2-\|\EE \outY\|_X^2\bigr)
 \le\frac{V^2}N\EE\|\outY\|_X^2,
 \label{eq:hilbert-variance}
\end{equation}
and a deterministic sample satisfies the same upper bound.
\end{lem}

\begin{proof}
Write \(Z_j=\outY_j-\EE \outY\).  Independence and zero means make
\(\EE\langle Z_j,Z_k\rangle=0\) for \(j\ne k\), so the expanded squared
norm leaves \(N\EE\|Z_1\|^2\).
\end{proof}

\begin{cor}[Vector-valued rates]
\label{cor:vector-rates}
Let \(\sigma\) be real and globally Lipschitz, and let \(N\ge1\).
If \(V=0\), all conclusions hold for the zero network.  Otherwise put
\(p=|\Gamma|/V\).
\begin{enumerate}[label=\textup{(\roman*)}]
\item
Suppose \cref{eq:second-moment} holds.  For every Borel probability
measure \(\zeta\) on \(\cH\) with
\(\int\|x\|^2\zeta(\dd x)<\infty\),
\begin{equation}
 \mathbb E\|f_N-f\|_{L^2(\zeta;\outY)}^2
 =\frac1N\left(
 V^2\int_{\cH\times\R}
   \|\sigma(\langle a,\cdot\rangle-b)\|_{L^2(\zeta)}^2\,p(\dd a,\dd b)
 -\|f\|_{L^2(\zeta;\outY)}^2\right).
 \label{eq:vector-variance-exact}
\end{equation}
Consequently,
\begin{align*}
  \EE\|f_N-f\|_{L^2(\zeta;\outY)}^2
  &\le\frac{V^2}{N}\int_{\cH\times\R}
  \|\sigma(\langle a,\cdot\rangle-b)\|_{L^2(\zeta)}^2\,p(\dd a,\dd b)
  \\
  &\le\frac{2V^2}{N}\Bigl(|\sigma(0)|^2
  +\Lip(\sigma)^2\bigl(1+\textstyle\int\|x\|^2\dd\zeta\bigr)M_2^2\Bigr).
\end{align*}
\item
For this part it suffices that \(\int(\|a\|+|b|)\,\dd p<\infty\).
For every compact \(K\subset \cH\),
\(\EE\|f_N-f\|_{C(K;\outY)}\le2V\,\mathfrak R_N^\outY(K;p,\sigma)\), where
\(\mathfrak R_N^\outY\) is \cref{eq:rademacher-complexity} with the absolute
value replaced by the norm of \(\outY\), and
\(\mathfrak R_N^\outY(K;p,\sigma)\to0\) as \(N\to\infty\).
\end{enumerate}
\end{cor}

\begin{proof}
Part (i) is the Hilbert-valued variance identity of
\cref{lem:hilbert-sampling} applied in \(X=L^2(\zeta;\outY)\).
The atom has squared norm
\(V^2\|\sigma(\langle a,\cdot\rangle-b)\|_{L^2(\zeta)}^2\) because
\(\|h\|_\outY=1\) almost everywhere.  Its second moment is finite by
\(|\sigma(u)|^2\le2|\sigma(0)|^2+2\Lip(\sigma)^2u^2\) and
\(u^2\le(\|a\|^2+b^2)(\|x\|^2+1)\), which also give the bound.
For (ii), the map \((a,b,y)\mapsto
[x\mapsto y\sigma(\langle a,x\rangle-b)]\) is continuous into
\(C(K;\outY)\); composing with the measurable polar density gives a
strongly measurable atom.  Its norm is at most
\(|\sigma(0)|+\Lip(\sigma)R_K\sqrt{\|a\|^2+b^2}\), which is integrable
under the first-moment assumption.  Apply
\cref{lem:banach-rademacher-vanishing} in the separable Banach space
\(C(K;\outY)\).
\end{proof}

Under a second moment, \cref{thm:lipschitz-barron} already gives
an explicit uniform rate for the whole Hilbert-valued output.
The variance formula records the corresponding exact \(L^2\) identity.
The first-moment conclusion requires less integrability and extends,
through \cref{lem:banach-rademacher-vanishing}, to integrable atoms in
any separable Banach space.
 \section{Examples on infinite-dimensional spaces}
\label{sec:examples}
\label{app:examples}

We apply the representation and approximation results to explicit targets
on infinite-dimensional spaces.  Gaussian targets and nonlinear layers
built from convolution and a Dirichlet solution operator demonstrate how
the coefficients and transform formulas can be computed.

A function \(f\) on \(\cH\) is cylindrical if \(f=\widetilde f\circ L\) for
a finite-rank linear map \(L\), continuous or not.  The examples below
identify explicit conditions under which the exact targets are
non-cylindrical.  Their finite-width approximants use finitely many
linear measurements, with directions selected from the coefficient law.  We write
\(\phi_v\) for the centred one-dimensional Gaussian density of variance
\(v>0\), with \(\phi_0=\delta_0\), and \(\sigma_{\mathrm{Gauss}}(u):=e^{-u^2/2}\) for the
Gaussian activation.

\subsection{A Gaussian target with a closed-form transform}
\label{subsec:closed-form}

\begin{lem}[Gaussian integral of a quadratic exponential]
\label{lem:gaussian-quadratic}
Let \(\Xi\) be a positive self-adjoint trace-class operator, \(S\) a
bounded positive self-adjoint operator, and \(x\in \cH\).  Put
\(M:=\Xi^{1/2}S\Xi^{1/2}\), which is trace class.  Then
\[
  \int_\cH e^{i\langle x,\xi\rangle-\langle S\xi,\xi\rangle/2}\,\mu_{\Xi}(\dd\xi)
  =\det(I+M)^{-1/2}
  \exp\Bigl(-\tfrac12\bigl\langle\Xi^{1/2}(I+M)^{-1}\Xi^{1/2}x,x\bigr\rangle\Bigr).
\]
\end{lem}

\begin{proof}
Let \((u_j)\) be an orthonormal eigenbasis of \(M\) with eigenvalues
\(m_j\ge0\), and let \(\eta_j\) be independent standard normal
variables.  The series \(\xi:=\sum_j\eta_j\Xi^{1/2}u_j\) converges in
\(L^2(\Omega;\cH)\) because \(\sum_j\|\Xi^{1/2}u_j\|^2=\tr\Xi<\infty\),
and its characteristic functional is
\(\exp(-\frac12\sum_j\langle\Xi^{1/2}u_j,z\rangle^2)=e^{-\langle\Xi z,z\rangle/2}\),
so \(\xi\sim\mu_{\Xi}\).  For the partial sums \(\xi_n\),
\(\langle S\xi_n,\xi_n\rangle=\sum_{j,k\le n}\eta_j\eta_k\langle Mu_j,u_k\rangle
=\sum_{j\le n}m_j\eta_j^2\), which converges almost surely to
\(\sum_jm_j\eta_j^2\) since \(\sum_jm_j<\infty\); continuity of the
quadratic form along an almost surely convergent subsequence gives
\(\langle S\xi,\xi\rangle=\sum_jm_j\eta_j^2\).  Likewise
\(\langle x,\xi\rangle=\sum_j\eta_jx_j\) with
\(x_j:=\langle\Xi^{1/2}x,u_j\rangle\).  By independence and dominated
convergence,
\[
  \EE\,e^{i\langle x,\xi\rangle-\langle S\xi,\xi\rangle/2}
  =\prod_j\EE\,e^{i\eta_jx_j-m_j\eta_j^2/2}
  =\prod_j(1+m_j)^{-1/2}e^{-x_j^2/(2(1+m_j))},
\]
and \(\sum_jx_j^2/(1+m_j)=\langle(I+M)^{-1}\Xi^{1/2}x,\Xi^{1/2}x\rangle\).
\end{proof}

\begin{ex}[Closed-form transform and integral representation of \(T_\alpha f_W\)]
\label{ex:closed-form}
Let \(W\) be a bounded, positive, injective, self-adjoint operator such
that \(M:=Q^{1/2}WQ^{1/2}\) is trace class, and let
\begin{equation}
  f_W(x):=e^{-\langle Wx,x\rangle/2},
  \qquad
  D:=\det(I+M),
  \qquad
  \kappa_W(\xi):=\langle Q^{1/2}(I+M)^{-1}Q^{1/2}\xi,\xi\rangle .
  \label{eq:gaussian-target}
\end{equation}
Then, for every \(\alpha>0\) and every band-pass \(\rho\):
\begin{enumerate}[label=\textup{(\roman*)}]
\item
\(F_Qf_W(\xi)=D^{-1/2}e^{-\kappa_W(\xi)/2}\) and
\(R_\rho f_W(a,b)=D^{-1/2}(\rho*\phi_{\kappa_W(a)})(-b)\); in particular
\(f_W\in\cD_\alpha\), and \(f_W\) is not cylindrical when \(W\) has
infinite rank.
\item
\(g:=F_Qf_W\) satisfies \cref{def:ray-regular}, and the image
\(T_\alpha f_W\) is represented by
\begin{equation}
  f_g(x)=D^{-1/2}\int_0^\infty
  \det\bigl(I+2sP^{1/2}S_WP^{1/2}\bigr)^{-1/2}
  \exp\Bigl(-\tfrac12\bigl\langle\Xi_sx,x\bigr\rangle\Bigr)
  s^{\alpha/2-1}\,\dd s,
  \label{eq:filtered-gaussian-target}
\end{equation}
where \(S_W:=Q^{1/2}(I+M)^{-1}Q^{1/2}\) and
\(\Xi_s:=2sP^{1/2}(I+2sP^{1/2}S_WP^{1/2})^{-1}P^{1/2}\).
\item
For every real, globally Lipschitz, non-polynomial \(\sigma\), including
ReLU, the coefficient measure \(R_\rho f_W\,\lambda_\alpha=\gamma_g\lambda_\alpha\)
has finite total variation and second moment, the integral network \(S_\sigma[R_\rho f_W\,\lambda_\alpha]\)
equals \(\iiprod{\sigma,\rho}_{\alpha}f_g\), and its sampled network converges at the rate \(N^{-1/2}\) in \(C(K)\) as in
\cref{eq:spectral-barron}.
\end{enumerate}
\end{ex}

\begin{proof}
\label{app:proof-closed-form}
\emph{(i).}  \Cref{lem:gaussian-quadratic} with \(\Xi=Q\), \(S=W\),
and \(x\) replaced by \(-\xi\) gives
\(F_Qf_W(\xi)=D^{-1/2}e^{-\kappa_W(\xi)/2}\).  By
\cref{eq:fourier-slice}, the bias Fourier transform of
\(R_\rho f_W(a,\cdot)\) is \(D^{-1/2}\rho^\sharp(-\omega)e^{-\omega^2\kappa_W(a)/2}\),
whose inverse Fourier transform is \(D^{-1/2}(\rho*\phi_{\kappa_W(a)})(-b)\).
Since \((I+M)^{-1}\ge(1+\|M\|)^{-1}I\), we have
\(\kappa_W(\xi)\ge(1+\|M\|)^{-1}\langle Q\xi,\xi\rangle\), and
\cref{lem:gaussian-decay} gives \(f_W\in\cD_\alpha\).  If
\(f_W=\widetilde F\circ L\) with \(L\) of finite rank and \(W\) of
infinite rank, take \(0\ne x\in\ker L\); then \(f_W(x)=f_W(0)=1\), while
\(\langle Wx,x\rangle=\|W^{1/2}x\|^2>0\) gives \(f_W(x)<1\).

\emph{(ii).}  \(g=F_Qf_W\in L^1(\nu_\alpha)\) by (i), so
\cref{thm:C}(iii) represents \(T_\alpha f_W\) by \(f_g\).  By
\cref{lem:mixture-integration}, applied to the absolutely integrable
complex integrand,
\[
 f_g(x)=D^{-1/2}\int_0^\infty
 \left[\int_\cH e^{i\langle x,\xi\rangle-\langle S_W\xi,\xi\rangle/2}
       \,\mu_{2sP}(\dd\xi)\right]s^{\alpha/2-1}\,\dd s.
\]
Applying \cref{lem:gaussian-quadratic} with \(\Xi=2sP\) and \(S=S_W\)
evaluates the bracket and gives \cref{eq:filtered-gaussian-target}.

\emph{(iii).}  \(g\) has the form of \cref{lem:ray-regular-examples}(a)
with \(q=D^{-1/2}\) and \(S=S_W\ge(1+\|M\|)^{-1}Q\), so it satisfies
\cref{def:ray-regular}.
\Cref{eq:slice-as-coefficient} gives \(R_\rho f_W=W_\rho g=\gamma_g\),
and \cref{thm:E} gives finite total variation of the coefficient measure, the identity
\(S_\sigma[\gamma_g\lambda_\alpha]=\iiprod{\sigma,\rho}_{\alpha}f_g\), and the
rate.
\end{proof}

Of the standing hypotheses on \(W\), the transform and convolution formulas
of part (i) and the membership \(f_W\in\cD_\alpha\) use only that \(W\) is
bounded, positive and self-adjoint with \(M\) trace class; injectivity of
\(W\) is needed for the non-cylindricity of \(f_W\), and the analysis filter
does not appear in part (ii) at all.

Formula \cref{eq:filtered-gaussian-target} displays the frame operator concretely: the Gaussian target and its image \(T_\alpha f_W\) are both explicit,
and the ridgelet coefficient of the former is the network coefficient of
the latter.

\subsection{Gaussian-parameter networks on \texorpdfstring{\(\ell^2\)}{l2}}
\label{subsec:gaussian-parameter}

Let \(\cH=\ell^2(\N)\) with standard basis \((e_j)\), and let
\(Qe_j=q_je_j\) with \(q_j>0\) and \(\sum_jq_j<\infty\), so that \(Q\) is
injective, trace class, and of infinite rank.

\begin{lem}[Absolute hinge representation of the Gaussian]
\label{lem:gaussian-hinge}
For \(\phi(u)=e^{-u^2/2}\), \(\phi(u)=\int_\R(u-b)_+\phi''(b)\dd b\)
absolutely for each \(u\), and
\(\int_\R(1+|b|^k)|\phi''(b)|\dd b<\infty\) for every \(k\ge0\).
\end{lem}

\begin{proof}
Since \(\phi''(b)=(b^2-1)e^{-b^2/2}\), all polynomially weighted absolute
integrals are finite.  Integrating by parts on \((-L,u)\) gives
\(\int_{-L}^u(u-b)\phi''(b)\dd b=\phi(u)-\phi(-L)-(u+L)\phi'(-L)\), which
tends to \(\phi(u)\) as \(L\to\infty\).
\end{proof}

\begin{ex}[ReLU and Gaussian networks with Gaussian parameters]
\label{ex:gaussian-parameter}
The integral networks
\begin{equation}
  f_{\mathrm{ReLU},Q}(x):=\int_\cH\ReLU(\langle a,x\rangle)\,\mu_{Q}(\dd a),
  \qquad
  f_{\mathrm{Gauss},Q}(x):=\int_\cH\sigma_{\mathrm{Gauss}}(\langle a,x\rangle)\,\mu_{Q}(\dd a)
  \label{eq:gaussian-parameter-networks}
\end{equation}
have the closed forms
\begin{equation}
 \begin{aligned}
  f_{\mathrm{ReLU},Q}(x)&=\sqrt{\frac{\langle Qx,x\rangle}{2\pi}},\qquad
  f_{\mathrm{Gauss},Q}(x)=\frac1{\sqrt{1+\langle Qx,x\rangle}},\\
  f_{\mathrm{Gauss},Q}(x)&=\int_{\cH\times\R}
  \ReLU(\langle a,x\rangle-b)\,\phi''(b)\,\mu_Q(\dd a)\,\dd b.
 \end{aligned}
  \label{eq:gaussian-parameter-closed-forms}
\end{equation}
where \(\phi(u)=e^{-u^2/2}\) and \(\phi''(b)=(b^2-1)e^{-b^2/2}\).  Neither
target is cylindrical.  The second identity shows that a Gaussian-activation
network with finite coefficient measure is also a ReLU network with finite
coefficient measure, with all parameter moments finite.
\end{ex}

\begin{proof}
For fixed \(x\), \(\langle a,x\rangle\) is centred Gaussian with variance
\(v=\langle Qx,x\rangle\); if \(Z\sim\mu_{v}\) then
\(\EE Z_+=\sqrt{v/2\pi}\) and \(\EE e^{-Z^2/2}=(1+v)^{-1/2}\).  Integrating
\(\phi''\) twice from \(-\infty\) gives \(\phi(u)=\int_\R(u-b)_+\phi''(b)\dd b\)
(\cref{lem:gaussian-hinge}), and Fubini is justified by
\(\int|\phi''|<\infty\) and finite Gaussian first moments.  If a target
factored through a finite-rank map, a nonzero vector in its kernel would
give the same value as \(x=0\), contradicting injectivity of \(Q\).
\end{proof}

\begin{cor}[Discretization of the Gaussian-parameter ReLU network]
\label{cor:relu-discretization}
Let \(a_1,\dots,a_N\) be independent with law \(\mu_{Q}\) and
\(f_{\mathrm{ReLU},Q,N}(x):=\frac1N\sum_{j=1}^N\ReLU(\langle a_j,x\rangle)\).  Then, for
every compact \(K\subset \cH\),
\(\EE\|f_{\mathrm{ReLU},Q,N}-f_{\mathrm{ReLU},Q}\|_{C(K)}\le 8R_K\sqrt{\tr Q}/\sqrt N\).
\end{cor}

\begin{proof}
Apply \cref{thm:lipschitz-barron} with \(V=1\), \(b=0\), \(\sigma(0)=0\),
\(\Lip(\sigma)=1\), and \(M_2^2=\EE\|a\|^2=\tr Q\).
\end{proof}

\subsection{Neural-operator layers}
\label{subsec:operator-layer}

The layers of neural operators \citep{Kovachki2021NOreview} compose an integral
operator, a pointwise nonlinearity, and another integral operator.  We
show that such a layer is an integral network in the sense of
\cref{def:integral-network}, with an explicit finite parameter measure,
and that with the Gaussian activation its ridgelet transform is explicit
and the whole theory of \cref{sec:reconstruction,sec:finite-width}
applies.

Let \((\Omega,m)\) be a finite measure space, let \(y\mapsto a_y\in \cH\)
be Borel and bounded, \(\|A\|_\infty:=\sup_y\|a_y\|<\infty\), and let
\(y\mapsto v_y\in \outY\) be Borel with \(\int_\Omega\|v_y\|_\outY m(\dd y)<\infty\).
Define the integral operators
\[
  (Ax)(y):=\langle a_y,x\rangle,
  \qquad
  Bu:=\int_\Omega v_y\,u(y)\,m(\dd y),
\]
and, for a continuous real activation \(\sigma\) of polynomial growth, the
layer
\begin{equation}
  F(x):=B\,\sigma(Ax)
  =\int_\Omega v_y\,\sigma(\langle a_y,x\rangle)\,m(\dd y),
  \qquad x\in \cH .
  \label{eq:operator-layer}
\end{equation}
For \(\varphi\in \outY\) the scalar observable is
\(F_\varphi(x):=\langle F(x),\varphi\rangle_\outY\), and
\(w_\varphi(y):=\langle v_y,\varphi\rangle_\outY\).

\begin{ex}[Neural-operator layer as an integral network]
\label{ex:operator-layer}
\begin{enumerate}[label=\textup{(\roman*)}]
\item
\emph{Network structure.}
\(F=S_\sigma[\Gamma]\) with the \(\outY\)-valued measure
\(\Gamma=\iota_\#(v_y\,m(\dd y))\), \(\iota(y)=(a_y,0)\).
Put \(B_1=\int_\Omega\|v_y\|_\outY\,m(\dd y)\).
Its variation satisfies \(V\le B_1\), and its normalized second
parameter moment is at most \(\|A\|_\infty^2\) when \(V>0\).
If \(\sigma\) is real and globally Lipschitz,
\cref{thm:lipschitz-barron} yields
\begin{equation}
 \EE\|F_N-F\|_{C(K;\outY)}
 \le\frac{B_1}{\sqrt N}
       \bigl(4|\sigma(0)|+8\Lip(\sigma)R_K\|A\|_\infty\bigr).
 \label{eq:operator-layer-uniform}
\end{equation}
Thus the bound controls the entire output function.  The scalar
observable has the same bound with \(B_1\) replaced by
\(\|w_\varphi\|_{L^1(m)}\).  The \(L^2(\zeta;\outY)\) variance
formula of \cref{cor:vector-rates} also applies when \(\zeta\)
has a finite second input moment.
\item
\emph{Gaussian activation: explicit transform.}
Let \(\sigma=\sigma_{\mathrm{Gauss}}\).  For \(\xi\in \cH\) put
\(\upsilon_y^2:=\langle Qa_y,a_y\rangle\) and
\(S_y:=Q-(1+\upsilon_y^2)^{-1}(Qa_y)\otimes(Qa_y)\).  Then
\(F_\varphi\in\cD_\alpha\) for every \(\alpha>0\), with
\begin{align}
  F_QF_\varphi(\xi)
  &=\int_\Omega w_\varphi(y)\,(1+\upsilon_y^2)^{-1/2}
  e^{-\langle S_y\xi,\xi\rangle/2}\,m(\dd y),
  \\
  R_\rho F_\varphi(a,b)
  &=\int_\Omega w_\varphi(y)\,(1+\upsilon_y^2)^{-1/2}
  \bigl(\rho*\phi_{\langle S_ya,a\rangle}\bigr)(-b)\,m(\dd y),
  \label{eq:operator-layer-transform}
\end{align}
These formulas are mixtures
of Gaussian Fourier factors and one-dimensional convolutions,
respectively,
and \(S_y\ge(1+\|Q\|\|A\|_\infty^2)^{-1}Q\).  Consequently
\(F_QF_\varphi\) satisfies \cref{def:ray-regular}, the reconstruction formulas of
\cref{thm:C} hold for \(F_\varphi\), and the ridgelet coefficient
\(R_\rho F_\varphi\) synthesizes, with any real Lipschitz non-polynomial
\(\sigma'\), the target
\(\iiprod{\sigma',\rho}_{\alpha}\,T_\alpha F_\varphi\) with the finite-width
rate of \cref{eq:spectral-barron}.  The same holds for \(F\)
itself as a \(\outY\)-valued target.
\item
\emph{ReLU form.}
By \cref{eq:gaussian-parameter-closed-forms}, the Gaussian-activation
layer is also the ReLU network
\(F(x)=\int_{\Omega\times\R}v_y\,\phi''(b)\,
\ReLU(\langle a_y,x\rangle-b)\,m(\dd y)\,\dd b\), whose coefficient
measure is finite with all moments finite.
\item
\emph{Non-cylindricity.}
If \(A\) has infinite rank, \(\sigma=\sigma_{\mathrm{Gauss}}\), and \(w_\varphi>0\)
\(m\)-almost everywhere, then \(F_\varphi\) is not cylindrical.
\end{enumerate}
\end{ex}

\begin{proof}
\label{app:proof-operator-layer}
\emph{(i).}  Since \(\|a_y\|\le\|A\|_\infty\) and \(\sigma\) has polynomial
growth, \(y\mapsto v_y\sigma(\langle a_y,x\rangle)\) is Bochner integrable
for each \(x\), and the change of variables for the pushforward measure
\(\Gamma=\iota_\#(v_ym(\dd y))\) gives
\(S_\sigma[\Gamma](x)=\int_\Omega v_y\sigma(\langle a_y,x\rangle)m(\dd y)=F(x)\).
The variation of a pushforward is at most the pushforward of the
variation, so \(\|\Gamma\|_{\rm TV}\le\int\|v_y\|m(\dd y)\), and
\(\int(\|a\|^2+b^2)\dd|\Gamma|\le\|A\|_\infty^2\|\Gamma\|_{\rm TV}\).  The
scalar observable has coefficient measure \(\iota_\#(w_\varphi m)\) with
total variation \(\|w_\varphi\|_{L^1(m)}\), and \cref{thm:lipschitz-barron}
gives the displayed bound.

\emph{(ii).}  Let \((Z,W)\) be a centred Gaussian pair with variances
\(\upsilon^2,\tau^2\) and covariance \(r\), and assume \(\upsilon>0\).  Given
\(Z\), \(W\) is Gaussian with mean \(rZ/\upsilon^2\) and variance
\(\tau^2-r^2/\upsilon^2\), so
\(\EE[e^{-iW}\mid Z]=e^{-irZ/\upsilon^2}e^{-(\tau^2-r^2/\upsilon^2)/2}\), and
\[
  \EE\bigl[e^{-Z^2/2}e^{-irZ/\upsilon^2}\bigr]
  =\int_\R e^{-z^2/2-irz/\upsilon^2}\frac{e^{-z^2/(2\upsilon^2)}}{\sqrt{2\pi\upsilon^2}}\dd z
  =(1+\upsilon^2)^{-1/2}\exp\Bigl(-\frac{r^2}{2\upsilon^2(1+\upsilon^2)}\Bigr),
\]
since the tilted density is Gaussian with variance \(\upsilon^2/(1+\upsilon^2)\).
Multiplying,
\(\EE[e^{-Z^2/2}e^{-iW}]=(1+\upsilon^2)^{-1/2}\exp(-\tau^2/2+r^2/(2(1+\upsilon^2)))\);
the case \(\upsilon=0\) is the limit.  Under \(\mu_Q\), the pair
\((\langle a_y,x\rangle,\langle x,\xi\rangle)\) is centred Gaussian with
\(\upsilon^2=\upsilon_y^2\), \(\tau^2=\langle Q\xi,\xi\rangle\), and
\(r=\langle Qa_y,\xi\rangle\), which gives
\(F_Q[\sigma_{\mathrm{Gauss}}(\langle a_y,\cdot\rangle)](\xi)=(1+\upsilon_y^2)^{-1/2}e^{-\langle S_y\xi,\xi\rangle/2}\).
Since \(\sigma_{\mathrm{Gauss}}\) is bounded and \(w_\varphi\in L^1(m)\), Fubini gives the
first formula in \cref{eq:operator-layer-transform}, and the Fourier
slice with the inverse Fourier transform of
\(\omega\mapsto e^{-\omega^2\langle S_ya,a\rangle/2}\) gives the second.
Cauchy--Schwarz for the positive form \(Q\) gives
\(\langle Qa_y,\xi\rangle^2\le\upsilon_y^2\langle Q\xi,\xi\rangle\), hence
\(\langle S_y\xi,\xi\rangle\ge(1+\upsilon_y^2)^{-1}\langle Q\xi,\xi\rangle
\ge(1+\|Q\|\|A\|_\infty^2)^{-1}\langle Q\xi,\xi\rangle\); in particular
\(F_\varphi\in\cD_\alpha\) by \cref{lem:gaussian-decay}.  Each
\(g_y(\xi):=(1+\upsilon_y^2)^{-1/2}e^{-\langle S_y\xi,\xi\rangle/2}\) is of
the form of \cref{lem:ray-regular-examples}(a) with \(\|S_y\|\le\|Q\|\)
and the uniform lower bound \(S_y\ge\theta Q\),
\(\theta:=(1+\|Q\|\|A\|_\infty^2)^{-1}\), just proved, and the family is
bounded by \(1\) because \((1+\upsilon_y^2)^{-1/2}\le1\).  Fix a compact
symmetric neighbourhood of \(I\) inside \(\R\setminus\{0\}\), let \(U\)
be its interior and \(r'>0\) the smallest \(|\omega|\) on it.  Since
\(\partial_\omega^k\bigl[(1+\upsilon_y^2)^{-1/2}e^{-\omega^2\langle S_ya,a\rangle/2}\bigr]\)
is a polynomial in \(\omega\) and \(\langle S_ya,a\rangle\) times
\(e^{-\omega^2\langle S_ya,a\rangle/2}\), and
\(\langle S_ya,a\rangle\le\|Q\|\|a\|^2\), the proof of part (a) gives
\[
  \sup_{\omega\in U}\bigl|\partial_\omega^kg_y(\omega a)\bigr|
  \le C_k(1+\|a\|)^{2k}e^{-r'^2\theta\langle Qa,a\rangle/2}=:h_k(a)
\]
for every \(y\) and every \(a\in \cH\), with \(C_k\) independent of \(y\),
and \cref{lem:gaussian-decay} makes the weights \(h_k\) satisfy the
moment bounds of \cref{lem:ray-regular-examples}(c).  Part (c) of that
lemma, applied with the finite measure \(|w_\varphi|m\) and the family
\((w_\varphi/|w_\varphi|)g_y\), which obeys the same bounds because the
extra factor has modulus at most one, shows that
\(F_QF_\varphi=\int w_\varphi(y)g_y\,m(\dd y)\) satisfies
\cref{def:ray-regular}. The
remaining claims are \cref{thm:C} and \cref{thm:E}, and the \(\outY\)-valued
statement is the same computation with Bochner integrals
(\cref{thm:vector-valued}).

\emph{(iii).}  By \cref{lem:gaussian-hinge},
\(\sigma_{\mathrm{Gauss}}(\langle a_y,x\rangle)=\int_\R\ReLU(\langle a_y,x\rangle-b)\phi''(b)\dd b\),
and Fubini is justified by
\(\int\|v_y\|\int(1+\|a_y\|\|x\|+|b|)|\phi''(b)|\dd b\,m(\dd y)<\infty\).

\emph{(iv).}  If \(F_\varphi=\widetilde F\circ L\) with \(L\) of finite
rank, then \(\ker L\) has finite codimension while \(\ker A\) has infinite
codimension, since \(\cH/\ker A\) is isomorphic to \(\Ran A\); hence there is
\(x\in\ker L\setminus\ker A\), and \(F_\varphi(tx)=\widetilde F(0)=F_\varphi(0)
=\int w_\varphi\dd m\) for all \(t\in\R\).  On the other hand
\(\sigma_{\mathrm{Gauss}}(t\langle a_y,x\rangle)\to\mathbf 1_{\{\langle a_y,x\rangle=0\}}\)
as \(t\to\infty\), boundedly, so
\(F_\varphi(tx)\to\int w_\varphi\mathbf 1_{\{Ax=0\}}\dd m\), which is
strictly smaller than \(\int w_\varphi\dd m\) because \(Ax\ne0\) on a set
of positive measure and \(w_\varphi>0\) almost everywhere.
\end{proof}

\begin{rmk}[Equivariance of the transform]
\label{rmk:equivariance}
If \(U\) is a unitary operator on \(\cH\) with \(UQU^*=Q\), then
\(U_\#\mu_Q=\mu_Q\) and, for \(f\in L^1(\mu_Q)\),
\(R_\rho[f\circ U](a,b)=R_\rho f(Ua,b)\); if also \(UPU^*=P\), then
\(\nu_\alpha\) and \(\lambda_\alpha\) are \(U\)-invariant and the
Plancherel identity, the frame operator, and the reconstruction formulas
commute with the action of \(U\).  This is the infinite-dimensional
analogue of the group equivariance of ridgelet transforms studied in
\citet{Sonoda2022gconv}.  On \(\R^d\), unlike on \(\mathbb T^d\), no
translation-invariant Gaussian measure with trace-class covariance
exists, so a Gaussian input weight on \(L^2(\R^d)\) necessarily breaks
translation invariance; the network structure (i) of
\cref{ex:operator-layer} does not involve the input weight and is
unaffected.
\end{rmk}

\begin{ex}[Periodic convolution layer]
\label{ex:convolution}
Let \(\Omega=\mathbb T^d=(\R/2\pi\Z)^d\) with normalized Lebesgue measure,
\(\cH=L^2(\mathbb T^d;\R)\), \(\outY=L^2(\mathbb T^d)\), and let
\(k,\psi\in L^2(\mathbb T^d;\R)\). For a kernel \(u\), its
Fourier coefficients are
\(\widehat u(n):=\int_{\mathbb T^d}u(y)e^{-i\langle n,y\rangle}\,\dd y\),
using normalized Haar measure. With \(a_y:=k(y-\cdot)\) and
\(v_y:=\psi(\cdot-y)\), the layer \cref{eq:operator-layer} is
\(F(x)=\psi*\sigma(k*x)\), a convolutional layer followed by a
smoothing convolution.  Here \(\|A\|_\infty=\|k\|_2\),
\(\int\|v_y\|\dd y=\|\psi\|_2\), and \(F\) commutes with all
translations of \(\mathbb T^d\), and with every isometry of \(\mathbb T^d\)
that fixes \(k\) and \(\psi\).  If \(\widehat k(n)\ne0\) for infinitely
many \(n\in\Z^d\), then \(A\) has infinite rank.  If additionally
\(\sigma=\sigma_{\mathrm{Gauss}}\), then with
\(\varphi\equiv1\) and \(\widehat\psi(0)\ne0\) the observable
\(F_1(x)=\widehat\psi(0)\int_{\mathbb T^d}\sigma((k*x)(y))\dd y\) is not
cylindrical.  Choosing \(Q=P=(I-\Delta)^{-s}\) with \(s>d/2\), which is
injective, trace class, and translation invariant, the transform is
equivariant as well: \(R_\rho[f\circ\tau_z](a,b)=R_\rho f(\tau_za,b)\)
for every translation \(\tau_z\) (\cref{rmk:equivariance}).
\end{ex}

\begin{proof}
With normalized Haar measure, \(\|a_y\|=\|k\|_2\) and \(\|v_y\|=\|\psi\|_2\)
for every \(y\), and \(y\mapsto a_y\), \(y\mapsto v_y\) are continuous
into \(L^2(\mathbb T^d)\).  Convolutions commute with translations and
with isometries fixing their kernels, and \(\sigma\) acts pointwise, so
\(F\) is equivariant.  In the Fourier basis \(A\) is diagonal
with entries \(\widehat k(n)\), so its rank is the number of nonzero
coefficients; with \(\varphi\equiv1\), \(w_1(y)=\int\psi=\widehat\psi(0)\)
is a nonzero constant.  For the Gaussian activation,
\cref{ex:operator-layer}(iv) applies after reversing the sign of the
observable if this constant is negative.  The
operator \((I-\Delta)^{-s}\) has eigenvalues \((1+|n|^2)^{-s}\), summable
for \(s>d/2\), is injective, and commutes with translations, so
\cref{rmk:equivariance} applies.
\end{proof}

\begin{ex}[Dirichlet solution operator with a pointwise nonlinearity]
\label{ex:dirichlet}
Let \(\Omega=(0,1)\) with Lebesgue measure, \(\cH=\outY=L^2(0,1)\), and let
\(L_D^{-1}:=(I-\partial_t^2)^{-1}\) with Dirichlet boundary conditions,
\[
  (L_D^{-1}x)(y)=\int_0^1g(y,t)\,x(t)\,\dd t,
  \qquad
  g(y,t)=\frac{\sinh(\min(y,t))\,\sinh(1-\max(y,t))}{\sinh 1},
\]
so that \(u=L_D^{-1}x\) solves \(-u''+u=x\), \(u(0)=u(1)=0\).  With
\(a_y=v_y=g(y,\cdot)\), the layer \cref{eq:operator-layer} is
\(F(x)=L_D^{-1}\,\sigma(L_D^{-1}x)\): the solution of the
linear problem whose source is a pointwise nonlinearity of the solution
with source \(x\), which is the nonlinear correction in the Picard update
\(u_1=L_D^{-1}x+F(x)\) starting from
\(u_0=L_D^{-1}x\) for \(-u''+u=\sigma(u)+x\).  The operator \(L_D^{-1}\) has eigenvalues
\(\lambda_n=(1+\pi^2n^2)^{-1}\) with eigenfunctions
\(e_n(t)=\sqrt2\sin(n\pi t)\), so it is injective, trace class, and of
infinite rank, and \(\|A\|_\infty\le\sup_y\|g(y,\cdot)\|_2<\infty\).
\Cref{ex:operator-layer} applies, and with \(Q=P=L_D^{-1}\) the input
measure is a Gaussian field with covariance equal to the inverse
Dirichlet operator.  Applying \(L_D^{-1}\) to white noise instead
would give covariance \(L_D^{-2}\).  The linear
operator \(L_D^{-1}\) itself is an exact ReLU network,
\(L_D^{-1}x=\sum_n\lambda_ne_n[\ReLU(\langle e_n,x\rangle)-\ReLU(-\langle e_n,x\rangle)]\),
whose \(2N\)-neuron truncation has error at most
\(\lambda_{N+1}\|x\|=O(N^{-2})\) uniformly on bounded sets.
This is the linear spectral-truncation estimate.  For the nonlinear
layer, \cref{eq:operator-layer-uniform} supplies a separate uniform
approximation estimate for the full output.
\end{ex}

\begin{proof}
Direct differentiation shows that \(g(y,\cdot)\) solves the boundary
value problem with a unit jump of the derivative at \(t=y\), so
\(L_D^{-1}\) is the integral operator with kernel \(g\); the eigenpairs
are checked by \(-e_n''+e_n=(1+\pi^2n^2)e_n\).  The kernel is continuous
and bounded by \(\sinh^2(1/2)/\sinh1\), so \(y\mapsto g(y,\cdot)\) is
continuous and bounded into \(L^2(0,1)\), and \(m=\dd y\) is finite; hence
\cref{ex:operator-layer} applies.  The ReLU expansion is the spectral
theorem with \(u=\ReLU(u)-\ReLU(-u)\), and orthogonality gives
\(\|L_D^{-1}x-(L_D^{-1})_Nx\|^2=\sum_{n>N}\lambda_n^2\langle e_n,x\rangle^2
\le\lambda_{N+1}^2\|x\|^2\).
\end{proof}

\section{Numerical experiments}
\label{sec:numerics}

We compare the observed finite-width approximation errors with the
predicted rates and check an explicit transform formula numerically.
Varying the input resolution distinguishes approximation in network
width from errors caused by numerical discretization.
The three experiments illustrate \cref{thm:D,thm:E} and
\cref{ex:operator-layer} on \(\ell^2\) or \(L^2(0,1)\).
For computation the input is truncated to \(d\) coordinates, with several
values of \(d\).  These finite computations do not establish a
uniform-in-dimension theorem.  Errors are suprema on finite test sets,
averaged over independent trials; fitted slopes use logarithmic
coordinates and widths \(N=4,8,\dots,4096\).
\Cref{app:numerics} gives the parameter-drawing procedures, quadrature,
and reproduction details.

\paragraph{Experiment 1: ReLU network with Gaussian random directions on \(\ell^2\).}
This is \cref{cor:relu-discretization}.  Let \(Q=\operatorname{diag}(j^{-2})\)
on \(\ell^2\), truncated to \(d\in\{10,10^2,10^3,10^4\}\) coordinates,
and let \(K\) consist of \(200\) points with coordinates
\(x_j=u_j/j\), \(u_j\) uniform in \([-1,1]\).  The target
\(f_{\mathrm{ReLU},Q}(x)=\sqrt{\langle Qx,x\rangle/2\pi}\) is compared with
\(f_{\mathrm{ReLU},Q,N}(x)=N^{-1}\sum_{j\le N}\ReLU(\langle a_j,x\rangle)\),
\(a_j\sim\mu_{Q}\), over \(20\) trials.  \Cref{tab:numerics} and
\cref{fig:numerics}(upper left) show that the mean supremum error decays like
\(N^{-1/2}\) and is similar across the tested \(d\); the bound of \cref{cor:relu-discretization}, which equals \(0.24\) at \(N=4096\), is
about twenty times larger than the observed error.  Thus the experiment
checks the rate rather than sharpness of the bound.

\paragraph{Experiment 2: the ridgelet coefficient of \(f_W\) synthesizes \(C\,T_\alpha f_W\).}
This is \cref{ex:closed-form}(iii) with \(\alpha=1\), \(P=Q=\operatorname{diag}(j^{-2})\),
\(W=I\), and the analysis filter \(\rho(t)=\mathrm{He}_4(t)e^{-t^2/2}\),
whose Fourier transform is \(\sqrt{2\pi}\,\omega^4e^{-\omega^2/2}\).
This filter has four vanishing moments and admissibility constant
\(\iiprod{\rho,\rho}_{1}=6\); it is not band pass, but \cref{thm:A}(ii) uses only
admissibility, and the finiteness of the total variation and of the
second moment of \(\gamma_g=R_\rho f_W\) is verified directly from
\(\int_\R|\gamma_g(a,b)|\dd b=D^{-1/2}\sqrt{2\pi}\,m_4\,(1+\kappa_W(a))^{-2}\),
\(m_4=\EE|\mathrm{He}_4(Z)|\), and \cref{lem:gaussian-decay}.  The
coefficient is explicit, \(\gamma_g(a,b)=D^{-1/2}\sqrt{2\pi}\,\phi_u^{(4)}(b)\)
with \(u=1+\kappa_W(a)\), and the target \(\iiprod{\rho,\rho}_{1}f_g\) is computed
from \cref{eq:filtered-gaussian-target} by quadrature.  Parameters are
drawn by rejection sampling for the bias and importance
resampling for the direction; these numerical procedures approximate the
coefficient law \(|\gamma_g|\lambda_1/V\), as detailed in
\cref{app:numerics}.  Over \(10\) trials, with
\(d\in\{10,10^2,10^3\}\) and \(K\) consisting of \(100\) points drawn
by the rule of Experiment~1, the mean supremum error of the sampled network
\(f_N=\frac VN\sum_j\operatorname{sgn}\gamma_g(\theta_j)\rho(\langle a_j,\cdot\rangle-b_j)\)
against \(6\,f_g\), whose supremum on \(K\) is about \(12.5\), decays
approximately like \(N^{-1/2}\) across the tested dimensions (\cref{tab:numerics},
\cref{fig:numerics}(upper right)).

\paragraph{Experiment 3: the Dirichlet operator layer on \(L^2(0,1)\).}
This is \cref{ex:dirichlet} with the Gaussian activation \(\sigma_{\mathrm{Gauss}}\) and the
observable \(\varphi\equiv1\), for which
\(w(y)=\int_0^1g(y,t)\dd t=1-\cosh(y-\tfrac12)/\cosh\tfrac12>0\) and
\(V=\|w\|_{L^1}=1-2\tanh\tfrac12\approx0.0758\).  The interval is
discretized by \(d\in\{64,512,4096\}\) midpoints, and the test set
consists of \(100\) inputs \(x=\sum_{n\le30}(u_n/n)e_n\) with \(u_n\)
uniform in \([-1,1]\).  (a) The sampled network
\(f_N(x)=\frac VN\sum_j\sigma_{\mathrm{Gauss}}(\langle g(y_j,\cdot),x\rangle)\), \(y_j\sim w/V\),
is compared with \(F_1(x)=\int_0^1w(y)\sigma_{\mathrm{Gauss}}((L_D^{-1}x)(y))\dd y\) over
\(20\) trials; the error decays like \(N^{-1/2}\) for \(d=512\) and
\(d=4096\), while for \(d=64\) it levels off near \(10^{-5}\), which is
the quadrature error of the \(64\)-point reference value.  (b) The
closed-form transform \cref{eq:operator-layer-transform}, with the
filter of Experiment~2, is compared at \(a=\sqrt2\sin(\pi t)+\tfrac12\sqrt2\sin(3\pi t)\)
and \(b\in\{0,-\tfrac12,-1\}\) with a Monte Carlo estimate of
\(\int F_1(x)\rho(\langle a,x\rangle-b)\,\mu_{L_D^{-1}}(\dd x)\) using
\(2\times10^5\) Gaussian samples and \(d'\in\{8,64,512\}\) modes
of the Karhunen--Lo\`eve expansion; the discretization is specified in
\cref{app:numerics}.
The closed-form values are \((0.18119,\ 0.09062,\ -0.06298)\); the Monte
Carlo estimates are \((0.18111,\ 0.09107,\ -0.06258)\) for \(d'=8\),
\((0.18106,\ 0.09041,\ -0.06303)\) for \(d'=64\), and
\((0.18077,\ 0.09072,\ -0.06258)\) for \(d'=512\), with standard errors
\((0.00013,\ 0.00023,\ 0.00018)\): all agree within about three
standard errors.

\begin{table}[ht]
\centering
\caption{Mean supremum error on \(K\) of the sampled networks at
\(N=4\) and \(N=4096\), and the fitted slope of the error against \(N\)
in logarithmic coordinates, for each truncation dimension \(d\).}
\label{tab:numerics}
\begin{tabular}{@{}lrrrr@{}}
\toprule
Experiment & \(d\) & error at \(N=4\) & error at \(N=4096\) & slope\\
\midrule
1: ReLU, Gaussian directions on \(\ell^2\) & \(10\) & \(0.41\) & \(0.011\) & \(-0.53\)\\
 & \(10^2\) & \(0.37\) & \(0.012\) & \(-0.50\)\\
 & \(10^3\) & \(0.42\) & \(0.011\) & \(-0.51\)\\
 & \(10^4\) & \(0.32\) & \(0.010\) & \(-0.50\)\\
\midrule
2: ridgelet coefficient of \(f_W\), target \(6\,f_g\) & \(10\) & \(11.9\) & \(0.46\) & \(-0.47\)\\
 & \(10^2\) & \(11.6\) & \(0.43\) & \(-0.48\)\\
 & \(10^3\) & \(12.8\) & \(0.36\) & \(-0.52\)\\
\midrule
3: Dirichlet operator layer on \(L^2(0,1)\) & \(64\) & \(1.0\times10^{-4}\) & \(1.2\times10^{-5}\) & \(-0.31\)\\
 & \(512\) & \(1.4\times10^{-4}\) & \(3.6\times10^{-6}\) & \(-0.53\)\\
 & \(4096\) & \(8.9\times10^{-5}\) & \(2.9\times10^{-6}\) & \(-0.50\)\\
\bottomrule
\end{tabular}
\end{table}

\begin{figure}[ht]
\centering
\includegraphics[width=0.49\textwidth]{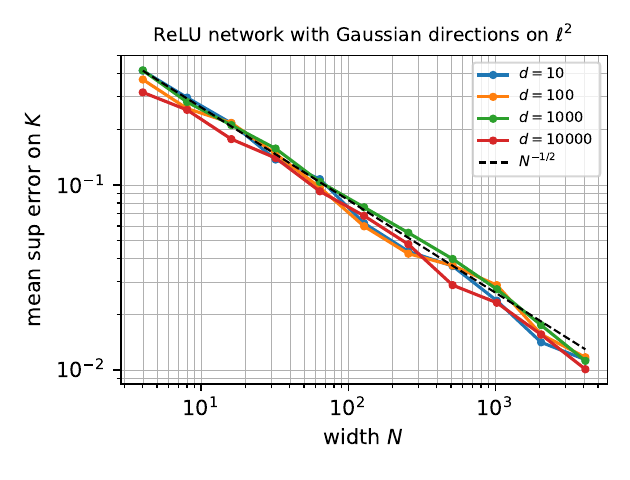}\hfill
\includegraphics[width=0.49\textwidth]{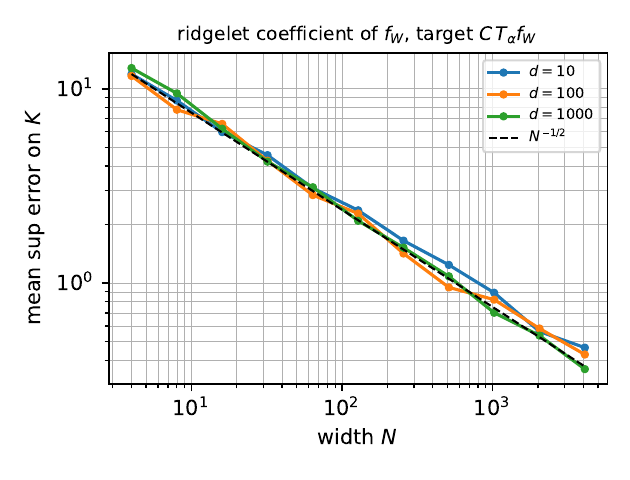}\par\smallskip
\includegraphics[width=0.49\textwidth]{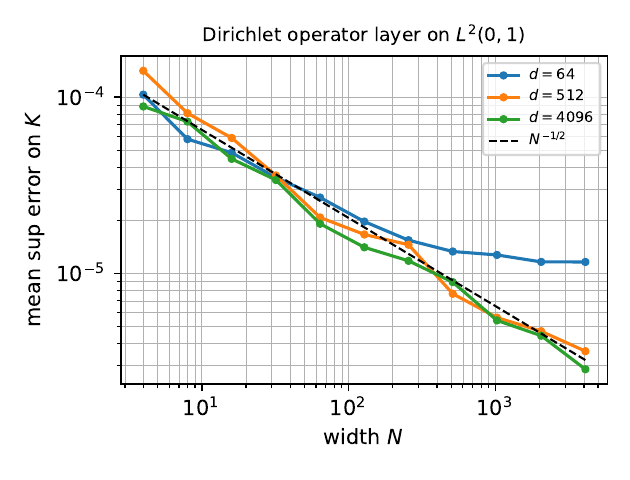}
\caption{Mean supremum error on \(K\) against the width \(N\) for
Experiments 1 (upper left), 2 (upper right), and 3(a) (bottom), for several
truncation dimensions \(d\); the dashed line is proportional to
\(N^{-1/2}\).}
\label{fig:numerics}
\end{figure}

The experiments illustrate the approximation estimates rather than the
computational cost, which grows with \(d\) through the inner products.
Across the tested resolutions, the errors and slopes are similar once
the discretization error is small enough.  Finite test sets, quadrature,
and importance resampling limit what these computations alone establish.
\FloatBarrier

 \section{Discussion}
\label{sec:discussion}

We discuss what the analysis, representation, and approximation results
establish and identify the remaining limitations.  These include the
preliminary approximation of a general continuous target, the numerical
representation of directions, and extensions of the measure construction.

\paragraph{Main results.}
With a Gaussian input weight and a homogeneous Gaussian mixture on the
direction parameter, the ridgelet transform on an infinite-dimensional
Hilbert space keeps the structure of the finite-dimensional theory.  The
admissibility constant is a one-dimensional integral in which no dimension
appears, the transform is injective with a Plancherel identity on
\(\cE_\alpha\), and synthesis returns the target up to the frame operator
\(T_\alpha\), the correction induced by the weights, which is undone after
synthesis, before analysis, or on the coefficients.  Targets with a spectral
density are integral networks with explicit coefficients; every continuous
target is a compact-open limit of such networks; and sampling the
coefficient gives finite-width networks at the rate \(N^{-1/2}\) when the
density satisfies \cref{def:ray-regular} and the activation is globally
Lipschitz.  Weak reconstruction covers every non-polynomial tempered
activation, and ReLU, \(\tanh\), the Gaussian distribution function, and the
Gaussian activation satisfy all of these conditions.  The neural-operator
layers of the examples are integral networks with explicit parameter
measures to which the finite-width bounds apply.

\paragraph{Approximation rates and numerical implementation.}
The Sobolev criterion of \cref{thm:weak-sobolev-synthesis}
permits non-band-pass filters and a finite range of coefficient moments.
For \(s>5/2\), it connects to the Hilbert-valued uniform approximation
theorem just as the smooth criterion does.  The experiments support
the predicted approximation rates on finite test sets and check an
explicit transform independently; they also expose a discretization
floor.  They do not remove the need to account for numerical integration
and finite representations of directions in an implementation.

\paragraph{Position among related results.}
The universality statement is of a type already available as an existence
theorem: compact-open density of networks with scalar ridge neurons holds
on locally convex spaces for every continuous non-polynomial activation
\citep{Ismailov2026AMAI}, and global density on weighted spaces is proved in
\citet{CuchieroSchmockerTeichmann2026}.  What the ridgelet transform adds is
the construction of the approximants: an explicit coefficient for every
target with a spectral density and every admissible analysis filter, a
Plancherel identity with an identified frame operator, and a sampling
argument that yields the finite-width rate and applies to the examples.  The
rate is a Barron-type statement for a restricted class, consistent with the
lower bounds for Lipschitz operator classes
\citep{LanthalerStuart2025,Lanthaler2024entropy,AdcockMaierParhi2025}.  The
transform differs from the ridgelet transform for Wiener functionals of
\citet{Martias2012ridgelet} in its atoms, its parameter measure, and its
Plancherel identity on a single space; \cref{sec:extended-related} records
these comparisons.

\paragraph{Open problems.}
Natural next steps are quantitative control of the spectral approximation
for regular target classes, the cost of representing infinite-dimensional
directions in a finite computation, and extensions to operator activations
and to input measures beyond the Gaussian construction.

\appendix
\crefalias{section}{appendix}
\crefalias{subsection}{appendix}
\section{A kernel interpretation of the frame operator}
This appendix gives a regularized characteristic functional and an optional
kernel interpretation. Neither calculation is a prerequisite of the main proofs.
\subsection{The regularized characteristic functional of the mixture}
\label{app:mixture-character}

The following calculation is used only in \cref{rmk:riesz-kernel}.

\begin{lem}[Truncation of the scale integral]
\label{lem:mixture-character}
For \(z\ne0\) and \(\alpha>0\), put \(q=\langle Pz,z\rangle>0\) and
\(\nu_\alpha^{\eps,M}:=\int_\eps^M\mu_{2sP}s^{\alpha/2-1}\dd s\) for
\(0<\eps<M<\infty\).  Then
\[
  \lim_{\eps\downarrow0,\,M\uparrow\infty}
  \int_\cH e^{i\langle z,\xi\rangle}\,
         \nu_\alpha^{\eps,M}(\dd\xi)
  =\int_0^\infty e^{-sq}s^{\alpha/2-1}\,\dd s
  =\Gamma(\alpha/2)\,q^{-\alpha/2}.
\]
In contrast, \(\int_\cH|e^{i\langle z,\xi\rangle}|\nu_\alpha(\dd\xi)=\infty\),
so the limit is not a Lebesgue integral against \(\nu_\alpha\).
\end{lem}

\begin{proof}
The truncated mixture is finite, so Fubini and the characteristic
functional of \(\mu_{2sP}\) give
\(\int_\cH e^{i\langle z,\xi\rangle}\nu_\alpha^{\eps,M}(\dd\xi)
=\int_\eps^Me^{-sq}s^{\alpha/2-1}\dd s\).  Monotone convergence and the
substitution \(u=sq\) finish the proof.
\end{proof}

\begin{rmk}[A Riesz-type kernel for the frame operator]
\label{rmk:riesz-kernel}
The regularized characteristic functional of \(\nu_\alpha\) is
\(\Gamma(\alpha/2)\langle Pz,z\rangle^{-\alpha/2}\)
(\cref{lem:mixture-character}).  Interchanging integrals formally in
(iii) gives
\[
  T_\alpha f(x)
  \;``=\text{''}\;
  \Gamma(\alpha/2)\int_\cH f(y)\,
  \langle P(x-y),x-y\rangle^{-\alpha/2}\,\mu_Q(\dd y),
\]
a formal integral operator with the Riesz-type kernel
\(\langle P(x-y),x-y\rangle^{-\alpha/2}\).
On \(\R^m\) with \(P=I\) and \(0<\alpha<m\), the kernel gives a
Riesz potential of order \(m-\alpha\), namely
\(k_{m,\alpha}(-\Delta)^{-(m-\alpha)/2}(fp_Q)\), in agreement with
\cref{cor:finite-backprojection}.  The interchange is not justified for an
infinite measure and an oscillatory integrand, and the identity is not
used.
\end{rmk}

\section{Coefficient projections and minimum-norm solutions}
\label{app:reconstruction-proofs}
This appendix derives additional range and minimum-norm formulas from the
coefficient adjoint and the reconstruction theorem proved in the main text.
\subsection{Coefficient projection}
\label{app:projection}

Let \(\cY=L^2(\cH\times\R,\lambda_\alpha)\), normed through the partial
Fourier transform in the bias,
\(\|\gamma\|_{\cY}^2=\frac1{2\pi}\int|\gamma^\sharp(a,\omega)|^2\nu_\alpha(\dd a)\dd\omega\),
and let \(\rho\) be \(\alpha\)-admissible with \(C=\iiprod{\rho,\rho}_{\alpha}\).

\begin{prop}[Bounded backprojection and orthogonal range projection]
\label{prop:coefficient-projection}
The Hilbert adjoint \(W_\rho^*\) of
\(W_\rho:L^2(\nu_\alpha)\to\cY\), called backprojection, is given by
\cref{eq:ray-average}: the integral there
converges absolutely for \(\nu_\alpha\)-almost every \(\xi\), is
independent, as an \(L^2\) class, of the jointly measurable Fourier
representative of \(\gamma\), and satisfies
\(\|W_\rho^*\gamma\|_{L^2(\nu_\alpha)}\le\sqrt C\|\gamma\|_{\cY}\) and
\(W_\rho^*W_\rho=C\Id\).  With \(P_{\cK_\alpha}\) the orthogonal
projection onto \cref{eq:spectral-range},
\begin{equation}
 \Pi_\rho=C^{-1}W_\rho P_{\cK_\alpha}W_\rho^*
 \label{eq:coefficient-projection}
\end{equation}
is the orthogonal projection onto \(\Ran R_\rho\).  The minimum-norm
solution of \(S_\rho\gamma=F\in\cE_\alpha'\) is
\(C^{-1}R_\rho T_\alpha^{-1}F\), and all solutions differ from it by an
element of \((\Ran R_\rho)^\perp\).
\end{prop}

\begin{proof}
The adjoint formula, representative independence and norm bound
are proved in \cref{lem:coefficient-adjoint}.

Since \(R_\rho=W_\rho F_Q\) and \(F_Q\) is unitary onto
\(\cK_\alpha\), the operator \(C^{-1/2}W_\rho|_{\cK_\alpha}\) is an
isometry with closed image; its product with its adjoint is
\cref{eq:coefficient-projection}.  Finally
\(R_\rho'R_\rho=CT_\alpha\) by \cref{eq:frame-identity}, and the kernel of
the transpose \(R_\rho'\) is \((\Ran R_\rho)^\perp\); orthogonality proves
minimum norm and the description of the solution set.
\end{proof}

The projection formula uses the closed spectral subspace
\(\cK_\alpha\) explicitly.  The factor \(P_{\cK_\alpha}\) distinguishes
projection onto the analysis range from projection onto
\(W_\rho L^2(\nu_\alpha)\).

\begin{rmk}[Range of the analysis operator]
\label{rmk:coefficient-range}
By \cref{thm:B}, \(\Ran R_\rho=W_\rho(\cK_\alpha)\subset W_\rho(L^2(\nu_\alpha))\),
and both are closed.  The projection onto \(W_\rho L^2(\nu_\alpha)\) is
\(C^{-1}W_\rho W_\rho^*\).  The projection onto \(\Ran R_\rho\) is
\(C^{-1}W_\rho P_{\cK_\alpha}W_\rho^*\), the minimum-norm solution of
\(S_\rho\gamma=F\) is \(C^{-1}R_\rho T_\alpha^{-1}F\), and all solutions
differ by elements of \((\Ran R_\rho)^\perp\)
(\cref{prop:coefficient-projection}).  These formulas retain the closed spectral subspace explicitly and
separate the reconstruction of analyzed targets from arbitrary
coefficient synthesis.
\end{rmk}

\section{Additional approximation bounds}
\label{app:sampling}
\label{app:supplementary-sampling}
This appendix gives concentration for bounded parameters and separates input
truncation from finite-width approximation. These are consequences of the
main approximation results, rather than ingredients of their proofs.
Use the polar decomposition \cref{eq:polar-decomposition} with \(V>0\);
the zero measure is represented by the zero network.
\begin{cor}[Concentration for bounded parameters]
\label{cor:sampling-concentration}
Under \cref{thm:lipschitz-barron}, suppose \(\|a\|^2+|b|^2\le B^2\)
almost surely for some \(B\ge0\), and put \(M_K=|\sigma(0)|+\Lip(\sigma)R_KB\).  With
probability at least \(1-\delta\),
\[
 \|f_N-f\|_{C(K)}
 \le\frac{8V}{\sqrt N}(|\sigma(0)|+\Lip(\sigma)R_KM_2)
       +VM_K\sqrt{\frac{2\log(1/\delta)}N}.
\]
\end{cor}

\begin{proof}
Each atom has norm at most \(M_K\), so replacing one sample changes
\(\|f_N-f\|_{C(K)}\) by at most \(2VM_K/N\).  The bounded-difference
inequality bounds the excess over the expectation, and
\cref{thm:lipschitz-barron} bounds the expectation.
\end{proof}

\begin{cor}[Input truncation and finite-width approximation errors]
\label{cor:two-stage-error}
Let \(\Pi_m\) be finite-rank orthogonal projections converging strongly
to the identity.  For \(f\in C(\cH)\) and compact \(K\),
\(\|f-f\circ\Pi_m\|_{C(K)}\to0\).  If \(f=S_\sigma\Gamma\) satisfies the
hypotheses of \cref{thm:lipschitz-barron} and the same samples are used
with directions \(\Pi_ma_j\), then
\[
 \EE\|f-f_{m,N}\|_{C(K)}
 \le\Lip(\sigma)\Bigl(\int\|a\|\dd|\Gamma|\Bigr)\sup_K\|x-\Pi_mx\|
 +\frac{8V}{\sqrt N}(|\sigma(0)|+\Lip(\sigma)R_KM_2).
\]
\end{cor}

\begin{proof}
A finite-net argument gives \(\sup_K\|x-\Pi_mx\|\to0\).  If
\(f\circ\Pi_m\) did not converge uniformly on \(K\), a sequence
\(x_m\in K\) with \(|f(\Pi_mx_m)-f(x_m)|\ge\eps\) would have a
convergent subsequence \(x_m\to x\), along which \(\Pi_mx_m\to x\) as
well, contradicting continuity of \(f\) at \(x\).
The triangle inequality separates truncation from sampling; in the
integral representation the Lipschitz bound gives the truncation term,
projecting directions does not increase their second moment, and the
sampling proof applies on the original parameter probability space.
\end{proof}

\section{The finite-dimensional case and the dilation obstruction}
\label{app:finite-dim}

This appendix compares the weighted construction with finite-dimensional
ridgelet analysis.  It also explains why a construction based on dilating
a Gaussian spectral measure cannot produce a finite coefficient measure
in infinite dimension.

\subsection{A Fourier calculation for comparison}
\label{subsec:finite-ridgelet}

For comparison only, take Lebesgue input and direction measures on
\(\R^m\), and write
\begin{equation}
 R^{\rm Leb}_\rho f(a,b)=\int_{\R^m}f(x)\rho(\langle a,x\rangle-b)\,\dd x.
 \label{eq:finite-ridgelet}
\end{equation}
For Schwartz \(f,\rho\), Fubini and a translation in the bias give
\begin{equation}
 (R^{\rm Leb}_\rho f)^\sharp(a,\omega)
 =\rho^\sharp(-\omega)\widehat f(\omega a).
 \label{eq:finite-slice}
\end{equation}
If \(\sigma\) is Schwartz and \(\rho\) is band pass, Parseval in
the bias followed by \(\xi=\omega a\) gives the iterated synthesis
identity
\begin{equation}
 S^{\rm Leb}_\sigma R^{\rm Leb}_\rho f=K_{\sigma,\rho}f,
 \quad K_{\sigma,\rho}=(2\pi)^{m-1}\int_\R
 \sigma^\sharp(\omega)\rho^\sharp(-\omega)|\omega|^{-m}\,\dd\omega.
 \label{eq:finite-reconstruction}
\end{equation}
The transformed integrals are absolute: the frequency factor is supported
away from zero and \(\widehat f\) is integrable.  This calculation explains
the scaling factor, but the Hilbert-space construction in the main text
is defined independently of it.

\subsection{The finite-dimensional reference measure}
\label{subsec:finite-reference-measure}

On \(\cH=\R^m\) with \(P=I\) and \(0<\alpha<m\), Tonelli and the Gaussian
density give, for \(a\ne0\),
\[
 \begin{aligned}
  \frac{\dd\nu_\alpha}{\dd a}(a)
  &=(4\pi)^{-m/2}\int_0^\infty s^{(\alpha-m)/2-1}e^{-\|a\|^2/(4s)}\,\dd s\\
  &=(4\pi)^{-m/2}\Bigl(\frac{\|a\|^2}4\Bigr)^{(\alpha-m)/2}
   \Gamma\Bigl(\frac{m-\alpha}2\Bigr),
 \end{aligned}
\]
by the substitution \(u=\|a\|^2/(4s)\).  Hence
\begin{equation}
  \nu_\alpha(\dd a)=c_{m,\alpha}\|a\|^{\alpha-m}\,\dd a,
  \qquad
  c_{m,\alpha}=2^{-\alpha}\pi^{-m/2}\Gamma\Bigl(\frac{m-\alpha}2\Bigr).
  \label{eq:finite-density}
\end{equation}
The density is locally integrable at zero because its radial integral
is proportional to \(\int_0^rt^{\alpha-1}\dd t\).  The restriction
\(\alpha<m\) is needed for convergence at \(s=\infty\); at
\(\alpha=m\) the mixture diverges logarithmically, and Lebesgue measure,
the classical choice of direction measure, must be chosen separately
rather than obtained by substituting \(\alpha=m\).  In infinite dimension
no such restriction arises, which is the content of
\cref{lem:homogeneous-mixture}.

\begin{rmk}[Comparison with familiar spaces]
\label{rmk:sobolev-comparison}
In finite dimension, \cref{eq:finite-density} gives
\[
\|f\|_{\cE_\alpha}^2=c_{m,\alpha}\int_{\R^m}|\widehat{fp_Q}(\xi)|^2
|\xi|^{\alpha-m}\dd\xi,\qquad 0<\alpha<m,
\]
where \(p_Q\) is the
Gaussian density: a homogeneous Sobolev norm of order \((\alpha-m)/2\) of
the weighted density \(fp_Q\).  In infinite dimension there is no
Lebesgue measure to remove the weight, and we do not identify
\(\cE_\alpha\) with a Sobolev space.  Likewise \(\cK_\alpha\) is a closed
subspace of \(L^2(\nu_\alpha)\) but is not asserted to be a reproducing
kernel Hilbert space.
\end{rmk}

\subsection{Filtered backprojection}

\begin{cor}[The frame operator in finite dimension]
\label{cor:finite-backprojection}
Let \(\cH=\R^m\), \(0<\alpha<m\), \(\nu_\alpha\) as in
\cref{eq:finite-density}, and let \(p\) be a nondegenerate Gaussian
density.  Assume \(f\in L^2(p\dd x)\) with \(g:=fp\in\cS(\R^m)\).  Then
\(f\in\cD_\alpha\), and the frame operator, written as a distribution
against the pivot measure \(p\dd x\), has the representative
\[
 t_f=\int_{\R^m}e^{i\langle x,\xi\rangle}\widehat g(\xi)\,\nu_\alpha(\dd\xi)
 =k_{m,\alpha}(-\Delta)^{-(m-\alpha)/2}g,
 \qquad k_{m,\alpha}=(2\pi)^mc_{m,\alpha}.
\]
For a band-pass \(\rho\), the synthesis \(S_\rho R_\rho f\) is represented
by \(\iiprod{\rho,\rho}_{\alpha}t_f\), and, distributionally,
\[
 f=\frac{p^{-1}}{k_{m,\alpha}\iiprod{\rho,\rho}_{\alpha}}
             (-\Delta)^{(m-\alpha)/2}S_\rho R_\rho f .
\]
If instead the direction measure is Lebesgue measure and \(\alpha=m\),
the multiplier is one and \(k=(2\pi)^m\), so that
\(f=(2\pi)^{-m}(\iiprod{\rho,\rho}_{m})^{-1}p^{-1}S_\rho R_\rho f\).
\end{cor}

\begin{proof}
Here \(F_Qf=\widehat g\).  The weight \(\|\xi\|^{\alpha-m}\) is locally
integrable and \(\widehat g\) is Schwartz, so \(\widehat g\in L^1\cap L^2(\nu_\alpha)\)
and Fourier inversion gives
\(\widehat{t_f}=(2\pi)^mc_{m,\alpha}\|\xi\|^{\alpha-m}\widehat g\).  For
smooth compactly supported \(h\), Fubini identifies \(\int t_f\overline h\,p\dd x\)
with \(\langle f,h\rangle_{\cE_\alpha}\), which is \cref{thm:C}(iii).
The multiplier \(\|\xi\|^{m-\alpha}\) cancels the locally integrable
multiplier, proving the formula in distributions.  For Lebesgue
direction measure, direct Fourier inversion gives \(t_f=(2\pi)^mfp\).
\end{proof}

The two weights are visible in the formula: division by \(p\) removes the
input weight, and \((-\Delta)^{(m-\alpha)/2}\) removes the homogeneous
direction weight.  With Lebesgue measure on both sides the filter is a
constant, which is the classical statement \cref{eq:finite-reconstruction}.

\subsection{Why the spectral measure cannot be dilated}
\label{subsec:dilation-obstruction}

In finite dimension, the Fourier-slice identity \cref{eq:finite-slice}
says that the bias Fourier slice of the ridgelet coefficient at
frequency \(\omega\) is \(\rho^\sharp(-\omega)\) times the spectral
measure of \(f\) dilated by \(1/\omega\).  One might try to define a
coefficient in infinite dimension by this prescription, starting from a
target with a finite spectral measure.  The following proposition shows
that this is impossible for a Gaussian spectral measure, that is, for the
target \(f_W\) of \cref{ex:closed-form}.

\begin{prop}[Dilation obstruction]
\label{prop:dilation-obstruction}
Let \(\dim \cH=\infty\) and let \(W\) be injective, positive, self-adjoint,
and trace class, with eigenvectors \(e_j\) and eigenvalues \(w_j>0\).
\begin{enumerate}[label=\textup{(\roman*)}]
\item
For \(t>0\), the Borel sets
\(E_t:=\{x:\lim_{n\to\infty}\frac1n\sum_{j=1}^n\langle x,e_j\rangle^2/w_j=t\}\)
are pairwise disjoint and \(\mu_{tW}(E_t)=1\).  Consequently a
sigma-finite measure dominates \(\mu_{tW}\) for at most countably many
\(t\).
\item
Let \(r\) be a bounded Borel function on \(\R\) with
\(\{r\ne0\}\) of positive Lebesgue measure.  There is no finite complex
Borel measure \(\Gamma\) on \(\cH\times\R\) whose bias slices satisfy
\[
 \begin{aligned}
 \Gamma^+_\omega(E)&:=\int_{E\times\R}e^{i\omega b}\,\Gamma(\dd a,\dd b)\\
 &=r(\omega)\,(D_{1/\omega})_\#\mu_{W}(E)
 \end{aligned}
\]
for almost every \(\omega\ne0\), \(E\in\cB(\cH)\).
\end{enumerate}
\end{prop}

\begin{proof}
(i) The strong law of large numbers for the independent normalized
Gaussian coordinates gives \(\mu_{tW}(E_t)=1\), and the sets are
disjoint for distinct \(t\).  If a sigma-finite measure \(\nu\) dominated
\(\mu_{tW}\) for uncountably many \(t\), then \(\nu(E_t)>0\) for
uncountably many disjoint sets; intersecting with an exhaustion by sets
of finite \(\nu\)-measure and using positive rational thresholds shows
that a sigma-finite measure has at most countably many disjoint sets of
positive measure, a contradiction.

(ii) If \(\Gamma\) existed, the finite measure
\(\mu(E):=|\Gamma|(E\times\R)\) would satisfy \(|\Gamma^+_\omega(E)|\le\mu(E)\)
for every \(\omega\), hence would dominate
\((D_{1/\omega})_\#\mu_{W}=\mu_{W/\omega^2}\) for almost every
\(\omega\) with \(r(\omega)\ne0\).  These are uncountably many distinct
values of \(t=\omega^{-2}\), contradicting (i).
\end{proof}

The finite-variation hypothesis is essential: the marginal of a
sigma-finite measure need not be sigma-finite, and the proposition makes
no claim about coefficient objects that are not measures.  The
construction of \cref{sec:transform} avoids the obstruction because
\(\nu_\alpha\) enters as an integrator in the parameter space and is never
required to dominate the dilates of a spectral measure; the targets it
represents exactly are those with a spectral density with respect to
\(\nu_\alpha\) (\cref{thm:A}), and Gaussian targets are reached through
the frame operator (\cref{ex:closed-form}) or by approximation
(\cref{thm:D}).

\section{Numerical methods and additional results}
\label{app:numerics}

This appendix describes the numerical procedures behind \cref{sec:numerics}.
It records how parameters are drawn, how integrals and function-space
inputs are discretized, and how the reported errors are measured.
The figures and principal results are in \cref{fig:numerics,tab:numerics}.

\subsection{Reproduction and error measurement}

The supplementary Python script \texttt{experiments.py} reproduces all
numbers and figures with the fixed seed \(20260908\).  The full run uses
widths \(N=4,8,\ldots,4096\), and the same sequence of sampled neurons is
accumulated as the width increases within each trial.  The test set is
fixed within each resolution and experiment.  At each width we record
the largest absolute error on that finite set and average over trials;
the slopes in \cref{tab:numerics} are least-squares slopes of
\(\log(\text{mean error})\) against \(\log N\).
Experiments 1 and 3(a) use \(20\) trials, and Experiment 2 uses \(10\).
The errors measure finite test sets rather than the supremum over the
underlying infinite compact family of inputs.

\subsection{Gaussian directions and coefficient measures}

For Experiment 1, independent coordinates of a direction have the form
\(a_j=Z_j/j\), with standard normal \(Z_j\), for \(1\leq j\leq d\).
The \(200\) test inputs have coordinates \(x_j=u_j/j\), with independent
\(u_j\) uniform in \([-1,1]\).  The target is evaluated from its exact
finite-dimensional formula
\(\sqrt{\sum_{j=1}^d j^{-2}x_j^2/(2\pi)}\).

For Experiment 2, the same rule draws \(100\) test inputs.
The matrices are diagonal, with \(p_j=q_j=j^{-2}\), \(w_j=1\), and
\((S_W)_{jj}=q_j/(1+q_j)\); determinants are evaluated by summing logarithms.
The target \(6f_g\) is evaluated by the one-dimensional quadrature in
\cref{eq:filtered-gaussian-target}.  In the mixture integral
\(\int_0^\infty h(s)s^{-1/2}\,\dd s\), the substitution \(s=r^2\)
on \((0,1)\) removes the endpoint singularity, and the remaining
integral is evaluated on \((1,\infty)\).

Parameters are
drawn from \(|\gamma_g|\lambda_1/V\) as follows: the bias given \(a\) is
sampled exactly by rejection from a mixture of Gaussian moment
densities; the direction is sampled by sampling--importance resampling
from a pool of \(2\times10^5\) proposals \((s,a)\) with
\(s\sim\theta_0\,\mathrm{BetaPrime}(\tfrac12,\tfrac32)\),
\(a\sim\mu_{2sP}\), and weight \((1+\kappa_W(a))^{-2}(1+s/\theta_0)^2\),
where \(\theta_0=(2\tr PS_W)^{-1}\) (effective sample sizes between
\(2.8\times10^3\) and \(9.7\times10^3\)).  The total variation \(V\) is
computed by quadrature from \cref{eq:gaussian-layer-integral}.

The conditional bias law can be made explicit.  Writing
\(u=1+\kappa_W(a)\) and \(b=\sqrt u\,z\), the density of \(z\) is
\(|\mathrm{He}_4(z)|\phi_1(z)/m_4\).  Rejection sampling uses the proposal
\[
 \tfrac14\phi_1(z)+\tfrac12z^2\phi_1(z)
      +\tfrac1{12}z^4\phi_1(z)
 =\frac{z^4+6z^2+3}{12}\phi_1(z)
\]
and accepts with probability
\(|\mathrm{He}_4(z)|/(z^4+6z^2+3)\).
This gives the exact conditional law.  Directions are drawn with
replacement from the weighted proposal pool, so this part approximates
the desired continuous direction law.  The pool is reused across trials
at each resolution; the reported trial variation is conditional on that
pool and does not measure its approximation error.
The code uses the equivalent bias convention \(\langle a,x\rangle+c\)
with \(c=-b\); the even bias density and its even sign factor give the
same synthesis integral.

The variation integral is evaluated using
\[
 \int (1+\kappa_W(a))^{-2}\,\nu_1(\dd a)
 =\int_0^\infty\!\int_0^\infty
 t e^{-t}\prod_{j=1}^d(1+4ts p_j(S_W)_{jj})^{-1/2}
 \,\dd t\,s^{-1/2}\,\dd s.
\]
Multiplication by \(D^{-1/2}\sqrt{2\pi}\,m_4\) gives \(V\).
The value \(m_4=\mathbb E|\mathrm{He}_4(Z)|\) is computed by quadrature,
split at the four real roots of \(\mathrm{He}_4\); the numerical integral
is restricted to \([-12,12]\).

\subsection{Dirichlet layer and transform verification}

For Experiment 3(a), the interval grid is
\(t_i=(i-1/2)/d\), \(1\leq i\leq d\), and inner products and
integrals use the midpoint rule with weight \(1/d\).
Each of the \(100\) test inputs is a sum
\(x=\sum_{n=1}^{30}(u_n/n)e_n\), with independent uniform
\(u_n\in[-1,1]\) and \(e_n(t)=\sqrt2\sin(n\pi t)\).
The reference value of \(F_1\) uses the Green matrix
\(g(t_i,t_j)\) on the same grid.  Draws from \(w/V\) use inversion of a
cumulative distribution computed by the trapezoidal rule on \(20001\)
uniform grid points, with linear interpolation.  The direction at a drawn
point \(y\) is the Green function \(g(y,\cdot)\), evaluated on the
input grid.  The plateau at \(d=64\) therefore includes the error in
the quadrature reference, rather than measuring a failure of the
finite-width approximation rate.

For Experiment 3(b), the closed-form transform uses \(2048\) midpoints
for inner products and \(400\) midpoints for the outer integral in \(y\).
The direct Monte Carlo calculation uses \(2\times10^5\) independent draws
\[
 x=\sum_{n=1}^{d'}\lambda_n^{1/2}\zeta_ne_n,
 \qquad \lambda_n=(1+\pi^2n^2)^{-1},\qquad
 d'\in\{8,64,512\},
\]
with independent standard normal \(\zeta_n\).
For each draw, \(F_1(x)\) and \(\langle a,x\rangle\) are evaluated
on the \(2048\)-point grid.  The reported standard error is the
empirical standard deviation of the integrand divided by the square
root of the number of draws.  Agreement with the closed-form values
is a check at the specified direction and biases; it does not eliminate
quadrature error or the truncation of the Gaussian input expansion.

\section{Extended overview of related work}
\label{sec:extended-related}

This appendix expands \cref{subsec:related}.  It
places the Gaussian-weighted ridgelet transform among four bodies of work:
harmonic analysis on infinite-dimensional spaces, universal approximation by
networks with infinite-dimensional inputs, approximation and complexity
theory for neural operators, and the analysis of Gaussian measures and
functional data.  The contributions of this paper are the explicit
coefficient \(W_\rho G\) for targets with a \(\nu_\alpha\)-spectral density,
the Plancherel identity and the frame operator \(T_\alpha\) on the Hilbert
space \(\cE_\alpha\), weak reconstruction for tempered non-polynomial
activations, the dimension-free \(N^{-1/2}\) rate for Lipschitz activations,
and the three infinite-dimensional examples.  We do not claim priority for
integrating over an infinite-dimensional space with a Gaussian measure, for
transforms of ridgelet or wavelet type on Wiener space, or for universal
approximation with infinite-dimensional inputs; each has a literature
summarized below.

\subsection{Ridgelet transforms and integral representations}
\label{subsec:ext-ridgelet}

The ridgelet transform of \citet{Murata1996}, \citet{Candes.PhD}, and
\citet{Candes1999a} represents a function on \(\R^m\) as a superposition of
ridge functions \(x\mapsto\sigma(\langle a,x\rangle-b)\), and it is the
continuous counterpart of a depth-two network with a coefficient measure on
the parameter space.  Its distributional extension admits unbounded
activations such as ReLU \citep{Sonoda2017acha}; weighted Sobolev activation
spaces and the geometry of the synthesis null space are developed in
\citet{Sonoda2021ghost}, where the infinite-dimensional null space refers to
the coefficient space over a finite-dimensional input; the Fourier-slice
method yields transforms for group convolutional networks and non-Euclidean
inputs \citep{Sonoda2022gconv,Sonoda2024fourier}; and complexity bounds
derived from integral representations are surveyed in \citet{Kurkova2012}.
The Hilbert-space construction in \citet[Section~4]{Sonoda2024fourier}
integrates over a finite-dimensional subspace with Lebesgue measure.  In
this paper, we replace that integral by the Gaussian input measure
\(\mu_Q\) and the homogeneous direction measure \(\nu_\alpha\) on the
whole Hilbert space, and we use the transform as a tool for networks: it
yields the coefficients, the universality theorem, the Barron-type rate, and
the examples.

\subsection{Harmonic analysis on infinite-dimensional spaces}
\label{subsec:ext-harmonic}

\paragraph{Fourier transforms of measures.}
The Fourier transform of a finite Borel measure \(\nu\) on a Hilbert space
\(H\) is its characteristic functional
\(\widehat\nu(\xi)=\int_H e^{-i\langle x,\xi\rangle}\,\dd\nu(x)\).
Gaussian measures are characterized by this transform, and the theorems of
Minlos and Sazonov describe which positive-definite functions are transforms
of Radon measures on the dual of a nuclear space or on a Hilbert space
\citep{Bogachev1998book,DaPrato2014}.  Bochner's theorem fails in infinite
dimensions: \(\xi\mapsto\exp(-\|\xi\|^2/2)\) is not the transform of a finite
Borel measure on an infinite-dimensional Hilbert space, a fact also relevant
to kernels on function spaces \citep{WynneDuncan2022}.  In this paper, we
use the characteristic functional of the finite measure \(f\,\mu_Q\) only
along lines through the origin, \(\omega\mapsto\omega a\) for a fixed
direction \(a\) (rays, in the terminology of tomography), and relate it to
the one-dimensional Fourier transform in the bias variable.  The transform of
the measure is therefore used only through its restriction to such lines, and
this is why coefficient measures cannot be obtained from a weak Fourier-side
identity alone.

\paragraph{Fourier--Wiener, Segal--Bargmann, and white-noise transforms.}
\citet{CameronMartin1945,CameronMartin1947} introduced the Fourier--Wiener
transform of analytic and of square-integrable functionals on Wiener space.
\citet{Segal1956} identified the Gaussian \(L^2\) space with the symmetric
Fock space, the setting of the infinite-dimensional Segal--Bargmann transform,
and \citet{Gross1967} developed potential theory on abstract Wiener space.
In white-noise analysis, \citet{Kuo1989} defined a Fourier transform of
generalized functionals on the Gel'fand triple
\((\cS)\subset L^2(\mu)\subset(\cS)^*\) and related it to the Gross and
Volterra Laplacians; \citet{Lee1991} characterized measures through this
transform; the monograph of \citet{HidaKuoPotthoffStreit1993} collects the
Laplacians and the Fourier and Fourier--Mehler transforms.
\citet{JorgensenSongTian2023} construct transforms between Gaussian fields
and reproducing kernel Hilbert spaces and derive from them a Fourier
transform for general Gaussian processes.  These transforms act on
functionals of a Gaussian variable, and the Fourier--Wiener transform also
enters the inversion of the Gaussian Radon transform by \citet{Holmes2013}.
In this paper, we use only the one-dimensional Fourier transform and the
Gaussian characteristic functional, and the Hilbert space \(\cE_\alpha\) is
defined by a spectral norm rather than through a Fock decomposition.

\paragraph{Gaussian harmonic analysis.}
Harmonic analysis with respect to a Gaussian measure is organized around the
Hermite expansion, the Ornstein--Uhlenbeck semigroup, and Meyer's inequalities
for Gaussian Sobolev norms; see the survey of \citet{Bogachev2018OU} for the
infinite-dimensional theory and \citet{UrbinaRomero2019} for Gaussian
analogues of Littlewood--Paley and maximal functions.  Wiener chaos and the
totality of Hermite polynomials \citep{Janson1997book} underlie the density
arguments of this paper.  The Ornstein--Uhlenbeck semigroup and the
Fourier--Mehler transform play, for Gaussian measures, the role that
dilations play for Lebesgue measure.  In this paper, we do not use a dilation
group on \(\cH\); the scale of a direction \(a\) enters only through the
homogeneity of \(\nu_\alpha\).

\paragraph{Gaussian Radon transforms.}
Gaussian Radon transforms provide a related approach to integration over
infinite-dimensional hyperplanes.  Early work on Gaussian surface measures
is due to \citet{Hertle1980}; \citet{MihaiSengupta2007} develop the
Radon--Gauss transform and an inversion procedure.  The abstract Wiener
space construction of \citet{HolmesSengupta2012} associates Gaussian
probability measures with closed affine Cameron--Martin subspaces.
Support theorems are established in \citet{BecnelSengupta2012} and
\citet{HolmesSengupta2012}, while \citet{Holmes2013} uses disintegration
and the Segal--Bargmann transform for inversion.
\citet{Becnel2019} gives another inversion procedure by taking limits of
finite-dimensional Radon inversions under continuity and growth assumptions.
\citet{HolmesSengupta2014} study Brownian conditioning and Fock space
decompositions in classical Wiener space, and
\citet{HolmesSengupta2015} interpret kernel ridge regression through the
Gaussian Radon transform.

The relation to our analysis transform can be made explicit.  For
\(X\sim\mu_Q\), \(a\ne0\), and bounded continuous \(f\), set
\(v_a=\langle Qa,a\rangle\), let \(p_{v_a}\) be the density of
\(N(0,v_a)\), and write
\(\mathcal G_Qf(a,t)=\mathbb E[f(X)\mid\langle a,X\rangle=t]\),
using the Gaussian conditional kernel.  Disintegration gives
\[
 R_\rho f(a,b)=\int_\R \rho(t-b)\,p_{v_a}(t)\,
                    \mathcal G_Qf(a,t)\,\dd t.
\]
Thus \(R_\rho\) filters the Gaussian marginal-density-weighted Radon
transform in the bias variable.  Formally taking \(\rho=\delta\)
recovers \(p_{v_a}(b)\mathcal G_Qf(a,b)\), with this identity understood
distributionally in \(b\).  The Dirac filter, however, has infinite
self-admissibility integral
\(\int_\R|\omega|^{-\alpha}\,\dd\omega\) for every \(\alpha>0\),
so it lies outside our Plancherel theorem.  As a synthesis activation,
\(\sigma=\delta\) is covered by \cref{thm:tempered-reconstruction}
with a band-pass analysis filter, in the regularized dual-space sense.
The comparison above concerns continuous ambient linear measurements;
the Banach-space and general Cameron--Martin formulations have a broader
geometric scope.  In this paper, our contribution concerns the homogeneous
direction measure, the resulting ridgelet frame identities, and their
connection to finite-width neural approximation, rather than a generalization
of all Gaussian Radon results.

\paragraph{Wavelets and ridgelets on Wiener space.}
Transforms of wavelet and ridgelet type on the classical Wiener space
\(W=C_0([0,1];\R^d)\) with Cameron--Martin space \(\cH_{\mathrm{CM}}\) have
been constructed by Martias.  In \citet{Martias2012ridgelet}, a profile
\(\Psi\in L^2(\mu)\) of the form \(\Psi(\omega)=\eta(\|\omega\|_{\cH_{\mathrm{CM}}})\)
with \(\int_0^\infty t^3|\eta(t)|^2\,\dd t<\infty\) generates the atoms
\(\Psi_{a,\omega_0,\omega_1}(\omega)
 =a^{-1/2}\,\Psi\bigl(a^{-1}[\langle\omega,\omega_0\rangle_{\cH_{\mathrm{CM}}}\,\omega_0-\omega_1]\bigr)\)
with scale \(a>0\), orientation \(\omega_0\), and location \(\omega_1\); the
transform of \(f\in L^2(\mu_\delta)\) is the pairing with these atoms, and the
parameters are integrated against \(a^{-2}\,\dd a\), a Wiener measure with a
variance parameter in \(\omega_0\), and the Wiener measure in \(\omega_1\).
The resulting identity equates the integrated product of two transforms with
a constant multiple of the inner product in \(L^2(\mu_{\delta/\sqrt2})\); they
note that it is not a resolution of the identity: the pairing is taken in
\(L^2(\mu_\delta)\) while the identity returns the inner product of a
Gaussian space with a different variance, so synthesizing the coefficients
with the same atoms does not reproduce \(f\), and they replace reconstruction
by an approximation that converges as \(\delta\to0\).  They also derive a
characterization of H\"older regularity from the decay of the transform in
the scale.  The companion works treat a wavelet transform for
Wiener functionals with applications to densities of diffusions and
backward stochastic differential equations \citep{Martias2014wavelet}, tail
properties of the ridgelet transform \citep{Martias2014subexp}, and
multiresolution analyses of \(L^2(\mu)\) built from integer translations along
one Cameron--Martin direction and dyadic dilations of the argument, with
orthonormal wavelet bases obtained through the Fourier--Wiener transform
\citep{Martias2022wavelets}.  In this paper, the construction differs in four
respects.  The atom is a neuron \(\sigma(\langle a,x\rangle-b)\) whose profile
is a function of one real variable, so that synthesis produces a network;
the parameter measure is the homogeneous measure \(\nu_\alpha\) on \(\cH\) in
the direction and Lebesgue measure in the bias, rather than a product of a
scale measure and Gaussian measures in orientation and location; the
Plancherel identity holds on a single space \(\cE_\alpha\), with the
correction induced by the weights identified as the frame operator
\(T_\alpha\) and undone explicitly; and the admissibility conditions are
formulated on filter pairs in weighted Sobolev spaces, which admits tempered
non-polynomial activations.  We therefore regard \citet{Martias2012ridgelet}
as the closest precedent for a ridgelet transform on an infinite-dimensional
Gaussian space, and the present construction as a different transform
designed for network synthesis.

\subsection{Universal approximation with infinite-dimensional inputs}
\label{subsec:ext-uat}

\paragraph{Functionals.}
Approximation of nonlinear continuous functionals on compact subsets of
\(C[a,b]\) or \(L^p\) by one-hidden-layer networks goes back to
\citet{ChenChen1993}, with the operator version in \citet{Chen1995} and a
radial-basis version in \citet{ChenChen1995rbf}.  \citet{MhaskarHahm1997}
obtain lower bounds on the achievable order and nearly matching constructions
for functionals on \(L^p\) and \(C([-1,1]^s)\); \citet{Stinchcombe1999}
characterizes activations through compactifications and points out the
abstract structure needed for infinite-dimensional inputs; and
\citet{RossiConanGuez2005} prove universality and consistency for functional
multi-layer perceptrons whose neurons are \(\rho(\int w\,x+b)\), the
architecture of functional data analysis closest to the one studied here.
Rates for functional deep ReLU networks are given in
\citet{SongFanChenZhou2023} and \citet{SongLiuFanZhou2023}, and for
functionals on reproducing kernel Hilbert spaces, with kernel-dependent
rates, in \citet{ZhouSuhChengHuo2024}.  These rates are obtained through
finite-dimensional discretizations and depend on the discretization
dimension; in this paper, the rate depends on the activation and on the
total variation and moments of the ridgelet coefficient instead.

\paragraph{Fr\'echet, topological, and non-Euclidean inputs.}
\citet{Benth2023} prove compact-open universality for networks
\(x\mapsto\langle\ell,\sigma(Ax+b)\rangle\) with operator-valued activations on
Fr\'echet spaces and study finite-dimensional implementability;
\citet{Galimberti2025} extends the setting to non-metric spaces, and
\citet{Galimberti2026Lp} proves density in \(L^p(\mu)\), \(1\le p<\infty\), for
every Radon probability measure \(\mu\) under a bounded separating
activation, with finite-parameter and Banach-valued versions.
\citet{BilokopytovXanthos2025} give infinite-dimensional versions of the
characterization of non-polynomial activations and strengthen the density
result of \citet{Benth2023}.  \citet{KratsiosBilokopytov2020} describe
feature and readout maps that preserve universality.  \citet{Ismailov2026AMAI}
proves density of \(\sum_i c_i\sigma(w_if_i(x)-\theta_i)\) on compact subsets
of a topological space for Tauber--Wiener activations and feature families
with a separation property (Theorem~2.1) and, on locally convex spaces with
\(f_i\) in the dual and any continuous non-polynomial \(\sigma\)
(Theorem~2.2), which is the compact-open statement for the scalar ridge
architecture; \citet{Ismailov2026nonEuclidean} extends this to deep and
vector-valued networks.  Moreover,
\citet{GalimbertiKratsiosLivieri2026} and \citet{AcciaioKratsiosPammer2024}
construct universal causal models for sequences in Banach and metric spaces.
Signatures of paths provide a linear universal family on path spaces;
\citet{CoxKhedherMaessen2026} prove universality of linear functionals of the
signature for infinite-dimensional rough paths in the norm topology and, on
norm-bounded sets, in a weak-\(*\) topology.

\paragraph{Weighted and global universality.}
\citet{CuchieroSchmockerTeichmann2026} introduce functional input neural
networks on weighted spaces: an additive family of real-valued maps forms the
hidden layer, a scalar activation is applied to each neuron, and linear
readouts take values in a Banach space.  Their Theorem~4.13 proves density
in the weighted space \(\cB_\psi(X;Y)\), which yields global approximation
beyond compact sets, from weighted Stone--Weierstrass theorems; Proposition~4.4
gives sufficient conditions on the activation, including a Fourier condition
(A3) in the spirit of the non-polynomial criterion; Corollary~4.17 treats
non-anticipative functionals; Theorem~5.4 gives a global result for linear
functions of the signature; and Theorem~6.5 identifies the reproducing kernel
Hilbert space of a signature kernel with the Cameron--Martin space of a
Gaussian process with values in \(\cB_\psi(X)\).  Their first example,
\(\varphi(x)=\sum_{n=1}^N y_n\rho(\int_S x\,\dd\nu_n+b_n)\) with finite signed
Radon measures \(\nu_n\), is a finite integral network in the terminology of
this paper, with a discrete coefficient measure.  \citet{SchmockerTeichmann2026}
extend this to differentiable maps on weighted infinite-dimensional
manifolds, approximating directional derivatives through a weighted Nachbin
theorem, and \citet{NeufeldSchmocker2024noncompact} prove the one-dimensional
weighted results, including dimension-independent rates for targets with
integrable Fourier transforms, on which the lifting relies; see also
\citet{vanNuland2024} for uniform approximation on non-compact domains.
Universality for operators and their Fr\'echet derivatives is studied by
\citet{Yao2025DIFNO} for Fourier neural operators and by \citet{deFeo2026}, who
shows that it fails in Fr\'echet compact-open topologies and Fr\'echet--Sobolev
norms and holds in Bastiani topologies and weighted Bastiani--Sobolev spaces.
In this paper, the universality theorem is of compact-open type, with a
coefficient measure chosen for each compact set and each tolerance, and it is
proved through targets with a spectral density; the global weighted
statements above are stronger in topology but do not identify the
coefficients.  Conversely, the explicit coefficient \(W_\rho G\), the
Plancherel identity, and the finite-width rate of this paper have no
counterpart in these existence theorems, and the admissible analysis filters
form a family, each giving a different representation of the same target.

\subsection{Neural operators: universality, rates, and complexity}
\label{subsec:ext-operators}

Universal approximation for DeepONets \citep{Lu2021,Lanthaler2022}, Fourier
neural operators \citep{KovachkiLanthalerMishra2021}, PCA-Net
\citep{Bhattacharya2021,Lanthaler2023PCANet}, and the general neural-operator
framework \citep{Kovachki2021NOreview} is reviewed in
\citet{KovachkiLanthalerStuart2024}.  \citet{LanthalerLiStuart2023} show that
nonlocality and nonlinearity together imply universality, with an averaging
neural operator as the minimal example; \citet{GoedekeFernsel2025} prove a
compact-open universality theorem for encoder--decoder architectures between
normed or metric spaces and distinguish architectures that can be chosen
independently of the compact set; \citet{ZappalaBagherian2024} treat
transformers and neural integral operators between Banach spaces; and
\citet{ChenHeTai2026} obtain universality for continuous-depth neural flows.

Quantitative results for solution operators of partial differential equations
include the error estimates of \citet{Lanthaler2022} and, for nonlinear
parabolic equations, the Picard-iteration construction of
\citet{FuruyaTaniguchiOkuda2024}, whose Theorem~1 achieves accuracy
\(\varepsilon\) on a ball of initial data with depth \(O((\log\varepsilon^{-1})^2)\)
and width \(O(\varepsilon^{-1}(\log\varepsilon^{-1})^2)\), the nonlinearity being
approximated as a function of one variable; \citet{TaniguchiSonoda2026}
derive generalization bounds for the corresponding Picard-type learning
problem in which the entropy term does not grow with the Picard depth.  These
rates exploit the structure of the equation.  For operator classes defined by
regularity alone, lower bounds show that no such rates exist:
\citet{LanthalerStuart2025} establish a curse of parametric complexity for
\(C^r\) and Lipschitz operators that applies to PCA-Net, DeepONet, and FNO,
\citet{Lanthaler2024entropy} gives metric-entropy lower bounds that hold for
every activation, \citet{KovachkiLanthalerMhaskar2024} prove a corresponding
curse of data complexity through \(n\)-widths, and \citet{AdcockMaierParhi2025}
show that the minimax risk for Lipschitz operators between Hilbert spaces
cannot decay algebraically in the sample size, H\"older smoothness giving no
improvement.  \citet{Kratsios2024MoNO} reduce the active complexity by routing
inputs to one of several experts.  In this paper, the rate is consistent with
these lower bounds because it applies to targets whose ridgelet coefficient has
finite total variation and moments, a Barron-type class determined by
smoothness of the spectral density along lines through the origin (regularity
along rays), and not to all Lipschitz operators.

\subsection{Summary tables}
\label{subsec:ext-tables}

\Cref{tab:uat-functionals,tab:uat-operators} list universal approximation
theorems with infinite-dimensional inputs, and
\cref{tab:rate-upper,tab:rate-upper-operators,tab:rate-lower} list
quantitative upper bounds and lower bounds, each with the class of targets, the form of the approximant,
and the norm or topology in which approximation is measured.  ``One hidden layer'' means a finite sum of neurons with trained
output weights; ``random features'' means that the hidden parameters are
sampled and only the readout is chosen; cc means uniform convergence on
compact subsets of the input space.

\begin{table}[htbp]
\footnotesize
\centering
\caption{Universal approximation with infinite-dimensional inputs:
functionals, Fr\'echet and topological inputs, weighted spaces, and
signatures.}
\label{tab:uat-functionals}
\begin{tabular}{@{}p{0.16\textwidth}p{0.26\textwidth}p{0.29\textwidth}p{0.18\textwidth}@{}}
\toprule
Reference & Target class & Approximant & Norm or topology \\
\midrule
\citet{ChenChen1993} &
continuous functionals on a compact \(K\subset C[a,b]\) or \(L^p[a,b]\) &
one hidden layer, Tauber--Wiener activation applied to finitely many point
values or integrals of the input &
uniform on \(K\) \\
\citet{Chen1995} &
continuous operators \(G\colon V\to C(K_2)\), \(V\subset C(K_1)\) compact &
branch--trunk network (the DeepONet form), Tauber--Wiener activations &
uniform on \(V\times K_2\) \\
\citet{Stinchcombe1999} &
continuous functions on compactifications; functionals through the same
structure &
one hidden layer, activation characterized by the compactification &
uniform on the compactification \\
\citet{RossiConanGuez2005} &
continuous functionals on a compact \(K\subset L^2(V)\) &
one hidden layer of functional neurons \(\rho(\int wx+b)\), regular
activation &
uniform on \(K\) \\
\citet{Benth2023} &
continuous \(f\colon X\to\R\), \(X\) Fr\'echet &
one hidden layer \(\sum_i\langle\ell_i,\sigma(A_ix+b_i)\rangle\), operator
activation \(\sigma\colon X\to X\) bounded and separating &
cc \\
\citet{Galimberti2025} (Thm.~4.2) &
continuous functions on quasi-Polish spaces &
same architecture composed with an injection into \(\ell^2\); finite-rank
projections (Thm.~4.3) &
cc \\
\citet{Galimberti2026Lp} &
\(f\in L^p(\mu)\), \(\mu\) a Radon probability on a Fr\'echet space,
\(1\le p<\infty\); Banach-valued version &
same architecture, bounded separating operator activation; finite-parameter
version through a Schauder basis &
\(L^p(\mu)\); Bochner \(L^p(\mu;E)\) \\
\citet{KratsiosBilokopytov2020} &
continuous functions between spaces with feature and readout maps, including
Cartan--Hadamard manifolds &
readout \(\circ\) Euclidean feed-forward network \(\circ\) feature map &
cc \\
\citet{Ismailov2026AMAI} (Thm.~2.1) &
continuous functions on a topological space &
\(\sum_i c_i\sigma(w_if_i(x)-\theta_i)\), \(f_i\) from a family with the
separation (D-)property, Tauber--Wiener \(\sigma\) &
cc \\
\citet{Ismailov2026AMAI} (Thm.~2.2) &
continuous functions on a locally convex space &
\(\sum_i c_i\sigma(f_i(x)-\theta_i)\), \(f_i\) continuous linear
functionals, continuous non-polynomial \(\sigma\) &
cc \\
\citet{CuchieroSchmockerTeichmann2026} (Thm.~4.13) &
\(f\in\cB_\psi(X;Y)\): closure of bounded continuous maps in the weighted
sup norm on a weighted space \((X,\psi)\), \(Y\) Banach &
one hidden layer \(\sum_n y_n\rho(h_n(x))\), \(h_n\) in an additive family
(e.g.\ continuous linear functionals), \(\rho\) \(\psi_1\)-activating
(sigmoidal or Fourier condition) &
weighted sup norm \(\sup_x\|f-\varphi\|_Y/\psi(x)\) (global) \\
\citet{CuchieroSchmockerTeichmann2026} (Thm.~5.4) &
\(f\in\cB_\psi\) on H\"older rough paths &
linear functions of the signature &
weighted sup norm (global) \\
\citet{SchmockerTeichmann2026} (Thm.~4.13) &
\(f\in\cB^k_\Psi(M;Y)\): \(C^k\) maps with weighted derivatives on a weighted
\(C^k_{\mathrm{loc}}\) manifold with the bounded approximation property &
one hidden layer, strongly non-polynomial \(\rho\in\cB^k_c(\R)\), additive
family &
weighted \(C^k\) norm, derivatives up to order \(k\) (global) \\
\citet{CoxKhedherMaessen2026} &
continuous functions on compact sets of infinite-dimensional weakly geometric
H\"older rough paths &
linear functionals of the signature &
uniform on compact sets (norm topology); on norm-bounded sets (weak-\(*\)
topology) \\
\bottomrule
\end{tabular}
\end{table}

\begin{table}[htbp]
\footnotesize
\centering
\caption{Universal approximation with infinite-dimensional inputs: neural
operators, derivatives, and this paper.}
\label{tab:uat-operators}
\begin{tabular}{@{}p{0.16\textwidth}p{0.26\textwidth}p{0.29\textwidth}p{0.18\textwidth}@{}}
\toprule
Reference & Target class & Approximant & Norm or topology \\
\midrule
\citet{KovachkiLanthalerMishra2021} &
continuous operators between Sobolev spaces on the torus, on compact input
sets &
Fourier neural operator &
uniform on compact sets \\
\citet{LanthalerLiStuart2023} &
continuous operators between function spaces on Euclidean domains, on
compact input sets &
averaging neural operator (one spatial average plus pointwise nonlinearity) &
uniform on compact sets \\
\citet{GoedekeFernsel2025} &
continuous operators between normed or metric spaces satisfying
encoder--decoder approximation properties &
encoder--decoder architectures (DeepONet, BasisONet, MIONet, frame based) &
cc; also compact-set-independent choice of the architecture \\
\citet{Yao2025DIFNO} (Thms.~3.2, 4.3) &
\(C^1\) operators between Sobolev spaces on the torus whose derivative is
continuous into Hilbert--Schmidt operators &
Fourier neural operator with smooth activation &
sup over a compact set of the operator error and of the Hilbert--Schmidt
derivative error; \(L^2_\mu\) of both for a Borel probability \(\mu\) with
moment conditions \\
\citet{deFeo2026} (Thms.~4.3, 4.6; Cor.~4.19) &
\(C^k\) operators together with derivatives, on Banach spaces &
encoder--decoder classes (DeepONet, Deep-H-ONet, Schauder nets) &
not dense in Fr\'echet \(C^k\) compact-open or Fr\'echet--Sobolev norms,
already for bounded linear targets; dense in Bastiani \(C^k\) compact-open
(derivatives uniform on compacts, directions from compacts), weighted
Bastiani--Sobolev, and Gaussian Sobolev \(W^{k,p}_\gamma\) norms \\
This paper &
continuous \(f\colon\cH\to\outY\), \(\cH\) separable Hilbert, \(\outY\)
separable Hilbert &
one hidden layer \(\sum_j v_j\sigma(\langle a_j,x\rangle-b_j)\),
continuous non-polynomial \(\sigma\) of polynomial growth; coefficients
from the ridgelet transform of a spectral approximant &
cc \\
\bottomrule
\end{tabular}
\end{table}

\begin{table}[htbp]
\footnotesize
\centering
\caption{Quantitative upper bounds for functionals and for networks with
Hilbert or Banach inputs, as rates in the width \(N\).}
\label{tab:rate-upper}
\begin{tabular}{@{}p{0.14\textwidth}p{0.23\textwidth}p{0.22\textwidth}p{0.14\textwidth}p{0.15\textwidth}@{}}
\toprule
Reference & Target class & Approximant & Norm & Bound \\
\midrule
\citet{Barron1993} &
\(f\) on a ball in \(\R^d\) with \(\int|\xi||\widehat f(\xi)|\,\dd\xi\le C\) &
one hidden layer, sigmoidal, width \(N\) &
\(L^2(\mu)\) on the ball &
\(O(C/\sqrt N)\), no explicit \(d\) \\
\citet{MhaskarHahm1997} &
smooth functionals on compact subsets of \(L^p\) or \(C([-1,1]^s)\) &
generalized translation networks &
uniform &
order bounds in the number of neurons, with matching lower bounds \\
\citet{SongFanChenZhou2023} (Thms.~1, 2) &
functionals on \(L^p([-1,1]^s)\) with \(\omega_F(r)\le cr^\lambda\), on the
unit ball \(K\) of \(C^\beta\) (Thm.~2) &
functional deep ReLU network: a bounded linear first layer followed by ReLU
layers, \(M\) nonzero weights, depth \(O(L_M^2)\) &
sup over \(K\) &
\(O(L_M^{-\beta\lambda/s})\) with \(L_M=\log M/\log\log M\) \\
\citet{Gonon2023reservoir} (Thm.~1) &
\(f\) on a separable Hilbert space with a Fourier--Stieltjes representation,
second moments, and a density w.r.t.\ the sampling measure &
random ReLU features, width \(N\), trained readout &
expected squared \(L^2\) error over the input law &
\(O(N^{-1})\) for the squared error \\
\citet{Korolev2022} (Thm.~3.24) &
variation space \(F_1(X;Y)\) of maps between duals of separable Banach
spaces with Schauder bases, Lipschitz for the weak-\(*\) metric on \(Y\) &
\(\sum_{i=1}^n\alpha_i(K_ix)_+\) with finite-rank \(K_i\), \(\|K_i\|\le1\) &
weak-\(*\) metric on \(Y\), pointwise with factor \(\|x\|\); Bochner
\(L^p_\mu\) for \(\mu\) with \(p\) moments &
\(n^{-1/2}\) times \(c\|f\|_{F_1}\|x\|\); inverse theorem \\
\citet{NeufeldSchmocker2023random} (Thm.~4.5; Cor.~4.21) &
Barron space \(\mathcal B^r_{G,\theta}(X)\) of a separable Banach space of
type \(t\); in Cor.~4.21, a Barron--ridgelet class on \(U\subset\R^m\)
defined by an admissible filter pair &
random feature model with \(N\) features; random neural network with \(N\)
neurons &
Bochner \(L^r\) in the \(X\)-norm; weighted Sobolev \(W^{k,p}(U,w)\) &
\(N^{-\kappa}\), \(\kappa=1-1/\min(r,t)\), so \(N^{-1/2}\) for type \(2\),
\(r\ge2\); \(\kappa=1-1/\min(2,p,r)\) in Cor.~4.21 \\
\citet{yang-shi2026} (Thm.~3.1) &
variation class of Hilbert-to-Hilbert operators &
encoder--decoder neural operator, width \(N\) &
Bochner \(L^q\), \(2\le q<\infty\) &
encoding errors plus \(O(N^{-1/2})\) independent of the encoding dimension \\
This paper &
coefficient measure with variation \(V\) and second moment \(M_2\)
(e.g.\ spectral densities smooth along lines through the origin) &
one hidden layer, globally Lipschitz \(\sigma\), width \(N\), weights
sampled from the coefficient &
expected sup on a compact \(K\subset\cH\), \(\outY\)-valued &
\(8VN^{-1/2}(|\sigma(0)|+\Lip(\sigma)R_KM_2)\) \\
\bottomrule
\end{tabular}
\end{table}

\begin{table}[htbp]
\footnotesize
\centering
\caption{Quantitative upper bounds for neural operators and surrogates of
operators.}
\label{tab:rate-upper-operators}
\begin{tabular}{@{}p{0.14\textwidth}p{0.23\textwidth}p{0.22\textwidth}p{0.14\textwidth}p{0.15\textwidth}@{}}
\toprule
Reference & Target class & Approximant & Norm & Bound \\
\midrule
\citet{Lanthaler2022} &
PDE solution operators; general error decomposition &
DeepONet &
\(L^2(\mu)\) over the input law &
encoder, approximator, reconstruction errors; algebraic rates in examples \\
\citet{FuruyaTaniguchiOkuda2024} (Thm.~1) &
solution operator of a nonlinear parabolic equation on an \(L^\infty\)
ball, short time &
weight-tied ReLU neural operator, Green-function expansion of rank \(N\) &
sup over the ball of \(L^r(0,T;L^s)\) &
depth \(O((\log\varepsilon^{-1})^2)\), width
\(O(\varepsilon^{-1}(\log\varepsilon^{-1})^2)\), \(N\) by the Green-function
error \\
\citet{Kratsios2024MoNO} &
uniformly continuous, in particular Lipschitz, operators on compact Sobolev
sets &
mixture of neural operators routed by a tree &
uniform on the compact set &
active expert of size \(O(\varepsilon^{-1})\) for Lipschitz targets \\
\citet{AdcockGriebelMaier2026} (Thm.~4.1) &
Lipschitz \(F\in L^2_\mu(X;Y)\), \(\mu\) Gaussian, whose PCA latent
representation lies in a Gaussian Sobolev space \(H^k\) &
Hermite--PCA approximation from \(N_X,N_Y\) PCA samples and \(M\)
input--output samples &
\(L^2_\mu(X;Y)\) &
PCA tails, sampling terms, and a term improving with the Sobolev order
\(k\) \\
\bottomrule
\end{tabular}
\end{table}

\begin{table}[htbp]
\footnotesize
\centering
\caption{Lower bounds: no algebraic rate exists for the class.}
\label{tab:rate-lower}
\begin{tabular}{@{}p{0.14\textwidth}p{0.23\textwidth}p{0.22\textwidth}p{0.14\textwidth}p{0.15\textwidth}@{}}
\toprule
Reference & Target class & Approximant & Norm & Bound \\
\midrule
\citet{LanthalerStuart2025} (Thm.~2.11, Cor.~2.12) &
\(r\)-times Fr\'echet differentiable functionals or operators on a compact
\(K\) containing an infinite-dimensional hypercube \(Q_\alpha\) &
functionals of neural-network type \(\Phi\circ L\), \(L\) finite-rank
linear, \(\Phi\) a ReLU network (covers PCA-Net, DeepONet, NOMAD, FNO) &
sup over \(K\) &
complexity \(\ge\exp(c\varepsilon^{-\kappa})\),
\(\kappa=1/((\alpha+1+\delta)r)\) \\
\citet{Lanthaler2024entropy} &
Lipschitz operators between Banach spaces &
any architecture, size measured in bits &
uniform on a compact set; expectation over an input law &
size exponential in \(\varepsilon^{-1}\), for every activation \\
\citet{KovachkiLanthalerMhaskar2024} &
Lipschitz and Fr\'echet-differentiable operators &
any method using \(n\) input--output samples (\(n\)-widths) &
as stated there &
\(n\) exponential in \(\varepsilon^{-1}\) \\
\citet{AdcockMaierParhi2025} &
uniformly bounded Lipschitz (or H\"older) operators between separable
Hilbert spaces &
any estimator from \(n\) noisy samples (minimax) &
\(L^2\) risk over an input law with given covariance spectrum &
no algebraic rate in \(n\); sharp for exponentially decaying spectra \\
\bottomrule
\end{tabular}
\end{table}

\subsection{Finite-width bounds and random features}
\label{subsec:ext-finite-width}

Finite-width bounds by sampling a parameter distribution follow
\citet{Barron1993} and the variation-norm viewpoint of \citet{Bach2017b} and
\citet*{E-Ma-Wu2022barron}.  \citet{Gonon2023reservoir} prove random ReLU
approximation on separable Hilbert input spaces with an expected squared
\(L^2\) error of order \(N^{-1}\) under a Fourier--Stieltjes representation
with moment conditions and absolute continuity with respect to a sampling
measure, and they treat infinite input histories in \citet{Gonon2024}.
\citet{Korolev2022} studies two-layer networks between Banach spaces with
separable preduals whose nonlinearity is the lattice positive part, proving
direct and inverse theorems with Monte Carlo rates for the associated
variation space.  \citet{NeufeldSchmocker2023random} prove universality and
rates for Banach-space-valued random feature models in Bochner norm; their
rates for random networks in weighted Sobolev spaces on \(\R^m\) are derived
from the finite-dimensional ridgelet reconstruction with an admissible
filter pair, which is the finite-dimensional counterpart of the sampling
argument used here.  Further,
\citet{yang-shi2026} obtain dimension-independent bounds for encoder--decoder
neural operators in Bochner \(L^q\).  Kernel and random-feature surrogates for
operators are developed by \citet{NelsenStuart2021,NelsenStuart2024} and
\citet{Batlle2024}.  In this paper, the sampled network is a random feature
model whose sampling density is proportional to the modulus of the explicit
ridgelet coefficient; the additional content is that the coefficient is
characterized by an analysis operator with a Plancherel identity and a
closed-range projection, that the rate is uniform on compact sets for
Lipschitz activations, and that it extends to Hilbert-valued outputs in a
Bochner norm.  Continuous frames and coorbit discretization
\citep{Fornasier2005,Holighaus2017errata} provide the general framework in
which the coefficient space could be studied further; we do not assume their
hypotheses.

\subsection{Gaussian measures, functional data, and statistical learning}
\label{subsec:ext-data}

Gaussian measures on Hilbert spaces and their Cameron--Martin spaces are
treated in \citet{Bogachev1998book}, \citet{DaPrato2014}, and
\citet{Prevot2007book}.  The same objects appear in functional data analysis,
where observations are modeled as realizations of a Hilbert-space-valued
process and expanded in Karhunen--Lo\`eve form
\citep{WangChiouMuller2016,FerratyVieu2006}, and in Gaussian process
regression on function spaces \citep{ShiChoi2011}, whose posterior
contraction rates depend on the position of the target relative to the
Cameron--Martin space \citep{vanderVaartvanZanten2008}; the equivalence
between Gaussian process regression and kernel ridge regression is reviewed in
\citet{Kanagawa2018}.  Characteristic kernels on separable Hilbert spaces and
the explicit reproducing kernel Hilbert space of Gauss-type kernels are
constructed in \citet{WynneDuncan2022}, and operator-valued kernels for
functional responses in \citet{Kadri2016}.  Neural networks with functional
inputs have been developed in statistics by \citet{RossiConanGuez2005},
\citet{YaoMuellerWang2021}, and \citet{ThindMultaniCao2023}.  On the learning
side, convergence rates for linear operators from noisy data are obtained in
\citet{deHoop2023}, minimax rates for nonlinear operators in
\citet{AdcockMaierParhi2025}, and near-optimal learning of Gaussian Sobolev
operators through Hermite expansions and principal component analysis in
\citet{AdcockGriebelMaier2026}.  In this paper, we develop an approximation
theory, not a learning theory: the input measure \(\mu_Q\) is the Gaussian measure of the functional data model,
the spectral density \(G\) is a Fourier-side description of the target with
respect to this measure, and the sampled network is a Monte Carlo
construction rather than an estimator from data.  Statistical rates are not
claimed.

\FloatBarrier

\section*{Declarations}

\paragraph{Funding.}
This work was partially supported by
JSPS KAKENHI 24K21316, 25H01453, 25K24910,
JST BOOST JPMJBY24E2,
JST CREST JPMJCR25I5, JPMJCR24Q1,
JST ERATO JPMJER2601, 
JST ARiSE JPMJAR2633.

\paragraph{Competing interests.}
The authors declare that they have no competing interests.

\paragraph{Data availability.}
No external datasets were used. The numerical results supporting
\cref{sec:numerics}, including the data in \cref{tab:numerics,fig:numerics},
are provided as \texttt{results.json} and figure files in the supplementary
archive \texttt{python-code.zip}. The computational setup is described in
\cref{app:numerics} and in the archive's README.

\paragraph{Code availability.}
The Lean formalization snapshot is supplied in \texttt{lean-code.zip};
the Python code for \cref{sec:numerics} is supplied in \texttt{python-code.zip}.
Both supplementary archives include instructions for reproducing the results.

\bibliographystyle{plainnat}
\bibliography{refs}

\begin{thebibliography}{103}
\providecommand{\natexlab}[1]{#1}
\providecommand{\url}[1]{\texttt{#1}}
\expandafter\ifx\csname urlstyle\endcsname\relax
  \providecommand{\doi}[1]{doi: #1}\else
  \providecommand{\doi}{doi: \begingroup \urlstyle{rm}\Url}\fi

\bibitem[Acciaio et~al.(2024)Acciaio, Kratsios, and
  Pammer]{AcciaioKratsiosPammer2024}
Beatrice Acciaio, Anastasis Kratsios, and Gudmund Pammer.
\newblock {\href{https://doi.org/10.1111/mafi.12389}{Designing universal causal
  deep learning models: the geometric (hyper)transformer}}.
\newblock \emph{Mathematical Finance}, 34\penalty0 (2):\penalty0 671--735,
  2024.

\bibitem[Adcock et~al.(2025)Adcock, Maier, and Parhi]{AdcockMaierParhi2025}
Ben Adcock, Gregor Maier, and Rahul Parhi.
\newblock {\href{https://arxiv.org/abs/2512.17805}{Towards sharp minimax risk
  bounds for operator learning}}, 2025.
\newblock arXiv:2512.17805.

\bibitem[Adcock et~al.(2026)Adcock, Griebel, and Maier]{AdcockGriebelMaier2026}
Ben Adcock, Michael Griebel, and Gregor Maier.
\newblock {\href{https://arxiv.org/abs/2607.11921}{Near-optimal learning of
  Gaussian Sobolev operators}}, 2026.
\newblock arXiv:2607.11921.

\bibitem[Bach(2017)]{Bach2017b}
Francis Bach.
\newblock {\href{http://jmlr.org/papers/v18/14-546.html}{Breaking the Curse of
  Dimensionality with Convex Neural Networks}}.
\newblock \emph{Journal of Machine Learning Research}, 18\penalty0
  (19):\penalty0 1--53, 2017.

\bibitem[Barron(1993)]{Barron1993}
Andrew~R. Barron.
\newblock {\href{https://doi.org/10.1109/18.256500}{Universal approximation
  bounds for superpositions of a sigmoidal function}}.
\newblock \emph{IEEE Transactions on Information Theory}, 39\penalty0
  (3):\penalty0 930--945, 1993.

\bibitem[Batlle et~al.(2024)Batlle, Darcy, Hosseini, and Owhadi]{Batlle2024}
Pau Batlle, Matthieu Darcy, Bamdad Hosseini, and Houman Owhadi.
\newblock {\href{https://doi.org/10.1016/j.jcp.2023.112549}{Kernel methods are
  competitive for operator learning}}.
\newblock \emph{Journal of Computational Physics}, 496:\penalty0 112549, 2024.

\bibitem[Becnel(2019)]{Becnel2019}
Jeremy~J. Becnel.
\newblock {\href{https://doi.org/10.31390/cosa.13.2.04}{A limiting process to
  invert the Gauss--Radon transform}}.
\newblock \emph{Communications on Stochastic Analysis}, 13\penalty0
  (2):\penalty0 219--242, 2019.

\bibitem[Becnel and Sengupta(2012)]{BecnelSengupta2012}
Jeremy~J. Becnel and Ambar~N. Sengupta.
\newblock {\href{https://doi.org/10.1090/S0002-9947-2011-05365-1}{A support
  theorem for a Gaussian Radon transform in infinite dimensions}}.
\newblock \emph{Transactions of the American Mathematical Society},
  364\penalty0 (3):\penalty0 1281--1291, 2012.

\bibitem[Benth et~al.(2023)Benth, Detering, and Galimberti]{Benth2023}
Fred~Espen Benth, Nils Detering, and Luca Galimberti.
\newblock {\href{https://doi.org/10.1007/s10472-022-09824-z}{Neural networks in
  Fr\'echet spaces}}.
\newblock \emph{Annals of Mathematics and Artificial Intelligence}, 91\penalty0
  (1):\penalty0 75--103, 2023.

\bibitem[Bhattacharya et~al.(2021)Bhattacharya, Hosseini, Kovachki, and
  Stuart]{Bhattacharya2021}
Kaushik Bhattacharya, Bamdad Hosseini, Nikola~B. Kovachki, and Andrew~M.
  Stuart.
\newblock {\href{https://doi.org/10.5802/smai-jcm.74}{Model reduction and
  neural networks for parametric {PDEs}}}.
\newblock \emph{The SMAI Journal of Computational Mathematics}, 7:\penalty0
  121--157, 2021.

\bibitem[Bilokopytov and Xanthos(2026)]{BilokopytovXanthos2025}
Eugene Bilokopytov and Foivos Xanthos.
\newblock {\href{https://doi.org/10.1016/j.jmaa.2026.130632}{A universal
  approximation theorem and its applications to vector lattice theory}}.
\newblock \emph{Journal of Mathematical Analysis and Applications},
  561\penalty0 (1):\penalty0 130632, 2026.

\bibitem[Bogachev(1998)]{Bogachev1998book}
Vladimir~I. Bogachev.
\newblock \emph{{\href{https://pubs.ams.org/ebooks/surv/062}{Gaussian
  Measures}}}, volume~62 of \emph{Mathematical Surveys and Monographs}.
\newblock American Mathematical Society, Providence, RI, 1998.

\bibitem[Bogachev(2018)]{Bogachev2018OU}
Vladimir~I. Bogachev.
\newblock {\href{https://doi.org/10.1070/RM9812}{Ornstein--Uhlenbeck operators
  and semigroups}}.
\newblock \emph{Russian Mathematical Surveys}, 73\penalty0 (2):\penalty0
  191--260, 2018.

\bibitem[Cameron and Martin(1945)]{CameronMartin1945}
Robert~H. Cameron and William~T. Martin.
\newblock {\href{https://doi.org/10.1215/S0012-7094-45-01244-0}{Fourier--Wiener
  transforms of analytic functionals}}.
\newblock \emph{Duke Mathematical Journal}, 12\penalty0 (3):\penalty0 489--507,
  1945.

\bibitem[Cameron and Martin(1947)]{CameronMartin1947}
Robert~H. Cameron and William~T. Martin.
\newblock {\href{https://doi.org/10.1215/S0012-7094-47-01409-9}{Fourier--Wiener
  transforms of functionals belonging to {$L_2$} over the space {$C$}}}.
\newblock \emph{Duke Mathematical Journal}, 14\penalty0 (1):\penalty0 99--107,
  1947.

\bibitem[Cand\`es(1998)]{Candes.PhD}
Emmanuel~Jean Cand\`es.
\newblock
  \emph{{\href{https://searchworks.stanford.edu/view/9949708}{Ridgelets: theory
  and applications}}}.
\newblock PhD thesis, Stanford University, 1998.

\bibitem[Cand\`es and Donoho(1999)]{Candes1999a}
Emmanuel~Jean Cand\`es and David~Leigh Donoho.
\newblock {\href{https://doi.org/10.1098/rsta.1999.0444}{Ridgelets: a key to
  higher-dimensional intermittency?}}
\newblock \emph{Philosophical Transactions of the Royal Society of London A:
  Mathematical, Physical and Engineering Sciences}, 357\penalty0
  (1760):\penalty0 2495--2509, 1999.

\bibitem[Chen et~al.(2026)Chen, He, and Tai]{ChenHeTai2026}
Shuang Chen, Juncai He, and Xue-Cheng Tai.
\newblock {\href{https://arxiv.org/abs/2605.22557}{Neural flow operators can
  approximate any operator: abstract frameworks and universal approximations}},
  2026.
\newblock arXiv:2605.22557.

\bibitem[Chen and Chen(1993)]{ChenChen1993}
Tianping Chen and Hong Chen.
\newblock {\href{https://doi.org/10.1109/72.286886}{Approximations of
  continuous functionals by neural networks with application to dynamic
  systems}}.
\newblock \emph{IEEE Transactions on Neural Networks}, 4\penalty0 (6):\penalty0
  910--918, 1993.

\bibitem[Chen and Chen(1995{\natexlab{a}})]{Chen1995}
Tianping Chen and Hong Chen.
\newblock {\href{https://doi.org/10.1109/72.392253}{Universal approximation to
  nonlinear operators by neural networks with arbitrary activation functions
  and its application to dynamical systems}}.
\newblock \emph{IEEE Transactions on Neural Networks}, 6\penalty0 (4):\penalty0
  911--917, 1995{\natexlab{a}}.

\bibitem[Chen and Chen(1995{\natexlab{b}})]{ChenChen1995rbf}
Tianping Chen and Hong Chen.
\newblock {\href{https://doi.org/10.1109/72.392252}{Approximation capability to
  functions of several variables, nonlinear functionals, and operators by
  radial basis function neural networks}}.
\newblock \emph{IEEE Transactions on Neural Networks}, 6\penalty0 (4):\penalty0
  904--910, 1995{\natexlab{b}}.

\bibitem[Cox et~al.(2026)Cox, Khedher, and Maessen]{CoxKhedherMaessen2026}
Sonja Cox, Asma Khedher, and Thijs Maessen.
\newblock {\href{https://arxiv.org/abs/2603.03058}{Universal approximation by
  signatures for infinite-dimensional rough paths}}, 2026.
\newblock arXiv:2603.03058.

\bibitem[Cuchiero et~al.(2026)Cuchiero, Schmocker, and
  Teichmann]{CuchieroSchmockerTeichmann2026}
Christa Cuchiero, Philipp Schmocker, and Josef Teichmann.
\newblock {\href{https://doi.org/10.1007/s00365-025-09726-3}{Global universal
  approximation of functional input maps on weighted spaces}}.
\newblock \emph{Constructive Approximation}, 63\penalty0 (2):\penalty0
  537--612, 2026.

\bibitem[Da~Prato and Zabczyk(2014)]{DaPrato2014}
Giuseppe Da~Prato and Jerzy Zabczyk.
\newblock \emph{{\href{https://doi.org/10.1017/CBO9781107295513}{Stochastic
  Equations in Infinite Dimensions}}}, volume 152 of \emph{Encyclopedia of
  Mathematics and its Applications}.
\newblock Cambridge University Press, Cambridge, 2nd edition, 2014.

\bibitem[de~Feo(2026)]{deFeo2026}
Filippo de~Feo.
\newblock {\href{https://arxiv.org/abs/2605.15285}{Universal approximation of
  nonlinear operators and their derivatives}}, 2026.
\newblock arXiv:2605.15285.

\bibitem[de~Hoop et~al.(2023)de~Hoop, Kovachki, Nelsen, and Stuart]{deHoop2023}
Maarten~V. de~Hoop, Nikola~B. Kovachki, Nicholas~H. Nelsen, and Andrew~M.
  Stuart.
\newblock {\href{https://doi.org/10.1137/21M1442942}{Convergence rates for
  learning linear operators from noisy data}}.
\newblock \emph{SIAM/ASA Journal on Uncertainty Quantification}, 11\penalty0
  (2):\penalty0 480--513, 2023.

\bibitem[Diestel and Uhl(1977)]{Diestel1977book}
Joseph Diestel and John~Jerry Uhl, Jr.
\newblock \emph{\href{https://doi.org/10.1090/surv/015}{Vector Measures}},
  volume~15 of \emph{Mathematical Surveys}.
\newblock American Mathematical Society, Providence, RI, 1977.
\newblock \doi{10.1090/surv/015}.

\bibitem[E et~al.(2022)E, Ma, and Wu]{E-Ma-Wu2022barron}
Weinan E, Chao Ma, and Lei Wu.
\newblock {\href{https://doi.org/10.1007/s00365-021-09549-y}{The Barron Space
  and the Flow-Induced Function Spaces for Neural Network Models}}.
\newblock \emph{Constructive Approximation}, 55\penalty0 (1):\penalty0
  369--406, 2022.

\bibitem[Ferraty and Vieu(2006)]{FerratyVieu2006}
Fr{\'e}d{\'e}ric Ferraty and Philippe Vieu.
\newblock \emph{Nonparametric Functional Data Analysis: Theory and Practice}.
\newblock Springer Series in Statistics. Springer, New York, 2006.

\bibitem[Fornasier and Rauhut(2005)]{Fornasier2005}
Massimo Fornasier and Holger Rauhut.
\newblock {\href{https://doi.org/10.1007/s00041-005-4053-6}{Continuous Frames,
  Function Spaces, and the Discretization Problem}}.
\newblock \emph{Journal of Fourier Analysis and Applications}, 11\penalty0
  (3):\penalty0 245--287, 2005.

\bibitem[Furuya et~al.(2025)Furuya, Taniguchi, and
  Okuda]{FuruyaTaniguchiOkuda2024}
Takashi Furuya, Koichi Taniguchi, and Satoshi Okuda.
\newblock
  {\href{https://proceedings.iclr.cc/paper_files/paper/2025/hash/d4b6ccf3acd6ccbc1093e093df345ba2-Abstract-Conference.html}{Quantitative
  approximation for neural operators in nonlinear parabolic equations}}.
\newblock In \emph{The Thirteenth International Conference on Learning
  Representations}, 2025.

\bibitem[Galimberti(2025)]{Galimberti2025}
Luca Galimberti.
\newblock {\href{https://doi.org/10.1142/S0219530525500514}{Neural networks in
  non-metric spaces}}.
\newblock \emph{Analysis and Applications}, 24\penalty0 (6):\penalty0
  1415--1463, 2025.

\bibitem[Galimberti(2026)]{Galimberti2026Lp}
Luca Galimberti.
\newblock {\href{https://arxiv.org/abs/2608.08230}{{$L^p$} approximation
  results for infinite dimensional neural networks}}, 2026.
\newblock arXiv:2608.08230.

\bibitem[Galimberti et~al.(2026)Galimberti, Kratsios, and
  Livieri]{GalimbertiKratsiosLivieri2026}
Luca Galimberti, Anastasis Kratsios, and Giulia Livieri.
\newblock {\href{https://doi.org/10.1007/s00365-026-09745-8}{Designing
  universal causal deep learning models: the case of infinite-dimensional
  dynamical systems from stochastic analysis}}.
\newblock \emph{Constructive Approximation}, 64\penalty0 (2):\penalty0
  411--477, 2026.

\bibitem[G{\"o}deke and Fernsel(2025)]{GoedekeFernsel2025}
Janek G{\"o}deke and Pascal Fernsel.
\newblock {\href{https://arxiv.org/abs/2503.24092}{New universal operator
  approximation theorem for encoder--decoder architectures}}, 2025.
\newblock arXiv:2503.24092.

\bibitem[Gonon et~al.(2023)Gonon, Grigoryeva, and Ortega]{Gonon2023reservoir}
Lukas Gonon, Lyudmila Grigoryeva, and Juan-Pablo Ortega.
\newblock {\href{https://doi.org/10.1214/22-AAP1806}{Approximation bounds for
  random neural networks and reservoir systems}}.
\newblock \emph{The Annals of Applied Probability}, 33\penalty0 (1):\penalty0
  28--69, 2023.

\bibitem[Gonon et~al.(2024)Gonon, Grigoryeva, and Ortega]{Gonon2024}
Lukas Gonon, Lyudmila Grigoryeva, and Juan-Pablo Ortega.
\newblock
  {\href{https://doi.org/10.1016/j.neunet.2024.106486}{Infinite-dimensional
  reservoir computing}}.
\newblock \emph{Neural Networks}, 179:\penalty0 106486, 2024.

\bibitem[Gross(1967)]{Gross1967}
Leonard Gross.
\newblock {\href{https://doi.org/10.1016/0022-1236(67)90030-4}{Potential theory
  on Hilbert space}}.
\newblock \emph{Journal of Functional Analysis}, 1\penalty0 (2):\penalty0
  123--181, 1967.

\bibitem[Hertle(1980)]{Hertle1980}
Alexander Hertle.
\newblock {\href{https://doi.org/10.1007/BFb0088248}{Gaussian surface measures
  and the Radon transform on separable Banach spaces}}.
\newblock In Dietrich K{\"o}lzow, editor, \emph{Measure Theory Oberwolfach
  1979}, volume 794 of \emph{Lecture Notes in Mathematics}, pages 513--531.
  Springer, Berlin, Heidelberg, 1980.

\bibitem[Hida et~al.(1993)Hida, Kuo, Potthoff, and
  Streit]{HidaKuoPotthoffStreit1993}
Takeyuki Hida, Hui-Hsiung Kuo, J{\"u}rgen Potthoff, and Ludwig Streit.
\newblock \emph{{\href{https://doi.org/10.1007/978-94-017-3680-0}{White Noise:
  An Infinite Dimensional Calculus}}}, volume 253 of \emph{Mathematics and Its
  Applications}.
\newblock Kluwer Academic Publishers, Dordrecht, 1993.

\bibitem[Holighaus and Balazs(2017)]{Holighaus2017errata}
Nicki Holighaus and Peter Balazs.
\newblock {\href{https://arxiv.org/abs/1702.06485}{Discretization in
  generalized coorbit spaces: extensions, annotations and errata for
  ``Continuous Frames, Function Spaces and the Discretization Problem'' by M.
  Fornasier and H. Rauhut}}.
\newblock \emph{ArXiv preprint 1702.06485}, 2017.

\bibitem[Holmes(2013)]{Holmes2013}
Irina Holmes.
\newblock {\href{https://doi.org/10.1142/S0219025713500343}{An inversion
  formula for the Gaussian Radon transform for Banach spaces}}.
\newblock \emph{Infinite Dimensional Analysis, Quantum Probability and Related
  Topics}, 16\penalty0 (4):\penalty0 1350034, 2013.

\bibitem[Holmes and Sengupta(2012)]{HolmesSengupta2012}
Irina Holmes and Ambar~N. Sengupta.
\newblock {\href{https://doi.org/10.1016/j.jfa.2012.09.005}{A Gaussian Radon
  transform for Banach spaces}}.
\newblock \emph{Journal of Functional Analysis}, 263\penalty0 (11):\penalty0
  3689--3706, 2012.

\bibitem[Holmes and Sengupta(2014)]{HolmesSengupta2014}
Irina Holmes and Ambar~N. Sengupta.
\newblock {\href{https://doi.org/10.31390/cosa.8.2.07}{The Gaussian Radon
  transform in classical Wiener space}}.
\newblock \emph{Communications on Stochastic Analysis}, 8\penalty0
  (2):\penalty0 247--268, 2014.

\bibitem[Holmes and Sengupta(2015)]{HolmesSengupta2015}
Irina Holmes and Ambar~N. Sengupta.
\newblock {\href{https://doi.org/10.1142/S0219025715500198}{The Gaussian Radon
  transform and machine learning}}.
\newblock \emph{Infinite Dimensional Analysis, Quantum Probability and Related
  Topics}, 18\penalty0 (3):\penalty0 1550019, 2015.

\bibitem[Ismailov(2026{\natexlab{a}})]{Ismailov2026AMAI}
Vugar~E. Ismailov.
\newblock {\href{https://doi.org/10.1007/s10472-026-10003-7}{On shallow
  feedforward neural networks with inputs from a topological space}}.
\newblock \emph{Annals of Mathematics and Artificial Intelligence},
  2026{\natexlab{a}}.
\newblock Published online.

\bibitem[Ismailov(2026{\natexlab{b}})]{Ismailov2026nonEuclidean}
Vugar~E. Ismailov.
\newblock {\href{https://arxiv.org/abs/2602.23381}{Universality of shallow and
  deep neural networks on non-Euclidean spaces}}, 2026{\natexlab{b}}.
\newblock arXiv:2602.23381.

\bibitem[Janson(1997)]{Janson1997book}
Svante Janson.
\newblock \emph{{\href{https://doi.org/10.1017/CBO9780511526169}{Gaussian
  Hilbert Spaces}}}, volume 129 of \emph{Cambridge Tracts in Mathematics}.
\newblock Cambridge University Press, Cambridge, 1997.

\bibitem[Jorgensen et~al.(2023)Jorgensen, Song, and
  Tian]{JorgensenSongTian2023}
Palle E.~T. Jorgensen, Myung-Sin Song, and James Tian.
\newblock
  {\href{https://doi.org/10.1007/s43670-023-00051-z}{Infinite-dimensional
  stochastic transforms and reproducing kernel Hilbert space}}.
\newblock \emph{Sampling Theory, Signal Processing, and Data Analysis},
  21\penalty0 (1):\penalty0 12, 2023.

\bibitem[Kadri et~al.(2016)Kadri, Duflos, Preux, Canu, Rakotomamonjy, and
  Audiffren]{Kadri2016}
Hachem Kadri, Emmanuel Duflos, Philippe Preux, St{\'e}phane Canu, Alain
  Rakotomamonjy, and Julien Audiffren.
\newblock {\href{https://jmlr.org/papers/v17/11-315.html}{Operator-valued
  kernels for learning from functional response data}}.
\newblock \emph{Journal of Machine Learning Research}, 17\penalty0
  (20):\penalty0 1--54, 2016.

\bibitem[Kanagawa et~al.(2018)Kanagawa, Hennig, Sejdinovic, and
  Sriperumbudur]{Kanagawa2018}
Motonobu Kanagawa, Philipp Hennig, Dino Sejdinovic, and Bharath~K.
  Sriperumbudur.
\newblock {\href{https://arxiv.org/abs/1807.02582}{Gaussian processes and
  kernel methods: a review on connections and equivalences}}, 2018.
\newblock arXiv:1807.02582.

\bibitem[Korolev(2022)]{Korolev2022}
Yury Korolev.
\newblock {\href{https://doi.org/10.1137/21M1458144}{Two-layer neural networks
  with values in a Banach space}}.
\newblock \emph{SIAM Journal on Mathematical Analysis}, 54\penalty0
  (6):\penalty0 6358--6389, 2022.

\bibitem[Kovachki et~al.(2021)Kovachki, Lanthaler, and
  Mishra]{KovachkiLanthalerMishra2021}
Nikola Kovachki, Samuel Lanthaler, and Siddhartha Mishra.
\newblock {\href{https://jmlr.org/papers/v22/21-0806.html}{On universal
  approximation and error bounds for Fourier neural operators}}.
\newblock \emph{Journal of Machine Learning Research}, 22\penalty0
  (290):\penalty0 1--76, 2021.

\bibitem[Kovachki et~al.(2023)Kovachki, Li, Liu, Azizzadenesheli, Bhattacharya,
  Stuart, and Anandkumar]{Kovachki2021NOreview}
Nikola Kovachki, Zongyi Li, Burigede Liu, Kamyar Azizzadenesheli, Kaushik
  Bhattacharya, Andrew Stuart, and Anima Anandkumar.
\newblock {\href{https://jmlr.org/papers/v24/21-1524.html}{Neural Operator:
  Learning maps between function spaces with applications to PDEs}}.
\newblock \emph{Journal of Machine Learning Research}, 24\penalty0
  (89):\penalty0 1--97, 2023.

\bibitem[Kovachki et~al.(2024{\natexlab{a}})Kovachki, Lanthaler, and
  Mhaskar]{KovachkiLanthalerMhaskar2024}
Nikola~B. Kovachki, Samuel Lanthaler, and Hrushikesh Mhaskar.
\newblock {\href{https://arxiv.org/abs/2405.15992}{Data complexity estimates
  for operator learning}}, 2024{\natexlab{a}}.
\newblock arXiv:2405.15992.

\bibitem[Kovachki et~al.(2024{\natexlab{b}})Kovachki, Lanthaler, and
  Stuart]{KovachkiLanthalerStuart2024}
Nikola~B. Kovachki, Samuel Lanthaler, and Andrew~M. Stuart.
\newblock {\href{https://doi.org/10.1016/bs.hna.2024.05.009}{Operator learning:
  algorithms and analysis}}.
\newblock In \emph{Numerical Analysis Meets Machine Learning}, volume~25 of
  \emph{Handbook of Numerical Analysis}, pages 419--467. Elsevier,
  2024{\natexlab{b}}.

\bibitem[Kratsios and Bilokopytov(2020)]{KratsiosBilokopytov2020}
Anastasis Kratsios and Ievgen Bilokopytov.
\newblock
  {\href{https://proceedings.neurips.cc/paper/2020/hash/786ab8c4d7ee758f80d57e65582e609d-Abstract.html}{Non-Euclidean
  universal approximation}}.
\newblock In \emph{Advances in Neural Information Processing Systems},
  volume~33, pages 10635--10646, 2020.

\bibitem[Kratsios et~al.(2024)Kratsios, Furuya, Lara~Benitez, Lassas, and
  de~Hoop]{Kratsios2024MoNO}
Anastasis Kratsios, Takashi Furuya, Jose~Antonio Lara~Benitez, Matti Lassas,
  and Maarten de~Hoop.
\newblock {\href{https://arxiv.org/abs/2404.09101}{Mixtures of neural operators
  reduce active complexity in operator learning}}, 2024.
\newblock arXiv:2404.09101.

\bibitem[K\r{u}rkov\'a(2012)]{Kurkova2012}
V\v{e}ra K\r{u}rkov\'a.
\newblock {\href{https://doi.org/10.1016/j.neunet.2012.05.002}{Complexity
  estimates based on integral transforms induced by computational units}}.
\newblock \emph{Neural Networks}, 33:\penalty0 160--167, 2012.

\bibitem[Kuo(1989)]{Kuo1989}
Hui-Hsiung Kuo.
\newblock {\href{https://doi.org/10.1016/0047-259X(89)90069-9}{The Fourier
  transform in white noise calculus}}.
\newblock \emph{Journal of Multivariate Analysis}, 31\penalty0 (2):\penalty0
  311--327, 1989.

\bibitem[Lanthaler(2023)]{Lanthaler2023PCANet}
Samuel Lanthaler.
\newblock {\href{https://jmlr.org/papers/v24/23-0329.html}{Operator learning
  with {PCA-Net}: upper and lower complexity bounds}}.
\newblock \emph{Journal of Machine Learning Research}, 24\penalty0
  (318):\penalty0 1--67, 2023.

\bibitem[Lanthaler(2024)]{Lanthaler2024entropy}
Samuel Lanthaler.
\newblock {\href{https://arxiv.org/abs/2406.18794}{Operator learning of
  Lipschitz operators: an information-theoretic perspective}}, 2024.
\newblock arXiv:2406.18794.

\bibitem[Lanthaler and Stuart(2026)]{LanthalerStuart2025}
Samuel Lanthaler and Andrew~M. Stuart.
\newblock {\href{https://doi.org/10.1093/imanum/draf028}{The parametric
  complexity of operator learning}}.
\newblock \emph{IMA Journal of Numerical Analysis}, 46\penalty0 (2):\penalty0
  647--712, 2026.

\bibitem[Lanthaler et~al.(2022)Lanthaler, Mishra, and
  Karniadakis]{Lanthaler2022}
Samuel Lanthaler, Siddhartha Mishra, and George~Em Karniadakis.
\newblock {\href{https://doi.org/10.1093/imatrm/tnac001}{Error estimates for
  DeepONets: a deep learning framework in infinite dimensions}}.
\newblock \emph{Transactions of Mathematics and Its Applications}, 6\penalty0
  (1):\penalty0 tnac001, 2022.

\bibitem[Lanthaler et~al.(2025)Lanthaler, Li, and
  Stuart]{LanthalerLiStuart2023}
Samuel Lanthaler, Zongyi Li, and Andrew~M. Stuart.
\newblock {\href{https://doi.org/10.1007/s00365-025-09718-3}{Nonlocality and
  nonlinearity implies universality in operator learning}}.
\newblock \emph{Constructive Approximation}, 62\penalty0 (2):\penalty0
  261--303, 2025.

\bibitem[Lee(1991)]{Lee1991}
Yuh-Jia Lee.
\newblock {\href{https://doi.org/10.1016/0022-1236(91)90115-L}{Analytic version
  of test functionals, Fourier transform, and a characterization of measures in
  white noise calculus}}.
\newblock \emph{Journal of Functional Analysis}, 100\penalty0 (2):\penalty0
  359--380, 1991.

\bibitem[Lu et~al.(2021)Lu, Jin, Pang, Zhang, and Karniadakis]{Lu2021}
Lu~Lu, Pengzhan Jin, Guofei Pang, Zhongqiang Zhang, and George~Em Karniadakis.
\newblock {\href{https://doi.org/10.1038/s42256-021-00302-5}{Learning nonlinear
  operators via DeepONet based on the universal approximation theorem of
  operators}}.
\newblock \emph{Nature Machine Intelligence}, 3\penalty0 (3):\penalty0
  218--229, 2021.

\bibitem[Martias(2012)]{Martias2012ridgelet}
Claude Martias.
\newblock {\href{https://doi.org/10.1016/j.crma.2012.03.002}{La transform\'ee
  en ridelettes pour les fonctionnelles de Wiener}}.
\newblock \emph{Comptes Rendus Math\'ematique}, 350\penalty0 (5--6):\penalty0
  259--262, 2012.

\bibitem[Martias(2014{\natexlab{a}})]{Martias2014subexp}
Claude Martias.
\newblock {\href{https://doi.org/10.1016/j.crma.2014.06.012}{On the
  subexponentiality of the ridgelet transform}}.
\newblock \emph{Comptes Rendus Math\'ematique}, 352\penalty0 (12):\penalty0
  1029--1031, 2014{\natexlab{a}}.

\bibitem[Martias(2014{\natexlab{b}})]{Martias2014wavelet}
Claude Martias.
\newblock {\href{https://doi.org/10.1080/17442508.2013.879144}{The wavelet
  transform for Wiener functionals and some applications}}.
\newblock \emph{Stochastics}, 86\penalty0 (5):\penalty0 817--834,
  2014{\natexlab{b}}.

\bibitem[Martias(2022)]{Martias2022wavelets}
Claude Martias.
\newblock {\href{https://doi.org/10.5772/intechopen.104713}{Bases of wavelets
  and multiresolution in analysis on Wiener space}}.
\newblock In Francisco Bulnes, editor, \emph{Recent Advances in Wavelet
  Transforms and Their Applications}. IntechOpen, London, 2022.

\bibitem[Mhaskar and Hahm(1997)]{MhaskarHahm1997}
Hrushikesh~N. Mhaskar and Nahmwoo Hahm.
\newblock {\href{https://doi.org/10.1162/neco.1997.9.1.143}{Neural networks for
  functional approximation and system identification}}.
\newblock \emph{Neural Computation}, 9\penalty0 (1):\penalty0 143--159, 1997.

\bibitem[Mihai and Sengupta(2007)]{MihaiSengupta2007}
Vochita Mihai and Ambar~N. Sengupta.
\newblock
  {\href{https://ambarnsg.wordpress.com/wp-content/uploads/2014/01/mihaiseng.pdf}{The
  Radon--Gauss transform}}.
\newblock \emph{Soochow Journal of Mathematics}, 33\penalty0 (3):\penalty0
  415--434, 2007.

\bibitem[Murata(1996)]{Murata1996}
Noboru Murata.
\newblock {\href{https://doi.org/10.1016/0893-6080(96)00000-7}{An integral
  representation of functions using three-layered networks and their
  approximation bounds}}.
\newblock \emph{Neural Networks}, 9\penalty0 (6):\penalty0 947--956, 1996.

\bibitem[Natterer(1986)]{Natterer1986book}
Frank Natterer.
\newblock
  \emph{\href{https://link.springer.com/book/10.1007/978-3-663-01409-6}{The
  Mathematics of Computerized Tomography}}.
\newblock B. G. Teubner, Stuttgart, 1986.
\newblock \doi{10.1007/978-3-663-01409-6}.

\bibitem[Nelsen and Stuart(2021)]{NelsenStuart2021}
Nicholas~H. Nelsen and Andrew~M. Stuart.
\newblock {\href{https://doi.org/10.1137/20M133957X}{The random feature model
  for input--output maps between Banach spaces}}.
\newblock \emph{SIAM Journal on Scientific Computing}, 43\penalty0
  (5):\penalty0 A3212--A3243, 2021.

\bibitem[Nelsen and Stuart(2024)]{NelsenStuart2024}
Nicholas~H. Nelsen and Andrew~M. Stuart.
\newblock {\href{https://doi.org/10.1137/24M1648703}{Operator learning using
  random features: a tool for scientific computing}}.
\newblock \emph{SIAM Review}, 66\penalty0 (3):\penalty0 535--571, 2024.

\bibitem[Neufeld and Schmocker(2023)]{NeufeldSchmocker2023random}
Ariel Neufeld and Philipp Schmocker.
\newblock {\href{https://arxiv.org/abs/2312.08410}{Universal approximation
  property of Banach space-valued random feature models including random neural
  networks}}, 2023.
\newblock arXiv:2312.08410.

\bibitem[Neufeld and Schmocker(2026)]{NeufeldSchmocker2024noncompact}
Ariel Neufeld and Philipp Schmocker.
\newblock {\href{https://doi.org/10.1142/S0219530525500423}{Universal
  approximation results for neural networks with non-polynomial activation
  function over non-compact domains}}.
\newblock \emph{Analysis and Applications}, 24\penalty0 (5):\penalty0
  1123--1173, 2026.

\bibitem[Pr\'ev\^ot and R\"ockner(2007)]{Prevot2007book}
Claudia Pr\'ev\^ot and Michael R\"ockner.
\newblock \emph{{\href{https://doi.org/10.1007/978-3-540-70781-3}{A Concise
  Course on Stochastic Partial Differential Equations}}}, volume 1905 of
  \emph{Lecture Notes in Mathematics}.
\newblock Springer, Berlin, 2007.

\bibitem[Rossi and Conan-Guez(2005)]{RossiConanGuez2005}
Fabrice Rossi and Brieuc Conan-Guez.
\newblock {\href{https://doi.org/10.1016/j.neunet.2004.07.001}{Functional
  multi-layer perceptron: a non-linear tool for functional data analysis}}.
\newblock \emph{Neural Networks}, 18\penalty0 (1):\penalty0 45--60, 2005.

\bibitem[Schmocker and Teichmann(2026)]{SchmockerTeichmann2026}
Philipp Schmocker and Josef Teichmann.
\newblock {\href{https://arxiv.org/abs/2606.09820}{Weighted universal
  approximation of differentiable maps on infinite-dimensional manifolds}},
  2026.
\newblock arXiv:2606.09820.

\bibitem[Segal(1956)]{Segal1956}
Irving~E. Segal.
\newblock {\href{https://doi.org/10.1090/S0002-9947-1956-0076317-8}{Tensor
  algebras over Hilbert spaces. I}}.
\newblock \emph{Transactions of the American Mathematical Society},
  81:\penalty0 106--134, 1956.

\bibitem[Shi and Choi(2011)]{ShiChoi2011}
Jian~Qing Shi and Taeryon Choi.
\newblock \emph{Gaussian Process Regression Analysis for Functional Data}.
\newblock Chapman \& Hall/CRC, Boca Raton, 2011.

\bibitem[Song et~al.(2023{\natexlab{a}})Song, Fan, Chen, and
  Zhou]{SongFanChenZhou2023}
Linhao Song, Jun Fan, Di-Rong Chen, and Ding-Xuan Zhou.
\newblock {\href{https://doi.org/10.1007/s00041-023-10027-1}{Approximation of
  nonlinear functionals using deep {ReLU} networks}}.
\newblock \emph{Journal of Fourier Analysis and Applications}, 29\penalty0
  (4):\penalty0 50, 2023{\natexlab{a}}.

\bibitem[Song et~al.(2023{\natexlab{b}})Song, Liu, Fan, and
  Zhou]{SongLiuFanZhou2023}
Linhao Song, Ying Liu, Jun Fan, and Ding-Xuan Zhou.
\newblock {\href{https://doi.org/10.1016/j.neunet.2023.07.012}{Approximation of
  smooth functionals using deep {ReLU} networks}}.
\newblock \emph{Neural Networks}, 166:\penalty0 424--436, 2023{\natexlab{b}}.

\bibitem[Sonoda and Murata(2017)]{Sonoda2017acha}
Sho Sonoda and Noboru Murata.
\newblock {\href{https://doi.org/10.1016/j.acha.2015.12.005}{Neural network
  with unbounded activation functions is universal approximator}}.
\newblock \emph{Applied and Computational Harmonic Analysis}, 43\penalty0
  (2):\penalty0 233--268, 2017.

\bibitem[Sonoda et~al.(2022)Sonoda, Ishikawa, and Ikeda]{Sonoda2022gconv}
Sho Sonoda, Isao Ishikawa, and Masahiro Ikeda.
\newblock
  {\href{https://papers.nips.cc/paper_files/paper/2022/hash/fcc3dc27672a12510babe448d665e152-Abstract-Conference.html}{Universality
  of Group Convolutional Neural Networks Based on Ridgelet Analysis on
  Groups}}.
\newblock In S.~Koyejo, S.~Mohamed, A.~Agarwal, D.~Belgrave, K.~Cho, and A.~Oh,
  editors, \emph{Advances in Neural Information Processing Systems 35}, pages
  38680--38694, New Orleans, Louisiana, USA, 2022. Curran Associates, Inc.

\bibitem[Sonoda et~al.(2024)Sonoda, Ishikawa, and Ikeda]{Sonoda2024fourier}
Sho Sonoda, Isao Ishikawa, and Masahiro Ikeda.
\newblock {\href{https://doi.org/10.1016/j.jspi.2024.106184}{A unified Fourier
  slice method to derive ridgelet transform for a variety of depth-2 neural
  networks}}.
\newblock \emph{Journal of Statistical Planning and Inference}, 233:\penalty0
  106184, 2024.

\bibitem[Sonoda et~al.(2026)Sonoda, Ishikawa, and Ikeda]{Sonoda2021ghost}
Sho Sonoda, Isao Ishikawa, and Masahiro Ikeda.
\newblock {\href{https://arxiv.org/abs/2106.04770}{Ghosts in Neural Networks:
  Existence, Structure and Role of Infinite-Dimensional Null Space}}.
\newblock \emph{ArXiv preprint 2106.04770v2}, 2026.

\bibitem[Stinchcombe(1999)]{Stinchcombe1999}
Maxwell~B. Stinchcombe.
\newblock {\href{https://doi.org/10.1016/S0893-6080(98)00108-7}{Neural network
  approximation of continuous functionals and continuous functions on
  compactifications}}.
\newblock \emph{Neural Networks}, 12\penalty0 (3):\penalty0 467--477, 1999.

\bibitem[Taniguchi and Sonoda(2026)]{TaniguchiSonoda2026}
Koichi Taniguchi and Sho Sonoda.
\newblock {\href{https://arxiv.org/abs/2605.10277}{Generalization error bounds
  for Picard-type operator learning in nonlinear parabolic {PDEs}}}, 2026.
\newblock arXiv:2605.10277.

\bibitem[Thind et~al.(2023)Thind, Multani, and Cao]{ThindMultaniCao2023}
Barinder Thind, Kevin Multani, and Jiguo Cao.
\newblock {\href{https://doi.org/10.1080/10618600.2022.2097914}{Deep learning
  with functional inputs}}.
\newblock \emph{Journal of Computational and Graphical Statistics}, 32\penalty0
  (1):\penalty0 171--180, 2023.

\bibitem[Urbina-Romero(2019)]{UrbinaRomero2019}
Wilfredo Urbina-Romero.
\newblock \emph{{\href{https://doi.org/10.1007/978-3-030-05597-4}{Gaussian
  Harmonic Analysis}}}.
\newblock Springer Monographs in Mathematics. Springer, Cham, 2019.

\bibitem[van~der Vaart and van Zanten(2008)]{vanderVaartvanZanten2008}
Aad~W. van~der Vaart and J.~Harry van Zanten.
\newblock {\href{https://doi.org/10.1214/009053607000000613}{Rates of
  contraction of posterior distributions based on Gaussian process priors}}.
\newblock \emph{The Annals of Statistics}, 36\penalty0 (3):\penalty0
  1435--1463, 2008.

\bibitem[van Nuland(2024)]{vanNuland2024}
Teun D.~H. van Nuland.
\newblock {\href{https://doi.org/10.1016/j.neunet.2024.106181}{Noncompact
  uniform universal approximation}}.
\newblock \emph{Neural Networks}, 173:\penalty0 106181, 2024.

\bibitem[Wang et~al.(2016)Wang, Chiou, and M{\"u}ller]{WangChiouMuller2016}
Jane-Ling Wang, Jeng-Min Chiou, and Hans-Georg M{\"u}ller.
\newblock
  {\href{https://doi.org/10.1146/annurev-statistics-041715-033624}{Functional
  data analysis}}.
\newblock \emph{Annual Review of Statistics and Its Application}, 3:\penalty0
  257--295, 2016.

\bibitem[Wynne and Duncan(2022)]{WynneDuncan2022}
George Wynne and Andrew~B. Duncan.
\newblock {\href{https://jmlr.org/papers/v23/20-1180.html}{A kernel two-sample
  test for functional data}}.
\newblock \emph{Journal of Machine Learning Research}, 23\penalty0
  (73):\penalty0 1--51, 2022.

\bibitem[Yang and Shi(2026)]{yang-shi2026}
Jia-Qi Yang and Lei Shi.
\newblock {\href{https://arxiv.org/abs/2606.01244}{Variation Spaces for
  Encoder--Decoder Neural Operators: Approximation and Generalization}}.
\newblock \emph{ArXiv preprint 2606.01244v2}, 2026.

\bibitem[Yao et~al.(2025)Yao, Luo, Cao, Kovachki, O'Leary-Roseberry, and
  Ghattas]{Yao2025DIFNO}
Boyuan Yao, Dingcheng Luo, Lianghao Cao, Nikola Kovachki, Thomas
  O'Leary-Roseberry, and Omar Ghattas.
\newblock {\href{https://arxiv.org/abs/2512.14086}{Derivative-informed Fourier
  neural operator: universal approximation and applications to
  {PDE}-constrained optimization}}, 2025.
\newblock arXiv:2512.14086.

\bibitem[Yao et~al.(2021)Yao, Mueller, and Wang]{YaoMuellerWang2021}
Junwen Yao, Jonas Mueller, and Jane-Ling Wang.
\newblock {\href{https://proceedings.mlr.press/v139/yao21c.html}{Deep learning
  for functional data analysis with adaptive basis layers}}.
\newblock In \emph{Proceedings of the 38th International Conference on Machine
  Learning}, volume 139 of \emph{Proceedings of Machine Learning Research},
  pages 11898--11908, 2021.

\bibitem[Zappala and Bagherian(2026)]{ZappalaBagherian2024}
Emanuele Zappala and Maryam Bagherian.
\newblock {\href{https://doi.org/10.3390/axioms15090676}{Universal
  approximation of operators with transformers and neural integral operators}}.
\newblock \emph{Axioms}, 15\penalty0 (9):\penalty0 676, 2026.

\bibitem[Zhou et~al.(2024)Zhou, Suh, Cheng, and Huo]{ZhouSuhChengHuo2024}
Tian-Yi Zhou, Namjoon Suh, Guang Cheng, and Xiaoming Huo.
\newblock {\href{https://arxiv.org/abs/2403.12187}{Neural operators for
  nonlinear functionals on {RKHS}}}, 2024.
\newblock arXiv:2403.12187.

\end{thebibliography}
\end{document}